\documentclass[12pt]{template} 

\setcitestyle{numbers,open={[},close={]}} 

\title[Universal Online Learning: an Optimistically Universal Learning Rule]{Universal Online Learning: an Optimistically Universal Learning Rule}
\usepackage{times}
\usepackage{thmtools}
\usepackage[mathscr]{eucal}

\altauthor{%
 \Name{Mo\"ise Blanchard} \Email{moiseb@mit.edu}\\
 \addr Massachusetts Institute of Technology
}

\newcommand{\comment}[1]{}

\newcommand{\suol}{\text{SUOL}}
\newcommand{\wuol}{\text{WUOL}}
\newcommand{\crf}{\text{CRF}}
\newcommand{\sual}{\text{SUAL}}

\newcommand{\smv}{{\text{SMV}}}
\newcommand{\wsmv}{{\text{WSMV}}}

\newcommand{\Acal}{\mathcal{A}}
\newcommand{\Bcal}{\mathcal{B}}

\newcommand{\Dcal}{\mathcal{D}}
\newcommand{\Ecal}{\mathcal{E}}
\newcommand{\Fcal}{\mathcal{F}}
\newcommand{\Gcal}{\mathcal{G}}

\newcommand{\Lcal}{\mathcal{L}}

\newcommand{\Pcal}{\mathcal{P}}
\newcommand{\Scal}{\mathcal{S}}
\newcommand{\Tcal}{\mathcal{T}}
\newcommand{\Ucal}{\mathcal{U}}
\newcommand{\Vcal}{\mathcal{V}}

\newcommand{\Xcal}{\mathcal{X}}
\newcommand{\Ycal}{\mathcal{Y}}
\newcommand{\Ocal}{\mathcal{O}}
\newcommand{\Qcal}{\mathcal{Q}}

\newcommand{\Ebb}{\mathbb{E}}
\newcommand{\Nbb}{\mathbb{N}}
\newcommand{\Pbb}{\mathbb{P}}

\newcommand{\Rbb}{\mathbb{R}}

\newcommand{\Xbb}{\mathbb{X}}
\newcommand{\Ybb}{\mathbb{Y}}

\usepackage{bbm}

\newcommand{\1}{\mathbbm{1}}

\definecolor{dark_red}{rgb}{0.2,0,0}

\newcommand{\mb}[1]{\ensuremath{\boldsymbol{#1}}}

\usepackage{graphicx}

\renewenvironment{proof}[1][]{\par\noindent{\bf Proof #1\ }}{\hfill\BlackBox\\[2mm]}

\begin{document}

\maketitle

\begin{abstract}%
\comment{We study the subject of optimistic universal online learning with non-i.i.d. processes. The notion of \emph{universal consistent learning} was defined by Hanneke \cite{hanneke2021learning} in an effort to study learning theory under minimal assumptions, where the objective is to obtain low long-run average loss for any target function. We are interested in characterizing processes for which learning is possible and whether there exist learning rules guaranteed to be universally consistent given the \emph{only} assumption that such learning is \emph{possible}. In the case where the loss function is unbounded, \cite{hanneke2021learning} gave a characterization of learnable processes and proved the existence of such learning rules. We give a drastically simpler formulation of this characterization which shows that the simple \emph{memorization} learning rule is adequate for this context. When the loss function is bounded, we show that learnable processes are invariant from the choice of output space which allows to restrict the analysis to the binary classification setting. Finally, we give a full characterization of processes for which learning is possible in the bounded loss setting and show that a variant from the 1-nearest neighbor algorithm is optimistically universal. This result holds for general separable input Borel space and separable output space with bounded loss. This closes the four main open problems posed in \cite{hanneke2021open} on universal learning.}

We study the subject of universal online learning with non-i.i.d. processes for bounded losses. The notion of an \emph{universally consistent learning} was defined by Hanneke \cite{hanneke2021learning} in an effort to study learning theory under minimal assumptions, where the objective is to obtain low long-run average loss for any target function. We are interested in characterizing processes for which learning is possible and whether there exist learning rules guaranteed to be universally consistent given the \emph{only} assumption that such learning is \emph{possible}. The case of unbounded losses is very restrictive, since the learnable processes almost surely visit a finite number of points and as a result, simple memorization is optimistically universal \cite{hanneke2021learning,blanchard2022universal}. We focus on the bounded setting and give a complete characterization of the processes admitting strong and weak universal learning. We further show that k-nearest neighbor algorithm (kNN) is not optimistically universal and present a novel variant of 1NN which is optimistically universal for general input and value spaces in both strong and weak setting. This closes all the COLT 2021 open problems posed in \cite{hanneke2021open} on universal online learning.
\end{abstract}


\section{Introduction}\label{sec:introduction}
We consider the fundamental question of learnability and generalizability for online learning. In this framework, a learner is sequentially given input points $\Xbb:=(X_t)_{t\geq 0}$ from a general separable metric \emph{instance space} $(\Xcal,\rho)$ and observes the corresponding values $\Ybb:=(Y_t)_{t\geq 0}$ from a separable near-metric \emph{value space} $(\Ycal,\ell)$. The learner's goal is to predict the values before their observation. The input points are given according to some stochastic process $\Xbb$ on $\Xcal$ and we assume that the process $\Ybb$ is generated from $\Xbb$ in a \emph{noiseless} fashion i.e. that there exists an unknown measurable function $f^*:\Xcal\to\Ycal$ such that $Y_t=f^*(X_t)$ for all $t\geq 0$. At time step $t$, the learner outputs a prediction $\hat Y_t$ based solely on the historical data $(X_u,Y_u)_{u<t}$ and the new input point $X_t$. We wish to obtain low long-run average errors $\frac{1}{t}\sum_{u\leq t}\ell(Y_u,\hat Y_u)$. Specifically we consider two types of consistency: strong consistency is achieved when the average error converges to $0$ almost surely; and weak consistency is achieved when the expected average error converges to $0$. We are interested in \emph{universal} online learning, in which we ask for consistency for any unknown measurable target function $f^*$. In this framework, the two main questions are 1. to characterize the input processes $\Xbb$ for which universal consistency is achievable and 2. if possible, provide a learning rule which would guarantee universal consistency whenever such objective is achievable.

\paragraph{Motivation and related work.}
This work builds upon the stream of papers on universal online learning \cite{hanneke2021learning,blanchard2022universal,blanchard2021universal}, which aims to study the question of \emph{learnability} under minimal assumptions. A classical objective in statistical learning is to provide learning rules with guarantees for some large class of problem instances. In general it is not possible to be consistent under all stochastic processes $\Xbb$ and target functions $f^*$. Therefore, it is necessary to impose instance constraints. In the literature, there is a rich variety on the types of proposed restrictions. A first category of works do not restrict the input sequences $\Xbb$ but instead the target functions $f^*$ \cite{littlestone1988learning,cesa2006prediction,ben2009agnostic,rakhlin2015online}. A large portion of the literature belongs to a second category which restricts both input process and target functions. For instance, if we assume that the input process is independent identically distributed (i.i.d.) and that the target function belongs to a class of finite VC dimension, there exist an algorithm guaranteeing $O(\log T)$ mistakes in expectation \cite{haussler1994predicting}. Other more involved restrictions on $\Xbb$ and $f^*$ have been considered \cite{kulkarni2002data,ryabko2006pattern,urner2013probabilistic,bousquet2021theory}. The subject of this paper is of a third category, in which we impose no assumptions on the set of target functions $f^*$, but instead restrict the input sequences $\Xbb$. Specifically, we focus on \emph{universally consistent} algorithms i.e. which achieve consistency for all target functions.

Most of the literature on universal learning considers standard \emph{ad-hoc} probabilistic assumptions on the input stochastic process, for instance assuming that the training samples are i.i.d. A classic result in this i.i.d. setting shows that in the Euclidian space, the 1-nearest neighbor rule is universally consistent \cite{cover1967nearest,stone1977consistent,devroye2013probabilistic}. The $k$-nearest neighbor rule with $k/\log T\to \infty$ and $k/T\to 0$ is also consistent under mild assumptions in the noisy setting where $(\Xbb,\Ybb)$ is any i.i.d. process \cite{stone1977consistent,devroye1994strong}. More recently, \cite{hanneke2021bayes,tsir2022medoid} proposed algorithms which achieve minimal risk for i.i.d. process $(\Xbb,\Ybb)$ in general metric spaces under mild hypothesis---this setting is referred to as universal Bayes consistency. Other similar assumptions on the input process $\Xbb$ include stationary ergodic \cite{morvai1996nonparametric,gyorfi1999simple,gyofi2002strategies} or satisfying the law of large numbers \cite{morvai1999regression,gray2009probability,steinwart2009learning}. Instead, we are interested in provably-minimal assumptions rooted in the learning problem itself. Specifically, we follow the so-called optimist's decision theory introduced by \citet{hanneke2021learning} and frequent in universal learning \cite{tsir2022medoid}: in order to achieve a given objective, the optimist's sole assumption is that this objective is at least achievable by some learning rule. In some sense, this assumption is minimal as it is necessary for any algorithm to have any positive guarantees. In this framework, we are particularly interested in algorithms which would reach the objective without further assumptions. These are named \emph{optimistically universal} learning rules. Such algorithms enjoy the convenient property that if they fail for a particular problem instance, any other learning rule would fail as well. In our case, we are interested in the set of learnable processes $\Xbb$ i.e. for which universal consistency is possible and aim to provide optimistically universal algorithms if they exist i.e. learning rules which are universally consistent on all processes $\Xbb$ for which universal consistency is achievable.

In the case of \emph{unbounded} losses $\ell$, these questions are settled \cite{hanneke2021learning,blanchard2022universal}. Precisely, the learnable processes are exactly the sequences visiting a finite number of input points almost surely, and as a result, the simple memorization is optimistically universal. Hence, universal learning with unbounded losses is very restrictive. In this paper, we focus on the bounded loss case for which it is known that i.i.d. and convergent relative frequencies processes are learnable \cite{hanneke2021learning}. Recently, \cite{blanchard2021universal} provided a reduction from any general bounded output setting $(\Ycal,\ell)$ to binary classification.

\paragraph{Contributions.}We propose a class of learning rule $k$C1NN for $k\geq 2$,  which we prove are strongly and weakly optimistically universal for general separable metric instance spaces $(\Xcal,\rho)$ and separable near-metric value spaces $(\Ycal,\ell)$ with bounded loss. These learning rule are simple variants of the classical 1-nearest neighbor (1NN). They essentially performs 1NN on a restricted dataset by deleting any input point from the historical dataset whenever it has been used as nearest neighbor at least $k$ times. We further show that any $(k_n)_n-$nearest neighbor fails to be optimistically universal under very mild conditions on the sequence $(k_n)_n$. Finally, we give a complete characterization of processes admitting strong and weak universal learning. This closes all main questions on universal online learning, which are stated as open problems in \cite{hanneke2021open}.

\paragraph{Outline of the paper.} The rest of this paper is organized as follows. In the next Section \ref{sec:formal_setup} we formally introduce universal learning and present the two main questions of this topic. The main results are then stated in Section \ref{sec:main_results}. In Section \ref{sec:nearest_neighbor} we focus on nearest neighbor learning rules and show that they are not universally consistent. We then construct a new class learning rule in Section \ref{sec:2C1NN}. For the sake of simplicity and exposition, we prove their strong optimistically universal consistence starting with the case $\Xcal=[0,1]$ and for $k\geq 4$. This allows to obtain the results for all standard Borel spaces and most importantly provides useful intuitions on the general case. We generalize the proof to all separable Borel spaces in Section \ref{sec:borel_spaces} then turn to weak universal learning in Section \ref{sec:weak_learning}. Finally, we give open research directions in Section \ref{sec:conclusion}.

\section{Formal setup and preliminaries}
\label{sec:formal_setup}

\paragraph{Instance and value space.} In this paper, we follow the general framework of online learning where one observes an input sequence $\Xbb = (X_t)_{t\geq 1}$ of points in a separable metric \textit{instance space} $(\Xcal,\rho)$, together with their corresponding target values $\Ybb = (Y_t)_{t\geq 1}$ coming from a separable near-metric \textit{value space} $(\Ycal,\ell)$. The loss $\ell:\Ycal^2\to[0,\infty)$ is said to be a near metric if it is symmetric $\ell(y_1,y_2) =\ell(y_2,y_1)$, discernable $\ell(y_1,y_2)=0$ if and only if $y_1=y_2$, and satisfies a relaxed triangle inequality $\forall y_1, y_2, y_3\in \Ycal^3: \ell(y_1,y_3)\leq c_{\ell} (\ell(y_2,y_1)+\ell(y_2,y_3))$, where $c_{\ell}$ is a fixed constant. Note that all metrics are near-metrics with $c_\ell=1$. As an important example for regression, the squared loss is near-metric with $c_{\ell} = 2$. We denote by $\bar \ell:=\sup_{y_1,y_2 \in  \Ycal} \ell(y_1,y_2) $ the loss function supremum and will be particularly interested in \emph{bounded} losses i.e $\bar \ell<\infty$. 

\paragraph{Input and output processes.} In an effort to study non-i.i.d. processes, the input sequence of points is a general stochastic process on the Borel space $(\Xcal,\Bcal)$ induced by metric $\rho$. This is a major difference with a majority of the statistical learning literature which often imposes ad-hoc hypothesis on $\Xbb$ as discussed in Section \ref{sec:introduction}. We consider a noiseless setting in which the output values $\Ybb$ are generated from $\Xbb$ through an unknown measurable function $f^*:\Xcal\to\Ycal$ such that $Y_t = f^*(X_t)$ for all $t\geq 1$.

\paragraph{Online learning and consistency.} In \emph{online} learning, the learning process is sequential: at time $t\geq 1$, one observes a new input data-point $X_t$ and outputs a prediction $\hat Y_t$ based solely on the historical data $(\Xbb_{\leq t-1},\Ybb_{\leq t-1})$ and the new covariate $X_t$. We measure the performance of the learning rule through the loss function $\ell$. \emph{Strong} consistency is achieved when the algorithm obtains asymptotic average loss $0$ almost surely. Alternatively, a learning rule is \emph{weakly} consistent when it guarantees $0$ asymptotic average loss in expectation. We now formally write these notions. A learning rule is a sequence $f_\cdot=\{f_t\}_{t=1}^{\infty}$ of measurable functions with $f_1:\Xcal\to\Ycal$ and $f_t:\Xcal^{t-1}\times \Ycal^{t-1}\times\Xcal \rightarrow \Ycal$ for $t\geq 2$. Given a history $(X_i,Y_i)_{i<t}$ and a new input point $X_t$, the rule $f_\cdot$ makes the prediction $f_t(\Xbb_{< t}, \Ybb_{< t}, X_t)$ for $Y_t$ and $t\geq 2$. For simplicity, for $t=1$ we may also use the notation $f_1(\Xbb_{< 1}, \Ybb_{< 1}, X_1)$ instead of $f_1(X_t)$. We write the average loss at time $T$ as
\begin{equation*}
    \Lcal_\Xbb(f_\cdot,f^*;T):=\frac{1}{T}\sum_{t=1}^{T}\ell(f_t(\Xbb_{< t},\Ybb_{< t}, X_{t}), f^*(X_t)).
\end{equation*}
We aim to this minimize the long-run average loss. The online learning rule $f_\cdot$ is strongly consistent under the input process $\Xbb$ and for the target function $f^*$ when $\Lcal_{\Xbb}(f_{.},f^*;T)\to 0 ~~(a.s.)$. For simplicity, we define $\Lcal_{\Xbb}(f_{.},f^*) = \limsup_{T\to\infty} \Lcal_\Xbb(f_\cdot,f^*;T)$. Therefore, the above condition can be rewritten as $\Lcal_\Xbb(f_\cdot,f^*)=0~~(a.s.)$. We also consider weak learning: similarly, $f_\cdot$ is weakly consistent under $\Xbb$ and for $f^*$ when $\Ebb \Lcal_\Xbb(f_\cdot,f^*;T)\to 0$.

\paragraph{Universal consistency and optimistically universal learning rule.}
Following \cite{hanneke2021learning}, we are interested in learning rules which achieve strong (resp. weak) consistency under a specific input sequence $\Xbb$ for all measurable target functions $f^*:\Xcal\to\Ycal$. Such learning rules are said to be strongly (resp. weakly) \emph{universally consistent} under $\Xbb$. We define $\suol$ the set of all stochastic processes $\Xbb$ for which strong universal online learning is achievable by some learning rule. Similarly, we denote by $\wuol$ the set of all processes $\Xbb$ that admit weak universal online learning. These sets may depend on the setup $(\Xcal,\rho), (\Ycal,\ell)$ so we will specify $\suol_{(\Xcal,\rho),(\Ycal, \ell)}$ and $\wuol_{(\Xcal,\rho),(\Ycal, \ell)}$ when the spaces are not clear from context. In this framework, two main areas of research are (1) characterizing the sets $\suol$ (resp. $\wuol$) for a given setup in terms of the properties of the stochastic process $\Xbb$, and (2) identifying learning rules which are strongly (resp. weakly) universally consistent for any input process $\Xbb$ in $\suol$ (resp. $\wuol$), i.e. that achieve strong (resp. weak) universal consistency whenever it is possible. These are called \emph{optimistically universal} learning rules. In the case of \emph{unbounded} loss functions i.e. $\bar\ell=\infty$, both questions are answered for any choice of $(\Xcal,\rho),(\Ycal,\ell)$ for strong universal consistency \cite{hanneke2021learning,blanchard2022universal}. Specifically, \cite{blanchard2022universal} shows the stochastic processes $\Xbb$ which admit strong universal online learning are exactly those which visit a \emph{finite} number of distinct input points of $\Xcal$ almost surely. As a consequence, the simple memorization learning rule is optimistically universal. Further, for unbounded losses, strong and weak universal learning are equivalent \cite{hanneke2021learning}. These results are rather negative in the sense that unbounded loss results in a very restricted set $\suol$.

\paragraph{Bounded loss.} The present paper will therefore focus on the \emph{bounded} loss case i.e. $\bar \ell<\infty$, for which both questions are open. This is the main case of interest of universal online learning. Contrary to the unbounded case, for bounded losses the set of learnable processes $\suol$ contains in particular all i.i.d. processes \cite{devroye1994strong,devroye2013probabilistic}. In fact, the simple 1-nearest neighbor learning rule achieves strong (and weak) universal consistency for all i.i.d. processes $\Xbb$. But it is an open question whether 1-nearest neighbor (1NN) is optimistically universal. In other terms, does there exist an input process $\Xbb$ such that 1NN fails to achieve consistency for some target function $f^*$ but universal consistency would still be achieved by some other---more sophisticated---learning rule? No characterization of $\suol$ is known either, although \cite{hanneke2021learning} proposed a necessary condition for belonging to $\suol$ and conjectured that it is also sufficient. We refer to this conditions as $\smv$ (sub-linear measurable visits). Intuitively, it asks that for any measurable partition of the input space $\Xcal$, the process $\Xbb$ only visits a sublinear number of its regions. Note that this condition does not depend on the choice of output setup $(\Ycal,\ell)$. \\

\noindent{\bf Condition \smv}~~ {\it Define the set ${\normalfont\smv_{(\Xcal,\rho)}}$ as the set of all processes $\Xbb$ satisfying the condition that, for every disjoint sequence $\{A_k\}_{k=1}^{\infty}$ in $\Bcal$ with $\cup_{k=1}^{\infty}A_k=\Xcal$ (i.e., every countable measurable partition), $|\{k\in \mathbb N: A_k\cap \mathbb{X}_{<T} \neq \emptyset \}| = o(T) \quad (a.s).$}\\

\noindent For the weak setting we can define a similar condition $\wsmv$ (weak sub-linear measurable visits).\\

\noindent{\bf Condition \wsmv}~~ {\it Define the set ${\normalfont\wsmv_{(\Xcal,\rho)}}$ as the set of all processes $\Xbb$ satisfying the condition that, for every countable measurable partition $\{A_k\}_{k=1}^\infty$, $\mathbb E[|\{k\in \mathbb N: A_k\cap \mathbb{X}_{<T} \neq \emptyset \}|] = o(T).$}\\

\noindent \citet{hanneke2021learning} showed that these conditions are necessary for strong and weak universal learning.
\begin{proposition}[\citet{hanneke2021learning}]
\label{prop:smv_necessary}
For any separable Borel space $\Xcal$ and separable near-metric output setting $(\Ycal,\ell)$ with $0<\bar \ell<\infty$ we have $\suol_{(\Xcal,\rho),(\Ycal,\ell)}\subset\smv_{(\Xcal,\rho)}$ and $\wuol_{(\Xcal,\rho),(\Ycal,\ell)}\subset\wsmv_{(\Xcal,\rho)}$.
\end{proposition}

\noindent However, it is an open question whether $\smv$ (resp. $\wsmv$) is also a sufficient condition for strong (resp. weak) universal learning. Together with the question of the existence of an optimistically universal learning rule, these are the main objectives for universal online learning. These questions are posed in the COLT 2021 open problems \cite{hanneke2021open}, which we now formally restate.

\paragraph{Hanneke's \$5000 open problem 1 \cite{hanneke2021open}} \textit{Does there exist an optimistically universal online learning algorithm? (in
either the weak or strong sense)}

\paragraph{Hanneke's \$1000 open problem 2 \cite{hanneke2021open}} \textit{Is $\smv$ (resp. $\wsmv$) equal to the set of all $\Xbb$ such that strong (resp. weak) universal online learning is possible under $\Xbb$?}\\

It is important to note that these questions are easily solved in the case where $\Xcal$ is countable \cite{hanneke2021learning}. Therefore, a main interest is to answer these questions for \emph{any} uncountable $\Xcal$. In fact, \citet{hanneke2021open} even announced a \$5000 (resp. \$1000) reward for solving open problem 1 (resp. 2) for the Euclidean $\Xcal=\Rbb^d$ case. Both questions will be solved in Section \ref{sec:2C1NN} for $\Xcal=[0,1]$. This is in fact a rather general case because its extension to all standard Borel spaces $\Xcal$ is immediate through an equivalence result from Kuratowski of all uncountable standard Borel spaces. For instance, this solves the question for all Euclidean spaces $\Rbb^d$ for $d\geq 1$. Most importantly, the special case $\Xcal=[0,1]$ allows for a simplified exposition and provides all useful intuitions. To further simplify the proof we focus on the rule $k$C1NN with $k\geq 4$ for $\Xcal=[0,1]$ but the generalization to separable Borel spaces presented in Section \ref{sec:borel_spaces} works for any rule $k$C1NN with $k\geq 2$.

\paragraph{Notations.}For any sequence $\mb x$, we will use the following notations when analyzing finite time horizons: ${\mb x}_{\leq t} := \{X_1,...,X_t\}$ and ${\mb x}_{< t} := \{X_1,...,X_{t-1}\} $ for simplicity. For a metric space $(\Xcal,\rho)$, a point $x\in \Xcal$ and $r\geq 0$, we denote by $B_\rho(x,r) := \{x'\in \Xcal,\; \rho(x,x')<r\}$ the open ball centered in $x$ of radius $r$, and $S_\rho(x,r) = \{x'\in \Xcal,\; \rho(x,x')=r\}$ the sphere centered in $x$ of radius $r$. We might omit the metric $\rho$ in subscript if there is no ambiguity. We also denote by $\ell_{01}$ the indicator loss function i.e. $\ell_{01}(i,j)=\1({i\neq j})$. Since it is a metric, it is also a near-metric with $c_\ell=1$. For simplicity, we will use the same notation $\ell_{01}$ irrespective of the output space $\Ycal$. For any measurable set $A$, we denote by $\1_A$ the function $\1_A(\cdot):=\1_{\cdot\in A}$. We will denote by $|\cdot|$ any norm on $\Rbb$. Recall that all norms are equivalent on finite dimensional spaces, hence the topology induced by these metrics is identical. When the space $(\Xcal,\rho)$ is obvious from the context, we may reduce the notation $\suol_{(\Xcal,\rho),(\Ycal,\ell)}$ to $\suol_{(\Ycal,\ell)}$. We might omit also the loss $\ell$ when there is no ambiguity.

\section{Main results}
\label{sec:main_results}
We first show that the simple nearest neighbor rule (1NN) is not optimistically universal. The proof generalizes to general $(k_n)_n$-nearest neighbor algorithms. 

\begin{restatable}{theorem}{ThmKnNn}
\label{thm:kn_nn}
The $(k_n)_n-$nearest neighbor learning rule is not strongly optimistically universal for the input space $\Xcal=[0,1]$ with usual topology and for binary classification, for any sequence $(k_n)_n$ such that $k_n=o\left(\frac{n}{(\log n)^{1+\delta}}\right)$ for any $\delta>0$.
\end{restatable}

This is obtained by constructing a specific process $\Xbb\in \suol_{([0,1],|\cdot|),(\{0,1\},\ell_{01})}$ under which nearest neighbor is not universally consistent. Intuitively, 1NN fails on the process because certain ``bad'' data points are used an arbitrarily large number of times as nearest neighbor for future input points and hence, induce a large number of mistakes for 1NN. To resolve this issue, we propose a new learning rule $k$-Capped-1-Nearest-Neighbor ($k$C1NN), variant of the classical 1NN, designed to ensure that the number of times each datapoint is used as nearest neighbor is capped at $k$. Specifically, once a datapoint $X_t$ has been used as nearest neighbor $k$ times, it is deleted from the training dataset. We show that this is an optimistically universal learning rule for both strong universal learning and weak universal learning.

\begin{theorem}
\label{thm:2C1NN_optimistically_universal}
For any separable Borel space $\Xcal$, and any separable near-metric output setting $(\Ycal,\ell)$ with bounded loss i.e. $\sup_{y_1,y_2}\ell(y_1,y_2)<\infty$, 2C1NN is a strongly (resp. weakly) optimistically universal learning rule.
\end{theorem}
The proof further shows that all $k$C1NN learning rules are optimistically universal for $k\geq 2$. Further, we give a characterization of the processes admitting strong and weak universal learning.

\begin{theorem}
\label{thm:suol=smv}
For any separable Borel space $\Xcal$, and any separable near-metric output setting $(\Ycal,\ell)$ with $0<\sup_{y_1,y_2}\ell(y_1,y_2)<\infty$, we have
\begin{equation*}
    \suol_{(\Xcal,\rho),(\Ycal,\ell)} = \smv_{(\Xcal,\rho)}\quad \text{and}\quad \wuol_{(\Xcal,\rho),(\Ycal,\ell)} = \wsmv_{(\Xcal,\rho)}.
\end{equation*}
\end{theorem}
If $\sup_{y_1,y_2}\ell(y_1,y_2)=0$, then the loss is identically null. Therefore, all stochastic processes are strongly and weakly learnable.

It is worth noting that although the sets $\suol$ and $\wuol$ differ---the set of weakly learnable processes $\wuol$ is larger than the set of strongly learnable processes $\suol$---the same learning rule 2C1NN is optimistically universal in both strong and weak settings. Theorem \ref{thm:2C1NN_optimistically_universal} and Theorem \ref{thm:suol=smv} close the two open problems of the existence of an optimistically universal learning rule and a characterization of the set of learnable input sequences, formulated in \cite{hanneke2021open}.

\section{On nearest neighbor consistency}
\label{sec:nearest_neighbor}

A natural candidate for a good learning rule in general spaces is the nearest neighbor algorithm. Indeed, for instance for $\Xcal=\Rbb$ and binary classification, under any process $\Xbb\in \suol$ which admits universal learning, nearest neighbor successfully learns simple functions---representing union of intervals \cite{blanchard2021universal}. Further, the special case of binary classification is not restrictive because if nearest neighbor were optimistically universal for binary classification, it would also be optimistically universal in the general separable bounded case \cite{blanchard2021universal}. In this section, we show that in fact nearest neighbor learning rule is not optimistically universal even on the interval $\Xcal=[0,1]$.

\begin{theorem}
\label{thm:1nn_nonoptimistic}
1NN is not optimistically universal for binary classification on $\Xcal=[0,1]$ with usual topology.
\end{theorem}

To prove this result, we first define the set of processes with convergent relative frequencies $\crf$ as the set of processes $\Xbb$ such that $\forall A\in \Bcal$,
\begin{equation*}
    \lim_{T\to \infty} \frac{1}{T}\sum_{t=1}^T \1_A(X_t) \quad \text{exists (a.s.)},
\end{equation*}
and explicitely construct a process $\Xbb^{(1)}\in \crf$ on which nearest neighbor fails. Because convergent relative frequencies processes are learnable $\crf\subset\suol$ \cite{hanneke2021learning}, this shows that 1NN is not optimistically universal for the online learning setting. As a remark, it was already known that the self-adaptive/inductive nearest neighbor learning rule is not optimistically universal for the self-adaptive setting \cite{hanneke2021learning} (Section 3.2). Similarly to the set $\suol$, we can define the set $\sual$ of processes $\Xbb$ admitting strong universal learning in the self-adaptive setting. The proof that self-adaptive nearest neighbor is not optimistically universal is also constructive but not relevant for the online setting because it relies on a completely different process $\Xbb^{(2)}\in \sual$ under which self-adaptive nearest neighbor fails but online learning nearest neighbor is universally consistent. Indeed, the set of learnable processes for online learning is larger than the set of learnable processes for self-adaptive learning $\sual\subset\suol$, and strictly larger whenever $\Xcal$ is infinite \cite{hanneke2021learning}.

The process $\Xbb^{(1)}$ is designed so that nearest neighbor fails on the function $f^*(\cdot) = \1_\Dcal(\cdot)$ where $\Dcal$ is the set of diadics. Intuitively, the process alternates between a carefully chosen random diadic $X_{n_k}\in \Dcal$ and a sequence of random points $X_t,\; n_k<t<n_{k+1}$ which converge exponentially to $X_{n_k}$ but that does not fall in the diadics almost surely. The nearest neighbor algorithm therefore uses the diadic $X_{n_k}$ as representant for most of the points $X_t$ for $n_k<t<n_{k+1}$ and as a result assigns the wrong category $\hat Y_t =1$. We then impose $n_{k+1}-n_k\to\infty$ so that nearest neighbor makes an asymptotic error rate of $1$. A major technical difficulty is to ensure that the process $\Xbb^{(1)}$ is still universally learnable and in particular in $\crf$. To do so, we randomly select $X_{n_k}$ in high-order diadics so that the convergence of the points $X_t$ for $n_k<t<n_{k+1}$ is mild compared to the discretization of $[0,1]$ of these high-order diadics.\\

\begin{proof}[of Theorem \ref{thm:1nn_nonoptimistic}]
To show that 1NN is not optimistically universal we construct a process $\Xbb\in \suol$ on which 1NN has asymptotic error rate $1$. Let $\epsilon>0$. We denote by $\Dcal_p:=\{\frac{i}{2^p},0\leq i\leq 2^p, \; i\text{ odd}\}$ the set of diadics of order $p$ i.e. with denominator $2^p$. Denote $n_k = \lfloor k(1+\log k)\rfloor$ and $p_k=k^2$ for $k\geq 1$. Let $(U_k)_{k\geq 1}$ be an i.i.d. sequence of uniforms $\Ucal([0,1])$ and $(D_k)_{k\geq 1}$ a sequence of independent random variables---also independent of $(U_k)$---such that $D_k\sim\Ucal(\Dcal_{p_k})$. We now define the process $\Xbb$ as follows,
\begin{equation*}
    X_{n_k} = D_k,\quad k\geq 1 \quad \text{and} \quad X_{n_k+i} = D_k + \frac{U_k-D_k}{2^{n_k}4^{i}},\quad k\geq 1, \; 1\leq i<n_{k+1}-n_k.
\end{equation*}

We first show that 1NN is not universally consistent on $\Xbb$. Indeed, we will show that 1NN is not consistent for the function $f^* = \1_\Dcal$ where $\Dcal = \{\frac{i}{2^p},p\geq 1,0\leq i\leq 2^p\}$ is the set of diadics. For any $k\geq 1$,
\begin{equation*}
    \Pbb\left[ \min_{t<n_k} |X_t-D_k| < \frac{1}{2^{n_k}}\right] \leq \sum_{t<n_k}  \Pbb\left[ X_t-\frac{1}{2^{n_k}} <  D_k < X_t+\frac{1}{2^{n_k}}\right] \leq \frac{2n_k}{2^{n_k}}.
\end{equation*}
where in the last inequality, we use the fact that $p_k>n_k$ which shows that there are at most $2^{p_k-(n_k-1)}$ diadics of order $p_k$ in an interval of length $\frac{1}{2^{n_k-1}}$. Remember that almost surely, $X_{n_k+i}\notin \Dcal$ for $k\geq 1$ and $1\leq i< n_{k+1}-n_k$. We will therefore denote by $\Ecal$ the event of probability $1$ where $\Xbb$ does not visit $\Dcal$ except for times $n_k$, $k\geq 1$. In other words,
\begin{equation*}
    \Ecal := \{X_{n_k+i}\notin \Dcal, \quad k\geq 1, 1\leq i\leq k-1\}
\end{equation*}
and $\Pbb(\Ecal)=1$. For simplicity, we also denote by $\Acal_k$ the event $\Acal_k:=\{\min_{t<n_k} |X_t-D_k| \geq 2^{-n_k}\}$. On the event $\Acal_k\cap \Ecal$, the nearest neighbor of $X_{n_k+1}$ is $X_{n_k}=D_k\in \Dcal$, and similarly, the nearest neighbor of $X_{n_k+i}$ is $X_{n_k}$ for all $1\leq i<n_{k+1}-n_k$. Therefore, 1NN makes an error in the prediction of all $X_{n_k+i}$ for $1\leq i<n_{k+1}-n_k$. Therefore, for any $k'\geq 1$, on the event $\Ecal\cap \bigcap_{k\geq k'}\Acal_k$, for any $t>n_{k'}$ we have $\ell(1NN(\Xbb_{<t},\Ybb_{<t},X_t),f^*(X_t))\geq \1_{t\notin (n_k)_{k\geq 1}}$ which gives $\Lcal_\Xbb(1NN,f^*) = 1$ since the frequency of the sequence $(n_k)_{k\geq 1}$ vanishes to $0$. By Borel-Cantelli, because $\Pbb[\Ecal \cap \Acal_k^c]\leq \frac{2n_k}{2^{n_k}}$ and $\sum_{n\geq 1} \frac{2n}{2^n}<\infty$, we obtain $\mathbb P[\Ecal\cap \cup_{k'\geq 1}  \bigcap_{k\geq k'}\Acal_k]= 1$. To summarize, on the event $\Ecal\cap \cup_{k'\geq 1}  \bigcap_{k\geq k'}\Acal_k$ of probability $1$, 1NN has error rate $\Lcal_\Xbb(1NN,f^*) = 1$, which shows that 1NN is not universally consistent for $\Xbb$.\\

We now show that $\Xbb\in \suol$. To do so, we show the stronger statement that $\Xbb\in \crf$. Indeed, we recall that $\crf\subset \suol$ \cite{hanneke2021learning}. Let $A\subset [0,1]$ a measurable set. We will show that almost surely
\begin{equation*}
    \frac{1}{T}\sum_{t=1}^T \1_A(X_t) \to \mu(A),
\end{equation*}
where $\mu$ is the Lebesgue measure. To do so, we introduce the random variables 
\begin{equation*}
    Y_k = \sum_{i=0}^{n_{k+1}-n_k-1}\1_A(X_{n_k+i}).
\end{equation*}
Note that for example $\sum_{t=1}^{n_{k+1}-1} \1_A(X_t) = \sum_{l=1}^k Y_k$ and that the random variables $(Y_k)_{k\geq 1}$ are together independent. We first show that $\frac{1}{n_{k+1}-n_k}\Ebb Y_k\to \mu(A)$.

For any $k\geq 1$ and $1\leq i<n_{k+1}-n_k$, recall that $X_{n_k+i}$ is defined as $D_k + \frac{U_k-D_k}{2^{n_k} 4^i}$. Therefore, $X_{n_k+i}$ is an absolutely continuous random variable with density
\begin{equation*}
    f(x)= \frac{1}{2^{p_k-1}}\sum_{l=0}^{2^{p_k-1}-1} f_l(x)
\end{equation*}
where $f_l(x)$ corresponds to the conditional density to $D_k = \frac{2l+1}{2^{p_k}}=:d_l$, i.e.
\begin{equation*}
    f_l(x) = 2^{n_k}4^i \cdot \1\left(x\in \left[ d_l-\frac{d_l}{2^{n_k}4^i}, d_l + \frac{1-d_l}{2^{n_k}4^i} \right]\right)
\end{equation*}
But $x\in \left[ d_l-\frac{d_l}{2^{n_k}4^i}, d_l + \frac{1-d_l}{2^{n_k}4^i} \right]$ i.if $
   \frac{2^{p_k-1}(x-\frac{1}{2^{n_k}4^i})}{1-\frac{1}{2^{n_k}4^i}}-\frac{1}{2} \leq l\leq \frac{2^{p_k-1}x}{1-\frac{1}{2^{n_k}4^i}}-\frac{1}{2}.$ Therefore, the number $N(x)$ of non-zero terms in the sum $f(x)= \frac{1}{2^{p_k-1}}\sum_{l=0}^{2^{p_k-1}-1} f_l(x)$ is 
\begin{equation*}
    \frac{2^{p_k-1-n_k-2i}}{1-\frac{1}{2^{n_k}4^i}} -1\leq N(x)\leq  \frac{2^{p_k-1-n_k-2i}}{1-\frac{1}{2^{n_k}4^i}}+1
\end{equation*}
Hence, 
\begin{equation*}
    \left|f(x) - \frac{1}{1-\frac{1}{2^{n_k}4^i}}\right| = \left|\frac{2^{n_k} 4^i N(x)}{2^{p_k-1}} - \frac{1}{1-\frac{1}{2^{n_k}4^i}}\right| \leq \frac{1}{2^{p_k-1-n_k-2i}}.
\end{equation*}
Finally, we obtain
\begin{equation*}
    |\mathbb P(X_{n_k+i}\in A) - \mu(A)|\leq \frac{1}{2^{p_k-1-n_k-2i}} + \frac{1}{2^{n_k+2i}-1}
\end{equation*}
Therefore,
\begin{align*}
    \left|\Ebb Y_k-(n_{k+1}-n_k)\mu(A)\right| &\leq |\Pbb(X_{n_k}\in A)-\mu(A)| +  \sum_{i=1}^{n_{k+1}-n_k-1} \Pbb(X_{n_k+i}\in A)\\
    &\leq 1 +  \frac{n_{k+1}-n_k}{2^{p_k-2n_{k+1}}} + \frac{n_{k+1}-n_k}{2^{n_k}-1}\\
    &\leq C
\end{align*}
where $C\geq 1$ is some universal constant, given that $\frac{n_{k+1}-n_k}{2^{p_k-2n_{k+1}}}\to 0$ and $\frac{n_{k+1}-n_k}{2^{n_k}-1}\to 0$ as $k\to\infty$. Now note that because $Y_k$ is a sum of $n_{k+1}-n_k$ random variables bounded by $1$, then $Var(Y_k)\leq (n_{k+1}-n_k)^2 = \Ocal((\log k)^2)$. Therefore, $\sum_{k\geq 1} \frac{Var(Y_k)}{k^2}<\infty.$ We can therefore apply Kolmogorov’s strong law of large numbers to the independent random variables $(Y_k)_{k\geq 1}$ which gives
\begin{equation*}
    \epsilon_k:= \frac{\sum_{l=1}^k Y_l - \mathbb E Y_l}{k}\to 0\quad (a.s.)
\end{equation*}
We now compute, 
\begin{align*}
\left|\frac{1}{n_{k+1}-1}\sum_{t=1}^{n_{k+1}-1} \1_A(X_t) - \mu(A)\right| &= \frac{1}{n_{k+1}-1}\left|\sum_{l=1}^k Y_l -  (n_{k+1}-n_k)\mu(A) \right|\\
  &=  \frac{1}{n_{k+1}-1}\left|k\epsilon_k +\sum_{l=1}^k \Ebb Y_l - (n_{k+1}-n_k)\mu(A)  \right|\\
  &\leq \frac{k(\epsilon_k + C)}{n_{k+1}-1}.
\end{align*}
Because $\frac{k}{n_{k+1}-1}\to 0$, we obtain $\frac{1}{n_{k+1}-1}\sum_{t=1}^{n_{k+1}-1} \1_A(X_t) \to \mu(A)\quad (a.s.)$. We complete the proof by noting that for any $n_k\leq T<n_{k+1}$,
\begin{equation*}
    \frac{1}{n_{k+1}-1}\sum_{t=1}^{n_{k}-1} \1_A(X_t) \leq \frac{1}{t}\sum_{t=1}^T \1_A(X_t) \leq \frac{1}{n_{k}-1}\sum_{t=1}^{n_{k+1}-1} \1_A(X_t).
\end{equation*}
Then, with $\eta_k =  \frac{n_k-1}{n_{k+1}-1}$,
\begin{equation*}
    \frac{ \eta_k}{n_{k}-1}\sum_{t=1}^{n_{k}-1} \1_A(X_t) \leq \frac{1}{t}\sum_{t=1}^T \1_A(X_t) \leq    \frac{1}{ \eta_k(n_{k}-1)}\sum_{t=1}^{n_{k+1}-1} \1_A(X_t).
\end{equation*}
Because $\eta\to 1$ as $k\to \infty$, we get the desired result that $\frac{1}{T}\sum_{t=1}^T \1_A(X_t) \to \mu(A)$. Therefore, $\Xbb\in \crf$. This ends the proof of the theorem.
\end{proof}

Using a similar proof structure, we can generalize the result to prove that general $(k_n)_n-$nearest neighbor algorithms are not optimistically universal under mild conditions on $(k_n)_n$ which yields Theorem \ref{thm:kn_nn}. We recall that the $(k_n)_n-$nearest neighbor learning rule, at step $n$, considers the closest $k_n$ neighbors to the new input point and follows the majority vote to make its prediction.

\ThmKnNn*

\begin{proof}
We adapt the parameters $n_k$, $p_k$ and the process $\Xbb$ of the proof of Theorem \ref{thm:1nn_nonoptimistic}. Let $\delta>0$ and a sequence $k_n = o\left(\frac{n}{(\log n)^{1+\delta}}\right)$. We now construct a process $\Xbb$ on which $(k_n)_n-$NN is not universally consistent. We use the same notation $\Dcal_p:=\{\frac{i}{2^p},0\leq i\leq 2^p, \; i\text{ odd}\}$ for the set of diadics of order $p$ and $\Dcal$ for the set of diadics.

Let $\epsilon>0$ such that $\frac{1+2\epsilon}{1-2\epsilon}<1+\frac{\delta}{2}$. Then pose for $k\geq 1$,
\begin{equation*}
    n_k = \lfloor e^{k^{1/2-\epsilon}}\rfloor,\quad d_k = \min\left(  \left \lfloor \frac{n_k}{(\log n_k)^{1+\delta}} \right\rfloor,n_{k+1}-n_k-1\right), \quad p_k = 4^k.
\end{equation*}
First note that $n_{k+1}-n_k\sim \left(\frac{1}{2}-\epsilon\right)\frac{n_k}{k^{1/2+\epsilon}}\sim \left(\frac{1}{2}-\epsilon\right) \frac{n_k}{(\log n_k)^{\frac{1/2+\epsilon}{1/2-\epsilon}}}$ therefore we obtain
\begin{equation*}
    d_k = o\left(\frac{n_k}{(\log n_k)^{1+\delta/2}}\right) = o(n_{k+1}-n_k).
\end{equation*}
Also, for $k$ large enough, $d_k = \left \lfloor \frac{n_k}{(\log n_k)^{1+\delta}} \right\rfloor$. We now construct the process $\Xbb$ in a similar way to the proof of Theorem \ref{thm:1nn_nonoptimistic}. Let $(U_k)_{k\geq 1}$ be an i.i.d. sequence of uniforms $\Ucal([0,1])$ and $(D_k)_{k\geq 1}$ a sequence of independent random variables---also independent of $(U_k)_k$---such that $D_k\sim\Ucal(\Dcal_{p_k})$. Additionally, we denote by $D_{k,i}$ the $i-$th closest diadic of order $p_k$ to $D_k$. For instance, $D_{k,1}=D_k$, and $|D_{k,i}-D_k|\leq \frac{i}{2^{p_k-1}}$. For intuition, if $D_k$ is not close to the boundary of $[0,1]$, we have $D_{k,i} = D_k +(-1)^i \cdot \frac{\lfloor i/2\rfloor }{2^{p_k}}$. We now define the process $\Xbb$ as follows for any $k\geq 1$,
\begin{equation*}
    X_{n_k+i} = D_{k,i+1},\quad 0\leq i\leq d_k \quad \text{and} \quad X_{n_k+d_k+j} = D_k + \frac{U_k-D_k}{2^{n_k}4^{j}},\quad 1\leq j<n_{k+1}-n_k-d_k.
\end{equation*}

We first prove that $(k_n)_n-$NN is not consistent for the function $f^* = 1_\Dcal$. For any $k\geq 1$,
\begin{equation*}
    \Pbb\left[ \min_{t<n_k} |X_t-D_k| < \frac{1}{2^{n_k}}\right] \leq \sum_{t<n_k}  \Pbb\left[ X_t-\frac{1}{2^{n_k}} <  D_k < X_t+\frac{1}{2^{n_k}}\right] \leq \frac{2n_k}{2^{n_k}},
\end{equation*}
because $n_k\leq p_k$. Now note that for all $k\geq 1$ and $0\leq i\leq d_k$ we have $X_{n_k+i}\in \Dcal_{p_k}$, while almost surely, all other random variables do not fall in $\Dcal$. Then, denote by $\Ecal$ the event of probability $1$ where $\Xbb$ does not visit $\Dcal$ except for times $n_k+i$ for $k\geq 1$ and $0\leq i\leq d_k$. In other words,
\begin{equation*}
    \Ecal := \{X_{n_k+i}\notin \Dcal, \quad k\geq 1, d_k< i< n_{k+1}-n_k\}
\end{equation*}
and $\Pbb(\Ecal)=1$. We also denote by $\Acal_k$ the event $\Acal_k:=\{\min_{t<n_k} |X_t-D_k| \geq 2^{-n_k}\}$ and $\Bcal_k$ the event $\Bcal_k:=\{|U_k-D_k|\geq 2^{-k}\}$. We have $\Pbb(\Bcal_k^c)\leq 2^{-k+1}$ and we showed previously $\Pbb(\Acal_k^c)\leq \frac{2n_k}{2^{n_k}}$. Now note that $\frac{d_k}{2^{p_k-1}}=o(\frac{1}{2^{n_k+2n_{k+1}+k+1}})$. Therefore, let $k_0$ such that for any $k\geq k_0$, $\frac{d_k}{2^{p_k-1}}\leq \frac{1}{2^{n_k+2n_{k+1}+k+1}}$. Then, for any $k\geq k_0$, on the event $\Acal_k\cap\Bcal_k\cap \Ecal$, for any $1\leq j<n_{k+1}-n_k-d_k$, the $d_k+1$ nearest neighbors of $X_{n_k+d_k+j}$ are exactly the points $\{X_{n_k+i}=D_{k,i+1},\; 0\leq i\leq d_k\}$. Indeed,
\begin{equation*}
    |X_{n_k+d_k+j}-D_{k,i}|\leq |X_{n_k+d_k+j}-D_k| + \frac{d_k}{2^{p_k-1}} \leq \frac{1}{2^{n_k}4^j} + \frac{1}{2^{n_k+2j}} < \frac{1}{2^{n_k+2j-1}}.
\end{equation*}
Further, for all $t<n_k$,
\begin{equation*}
    |X_{n_k+d_k+j}-X_t| \geq |D_k-X_t| - |X_{n_k+d_k+j}-D_k| \geq \frac{1}{2^{n_k}} - \frac{1}{2^{n_k+2}} >\frac{1}{2^{n_k+2j-1}}.
\end{equation*}
and finally, for $1\leq j'<j$ and any $0\leq i\leq d_k$, we have
\begin{align*}
    |X_{n_k+d_k+j}-X_{n_k+d_k+j'}|&\geq |X_{n_k+d_k+j}-X_{n_k+d_k+j-1}| =  3\cdot \frac{|U_k-D_k|}{2^{n_k+2j}}\\
    &\geq |X_{n_k+d_k+j}-D_k| + 2\cdot \frac{1}{2^{n_k+2j+k}}\\
    &\geq |X_{n_k+d_k+j}-D_k| +  2\cdot \frac{d_k}{2^{p_k-1}}\\
    &> |X_{n_k+d_k+j}-D_k| +  |D_k-D_{k,i}|\\
    &\geq |X_{n_k+d_k+j}-D_{k,i}|.
\end{align*}
We now observe that 
\begin{equation*}
    \max_{n_k+d_k+1\leq n< n_{k+1}} k_n = o\left(\frac{n_{k+1}}{(\log n_{k})^{1+\delta}}\right) = o(d_k).
\end{equation*}
Therefore, let $k_1$ such that for any $k\geq k_1$, and any $1\leq j< n_{k+1}-n_k-d_k$, we have $k_{n_k+d_k+j}\leq d_k$. Now for any $k\geq \max(k_0,k_1)$, on the event $\Acal_k\cap\Bcal_k\cap \Ecal$, $(k_n)_n-$NN makes an error in the prediction of all $X_{n_k+d_k+j}$ for $1\leq j<n_{k+1}-n_k-d_k$ since its $k_n$ closest neighbors are in the set $\{X_{n_k+i}=D_{k,i+1},\; 0\leq i\leq d_k\}$ which all have value $\1_\Dcal(X_{n_k+i})=1$ instead of $\1_\Dcal(X_{n_k+d_k+j})=0$.

Last, note that the frequency of the times of the form $n_k+i$ for $k\geq 1$ and $0\leq i\leq d_k$ vanishes to $0$, because $d_k = o(n_{k+1}-n_k)$ and $n_{k+1}\sim n_k$. Therefore, on the event $\Ecal\cap \cup_{k'\geq 1}  \bigcap_{k\geq k'}(\Acal_k\cap \Bcal_k)$, the learning rule $(k_n)_n-$NN has error rate $\Lcal_\Xbb((k_n)-NN,f^*) = 1$. Using the same argument as in the proof of Theorem \ref{thm:1nn_nonoptimistic}, we can show that $\Pbb [\Ecal\cap \cup_{k'\geq 1}  \bigcap_{k\geq k'}(\Acal_k\cap \Bcal_k)]=1$, which shows that $(k_n)_n-$NN is not consistent for process $\Xbb$ and target function $f^*=\1_\Dcal$. This ends the proof that $(k_n)_n-$NN is not universally consistent for process $\Xbb$.\\

We now show that $\Xbb\in \suol$ by showing that in fact $\Xbb\in \crf$. Let $A\subset [0,1]$. We will show that the frequencies of falling in $A$ converge almost surely to $\mu(A)$ where $\mu$ is the Lebesgue measure. We introduce the random variables 
\begin{equation*}
    Y_k = \sum_{i=0}^{n_{k+1}-n_k-1}\1_A(X_{n_k+i}).
\end{equation*}
Again, for $k\geq 1$, and $1\leq j< n_{k+1}-n_k-d_k$, $X_{n_k+d_k+j}$ is an absolutely continuous random variable with density $f(x)= \frac{1}{2^{p_k-1}}\sum_{l=0}^{2^{p_k-1}-1} f_l(x)$ where $f_l(x)$ corresponds to the conditional density to $D_k = \frac{2l+1}{2^{p_k}}=:d_l$, i.e.
\begin{equation*}
    f_l(x) = 2^{n_k}4^j \cdot \1\left(x\in \left[ d_l-\frac{d_l}{2^{n_k}4^j}, d_l + \frac{1-d_l}{2^{n_k}4^j} \right]\right)
\end{equation*}
The same proof as for Theorem \ref{thm:1nn_nonoptimistic} gives
\begin{equation*}
    |\mathbb P(X_{n_k+d_k+j}\in A) - \mu(A)|\leq \frac{1}{2^{p_k-1-n_k-2j}} + \frac{1}{2^{n_k+2j}-1}.
\end{equation*}
Therefore,
\begin{align*}
    \left|\Ebb Y_k-(n_{k+1}-n_k)\mu(A)\right| &\leq \sum_{i=0}^{d_k}|\Pbb(X_{n_k+i}\in A)-\mu(A)| +  \sum_{j=1}^{n_{k+1}-n_k-d_k-1} \Pbb(X_{n_k+d_k+j}\in A)\\
    &\leq d_k+1 +  \frac{n_{k+1}-n_k}{2^{p_k-2n_{k+1}}} + \frac{n_{k+1}-n_k}{2^{n_k}-1}\\
    &\leq d_k + C
\end{align*}
where $C\geq 1$ is some universal constant, given that $\frac{n_{k+1}-n_k}{2^{p_k-2n_{k+1}}}\to 0$ and $\frac{n_{k+1}-n_k}{2^{n_k}-1}\to 0$ as $k\to\infty$. Now note that because $Y_k$ is a sum of $n_{k+1}-n_k$ random variables bounded by $1$, then 
\begin{equation*}
    Var(Y_k)\leq (n_{k+1}-n_k)^2 = \Ocal\left(\frac{n_{k+1}^2}{k^{1+2\epsilon}}\right).
\end{equation*}
Therefore, $\sum_{k\geq 1} \frac{Var(Y_k)}{(n_{k+1}-1)^2}<\infty.$ Further, we can note that the random variables $(Y_k)_{k\geq 1}$ are together independent. Thus, by Kolmogorov’s Convergence Criteria, we obtain
\begin{equation*}
    \sum_{l=1}^k \frac{ Y_l - \mathbb E Y_l}{n_{k+1}-1}\to 0\quad (a.s.)
\end{equation*}
We then apply Kronecker's lemma which gives 
\begin{equation*}
    \epsilon_k := \frac{\sum_{l=1}^k  Y_l - \mathbb E Y_l}{n_{k+1}-1}\to 0\quad (a.s.)
\end{equation*}
We now compute, 
\begin{align*}
\left|\frac{1}{n_{k+1}-1}\sum_{t=1}^{n_{k+1}-1} \1_A(X_t) - \mu(A)\right| &= \frac{1}{n_{k+1}-1}\left|\sum_{l=1}^k Y_l -  (n_{k+1}-n_k)\mu(A) \right|\\
  &=  \frac{1}{n_{k+1}-1}\left|(n_{k+1}-1)\epsilon_k +\sum_{l=1}^k \Ebb Y_l - (n_{k+1}-n_k)\mu(A)  \right|\\
  &\leq \epsilon_k +  \frac{Ck+\sum_{l=1}^k d_l}{n_{k+1}-1}.
\end{align*}
Because $\frac{k}{n_{k+1}-1}\to 0$ and $\sum_{l=1}^k d_l = o(n_{k+1}-1)$, we obtain $\frac{1}{n_{k+1}-1}\sum_{t=1}^{n_{k+1}-1} \1_A(X_t) \to \mu(A)\quad (a.s.)$. We complete the proof by noting that for any $n_k\leq T<n_{k+1}$,
\begin{equation*}
    \frac{1}{n_{k+1}-1}\sum_{t=1}^{n_{k}-1} \1_A(X_t) \leq \frac{1}{T}\sum_{t=1}^T \1_A(X_t) \leq \frac{1}{n_{k}-1}\sum_{t=1}^{n_{k+1}-1} \1_A(X_t),
\end{equation*}
and that $\frac{n_k-1}{n_{k+1}-1}\to 1$ as $k\to\infty$. Therefore $\frac{1}{T}\sum_{t=1}^T \1_A(X_t) \to \mu(A)\quad (a.s.)$ which shows that $\Xbb\in \crf$. This ends the proof of the theorem.
\end{proof}

\section{An optimistically universal learning rule}
\label{sec:2C1NN}

\begin{algorithm}[h]
\hrule height\algoheightrule\kern3pt\relax
\caption{$k$C1NN learning rule}\label{alg:kC1NN}
\hrule height\algoheightrule\kern3pt\relax
\textbf{Input:} Historical samples $(X_t,Y_t)_{t<T}$ and new input point $X_T$\\
\textbf{Output:} Predictions $\hat Y_t = kC1NN_t({\mb X}_{<t},{\mb Y}_{<t},X_t)$ for $t\leq T$\\
$\hat Y_1:= 0$\\
$\Dcal_2:= \{1\}$\\
$n_1 \gets 0$\\
$t\gets 2$\\
\While{$t\leq T$}{
  \eIf{exists $u<t$ such that $X_u=X_t$}{
    $\hat Y_t := Y_u$\\
    $\Dcal_{t+1} := \Dcal_t$
  }{
  $\phi(t):= \arg\min_{u\in \Dcal_t} \rho(X_t,X_u)$\\
  $\hat Y_t:=Y_{\phi(t)}$\\
  $n_{\phi(t)}\gets n_{\phi(t)}+1$\\
  $n_t\gets 0$\\
  \eIf{$n_{\phi(t)}=k$}{
  $\Dcal_{t+1}: = (\Dcal_t\setminus \{\phi(t)\})\cup \{t\}$}
  {$\Dcal_{t+1}:=\Dcal_t\cup\{t\}$}
  }
$t\gets t+1$
}
\hrule
\end{algorithm}

In this section, we present an optimistically universal algorithm and give a characterization of $\suol$. We start by defining our new learning rule $k-$Capped $1-$Nearest Neighbor ($k$C1NN). This is a simple variant of the traditional $1$NN learning rule where $k$C1NN performs the 1NN learning rule over a reduced training set. Recall that in the 1NN learning rule, we assign to the new input $X_t$ the value of the nearest neighbor $Y_{NN(t)}$ where $NN(t)=\arg\min_{u<t} \rho(X_t,X_u)$. We refer to the input point $X_{NN(t)}$ as the representant of the input value $X_t$. In the $k$C1NN learning rule, we keep in memory the number of times $n_t$ each point $X_t$ is used as representant for following input data and cap this value at $k$. Precisely, at each step $t$ we update the dataset $\Dcal_t\subset\{u,\; u<t\}$ containing the indices of data points on which 1NN may be performed. To do so, when $n_u$ reaches $k$ for some $u<t$, we delete $u$ from the current dataset $\Dcal_t$. At each iteration, if the input $X_t$ has already been visited, we use simple memorization to predict $Y_t$, we do not update the values $(n_u)_{u<t}$ and do not include $t$ in the dataset $\Dcal_{t+1}$. Otherwise $k$C1NN performs the 1NN learning rule on the current dataset $(X_u,Y_u)_{u\in \Dcal_t}$, where ties can be broken arbitrarily for instance with minimum index, and updates $(n_u)_{u\in \Dcal_t}$ and the dataset accordingly. The rule is formally described in Algorithm \ref{alg:kC1NN}.

In Section \ref{sec:nearest_neighbor} we presented a process $\Xbb$ on which nearest neighbor fails. The main reason for this failure is that some specific input points $X_t$ can be used an arbitrarily large number of times as representant for future points, thereby inducing a large number of prediction errors. The learning rule $k$C1NN is designed precisely to tackle this issue by ensuring that any datapoint $X_t$ for $t\geq 1$ is used at most $k$ times as representant i.e. $|\{u>t:\; \phi(u)=t\}|\leq k$. To provide a simpler exposition of the result, we now show that $k$C1NN is in fact optimistically universal for $k\geq 4$ starting with $\Xcal=[0,1]$. This will in turn give the result for general standard Borel space as shown in Section \ref{subsec:standard_borel} and already provides all the intuitions necessary for the general case presented in Section \ref{sec:borel_spaces}.

\subsection{Universal online learning on $\Xcal=[0,1]$}

We will consider the case $\Xcal=[0,1]$ in this section and show that 4C1NN is optimistically universal for this input space. To do so, we prove that 4C1NN is universally consistent under all processes in $\smv_{([0,1],|\cdot|)}$ which yields $\smv_{([0,1],|\cdot|)}\subset\suol_{([0,1],|\cdot|),(\{0,1\},\ell_{01})}$. Together with Proposition \ref{prop:smv_necessary}, this will show that $\suol_{([0,1],|\cdot|),(\{0,1\},\ell_{01})}=\smv_{([0,1],|\cdot|)}$ and as a result, that 4C1NN is optimistically universal. As a first step, we focus on the simple function $f^*$ represented by the fixed interval $[0,1/2]$ in the binary classification setting, and show that 4C1NN is consistent under any input process for this target function.

\begin{proposition}
\label{prop:consistent_interval}
Let $\Xcal=[0,1]$ with the usual topology. We consider the binary classification setting $\Ycal =\{0,1\}$ with $\ell_{01}$ binary loss. Under any input process $\Xbb\in \smv_{([0,1],|\cdot|)}$, the learning rule 4C1NN is strongly consistent for the target function $f^* = \1_{[0,1/2]}$.
\end{proposition}

\begin{proof}
We reason by the contrapositive and suppose that 4C1NN is not consistent on $f^*$. We will show that the process $\Xbb$ disproves the $\smv_{([0,1],|\cdot|)}$ condition by considering the partition $\Pcal$ of $\Xcal$ defined by 
\begin{equation*}
    \left\{\frac{1}{2}\right\} \cup \bigcup_{k\geq 1}\left[\frac{1}{2} - \frac{1}{2k}; \frac{1}{2} - \frac{1}{2(k+1)} \right)  \cup \bigcup_{k\geq 1}\left(\frac{1}{2} + \frac{1}{2(k+1)};\frac{1}{2}+  \frac{1}{2k}  \right].
\end{equation*}
Precisely, we will show that the process does not visit a sublinear number of sets of this partition with nonzero probability.

Because 4C1NN is not consistent, $\delta:=\mathbb P(\Lcal_\Xbb (4C1NN,f^*)>0)>0$. Define 
\begin{equation*}
    \Acal:=\{\Lcal_\Xbb (4C1NN,f^*)>0\}.
\end{equation*}
We now consider a specific realization $\mb x = (x_t)_{t\geq 0}$ of the process $\Xbb$ falling in the event $\Acal$. Note that $\mb x$ is not random anymore. We now show that $\mb x$ does not visit a sublinear number of sets in the partition $\Pcal$. By construction $\epsilon:=\Lcal_{\mb x} (4C1NN,f^*)>0$. We now denote by $(t_k)_{k\geq 1}$ the increasing sequence of all times when 4C1NN makes an error in the prediction of $f^*(x_t)$. Now define an increasing sequence of times $(T_l)_{l\geq 1}$ such that
\begin{equation*}
    \frac{1}{T_l}\sum_{t=1}^{T_l} \ell_{01}(4C1NN(\mb x_{<t},\mb y_{<t},x_t),f^*(x_t))> \frac{\epsilon}{2}.
\end{equation*}
For any $l\geq 1$ consider the last index $k = \max\{u,t_u\leq T_l\}$ when 4C1NN makes a mistake. Then we obtain $k > \frac{\epsilon}{2} T_l \geq \frac{\epsilon}{2} t_k$. Considering the fact that $(T_l)_{l\geq 1}$ is an increasing unbounded sequence we therefore obtain an increasing sequence of indices $(k_l)_{l\geq 1}$ such that $t_{k_l}<\frac{2k_l}{\epsilon}$.

At an iteration where the new input $x_t$ has not been previously visited we will denote by $\phi(t)$ the index of the nearest neighbor of the current dataset in the 4C1NN learning rule. Now let $l\geq 1$. We focus on the time $t_{k_l}$. Consider the tree $\Gcal$ where nodes are times $\Tcal:=\{t,\; t\leq t_{k_l},\; x_t\notin\{x_u, u<t\} \}$ for which a new input was visited, where the parent relations are given by $(t,\phi(t))$ for $t\in \Tcal\setminus\{1\}$. In other words, we construct the tree in which a new input is linked to its representant which was used to derive the target prediction. Note that by definition of the 4C1NN learning rule, each node has at most $4$ children and a node is not in the dataset at time $t_{k_l}$ when it has exactly $4$ children.

By symmetry, we will suppose without loss of generality that the majority of input points on which 4C1NN made a mistake belong to the first half $[0,\frac{1}{2}]$ i.e.
\begin{equation*}
    \left|\left\{t\leq t_{k_l},\; \ell_{01}(4C1NN(\mb x_{<t},\mb y_{<t},x_t),f^*(x_t))=1,\; x_t\in \left[0,1/2\right] \right\}\right| \geq \frac{k_l}{2}
\end{equation*}
or equivalently, $\left|\left\{k\leq k_l,\; x_{t_k}\leq \frac{1}{2}\right\}\right| \geq \frac{k_l}{2}$.

Let us now consider the subgraph $\tilde \Gcal$ given by restricting $\Gcal$ only to nodes in the first half-space $[0,1/2]$ which are mapped to the true value $1$ i.e. on times $\{t\in \Tcal,\; x_t\leq \frac{1}{2}\}$. In this subgraph, the only times with no parent are times $t_k$ with $k\leq k_l$ and $x_{t_k}\leq \frac{1}{2}$ and possibly time $t=1$. Indeed, if a time in $\tilde \Gcal$ has a parent $\phi(t)$ in $\tilde \Gcal$, the prediction of 4C1NN for $x_t$ returned the correct answer $1$. The converse is also true except for the root time $t=1$ which has no parent in $\Gcal$. Therefore, $\tilde \Gcal$ is a collection of disjoint trees with roots times $\{t_k, \; k\leq k_l, \; x_{t_k}\leq \frac{1}{2}\}$ (and possibly $t=1$). For a given time $t_k$ with $k\leq k_l$ and $x_{t_k}\leq \frac{1}{2}$, we will denote by $\Tcal_k$ the corresponding tree in $\tilde \Gcal$ with root $t_k$. We say that the $\Tcal_k$ is a \emph{good} tree if all times $t\in \Tcal_k$ of this tree are parent in $\Gcal$ to at most $1$ time from the second half-space $(\frac{1}{2},1]$ i.e. if 
\begin{equation*}
    \forall t\in \Tcal_k,\quad \left|\left\{u
    \leq t_{k_l},\; \phi(u) = t,\; x_u>\frac{1}{2}\right\}\right|\leq 1.
\end{equation*}
We denote by $G = \{k\leq k_l,\; x_{t_k}\leq \frac{1}{2},\; \Tcal_k \text{ good}\} $ the set of indices of good trees. By opposition, we will say that a tree is \emph{bad} otherwise. We now give a simple upper bound on $N_{\text{bad}}$ the number of bad trees. Note that for any time $t\in \Tcal_k$ of a tree, times in $\left\{u \leq t_{k_l},\; \phi(u) = t,\; x_u>\frac{1}{2}\right\}$ are when 4C1NN makes a mistake on the second-half $(\frac{1}{2},1]$. Therefore,
\begin{equation*}
    \sum_{k\leq k_l,\;  x_{t_k}\leq \frac{1}{2}} \sum_{t\in \Tcal_k} \left|\left\{u<t_{k_l},\; \phi(u) = t,\; x_u>\frac{1}{2}\right\}\right| \leq \left|\left\{k\leq k_l, x_{t_k}>\frac{1}{2}\right\}\right| \leq \frac{k_l}{2}
\end{equation*}
because by hypothesis $\left|\left\{k\leq k_l,\; x_{t_k}\leq \frac{1}{2}\right\}\right| \geq \frac{k_l}{2}$. Therefore, since each bad tree contains a node which is parent to at least $2$ times of mistake in $(\frac{1}{2},1]$, we obtain
\begin{equation*}
    N_{\text{bad}} \leq \sum_{k\leq k_l,\;  x_{t_k}\leq \frac{1}{2}} \sum_{t\in \Tcal_k} \frac{1}{2}\left|\left\{u<t_{k_l},\; \phi(u) = t,\; x_u>\frac{1}{2}\right\}\right| \leq  \frac{k_l}{4}.
\end{equation*}
Thus, the number of good trees is $|G|=\left|\left\{k\leq k_l,\; x_{t_k}\leq \frac{1}{2}\right\}\right|-N_{\text{bad}}\geq \frac{k_l}{4}$. We now focus on good trees only and analyze their relation with the final dataset $\Dcal_{t_{k_l}}$. Precisely, for a good tree $\Tcal_k$, denote $\Vcal_k = \Tcal_k\cap \Dcal_{t_{k_l}}$ the set of times which are present in the final dataset and belong to the tree induced by error time $t_k$. One can note that the sets $\{x_u,\; u\in \Vcal_k\}_{k\in G}$ are totally ordered:
\begin{equation*}
    \forall k_1<k_2\in G,\; \forall t_1\in \Tcal_{k_1},\; \forall t_2 \in \Tcal_{k_2},\quad x_{t_1} < x_{t_2}.
\end{equation*}
This can be shown by observing that at each iteration $t$ of 4C1NN, the following invariant is conserved: the sets $\{x_u,\; u\in \Tcal_k\cap \Dcal_t  \}_{k\in \{l\in G,\; t_l\leq t\}}$ are totally ordered. The induction follows from the fact that when a new input point is visited, 4C1NN performs the 1NN learning rule on the current dataset $\Dcal_l$. Therefore, either the sets $\{x_u,\; u\in \Tcal_k\cap \Dcal_t  \}_{k\in \{l\in G,\; t_l\leq t\}}$ are conserved, or a new point is added when $t=t_k$ for some $k\leq k_l$ which forms its own tree and is closest to $(\frac{1}{2},1]$ than all other sets $\{x_u,\; u\in \Tcal_k\cap \Dcal_t  \}_{k\in \{l\in G,\; t_l\leq t\}}$, or a new point is added to an existing tree $\Tcal_k$ in which case it should be closer to some time of $\Tcal_k\cap \Dcal_t$ than any time in $\Tcal_{k-1}\cap \Dcal_t$ or $\Tcal_{k+1}\cap\Dcal_t$---if $\Tcal_{k-1}$ or $\Tcal_{k+1}$ exist. Additionally, a time may be removed which is still consistent with the invariant. Last, we observe that these sets never run empty because a time is removed only when at least $3$ other points were added to the same set.

We now reason by induction to show that the sets $\{x_u,\; u\in \Vcal_k\}_{k\in G}$ are also well separated---in a multiplicative way. Let us order the good trees by $G=\{g_1<\ldots<g_{|G|}\}$ and start with tree $\Tcal_{g_1}$. Consider any leaf of this tree and the corresponding path to the root $p_l\to p_{l-1}\to p_0=t_{g_1}$ and define $x^1=\min_{1\leq i\leq l} x_{p_i}$. By construction, any point on this path is being replaced by its parent. Therefore, at any step of the algorithm 4C1NN at least one point on this path is available in the dataset $\Dcal_t$ for any $t\geq t_{g_1}$---for instance the last time $p_i$ such that $p_i\leq t$. This point $x^1$ provides a lower bound for the maximum point in $\{x_u,\; u\in \Tcal_{g_1}\cap \Dcal_t  \}$ which in turn will provide a lower bound for all points in $\{x_u,\; u\in \Tcal_{g_2}\cap \Dcal_t  \}$.

Let us now turn to $\Tcal_{g_2}$. By construction, in a good tree $\Tcal_k$, a time $t\in \Tcal_k$ which is not in the final dataset $\Dcal_{t_{k_l}}$ must be parent to at least $3$ other times within $\Tcal_k$. Therefore, until the minimal depth of an available time $\Vcal_{g_2} = \Tcal_{g_2}\cap \Dcal_{t_{k_l}}$ in the current dataset $\Dcal_{t_{k_l}}$, each node of the tree $\Tcal_{g_2}$ has at least $3$ parents which correspond necessarily to times $t>t_{g_2}$. Therefore, the minimal depth $d(g_2)$ of an available time $\Vcal_k$ in the current dataset satisfies 
\begin{equation*}
    \sum_{i=0}^{d(g_2)-1} 3^i\leq |\Tcal_{g_2}| \leq t_{k_l}.
\end{equation*}
Therefore $d(g_2) \leq \log_3 (2t_{k_l}  + 1)\leq \log_3 t_{k_l}.$ Now consider the specific path from this node in $\Vcal_{g_2}$ of minimal depth to the root $t_{g_2}$. Denote this path $p_{d(g_2)}\to p_{d(g_2)-1}\to p_0=t_{g_2}$. Each arc of this path represents the fact that at the corresponding iteration $p_i$ of 4C1NN, the parent $x_{p_{i-1}}$ was closer from $x_{p_{i}}$ than any other point of the current dataset $\Dcal_{p_i}$, in particular any point of $\{x_u,\; u\in \Tcal_{g_1}\cap \Dcal_{p_i}  \}$. This gives $|x_{p_{i-1}}-x_{p_{i}}|\leq |x^1-x_{p_{i-1}}| = x_{p_{i-1}}-x^1$ because we have $x_{p_{i-1}},x_{p_i}>x^1$. Therefore we obtain 
\begin{equation*}
    x_{p_{i-1}} \geq \frac{x^1+x_{p_i}}{2}.
\end{equation*}
Indeed, if this were not the case we would have $|x_{p_{i-1}}-x_{p_{i}}| = x_{p_i}-x_{p_{i-1}}>x_{p_{i-1}}-x^1$. Similarly, considering the fact that 4C1NN makes a mistake at time $t_{g_2}$, the parent of $t_{g_2}$ satisfies $x_{\phi(t_{g_2})} >  \frac{1}{2}$ which yields $x_{t_{g_2}}\geq \frac{x^1 + x_{\phi(t_{g_2})}}{2}\geq \frac{x^1 + \frac{1}{2}}{2}$. Hence, for any $0\leq i\leq d(g_2),$
\begin{equation*}
    x_{p_i}\geq x^1 \left(1-\frac{1}{2^i}\right) +  \frac{x_{t_{g_2}}}{2^{i}}\geq x^1 + \frac{x_{t_{g_2}}-x^1}{2^{d(g_2)}}\geq x^1 +\left(\frac{1}{2}-x^1\right)t_{k_l}^{-\frac{\log 2}{\log 3}}.
\end{equation*}
Again, at every iteration $t\geq t_{g_2}$ of 4C1NN, at least one of the points $x_{p_i}$ is available in the dataset $\Dcal_t$---for instance the last $x_{p_i}$ such that $p_i\leq t$. By total ordering, this $x^2:=\min_{0\leq i\leq d(g_2)} x_{p_i}$ provides a lower bound for all points $\{x_u,\; u\in \Tcal_{g_3}\cap \Dcal_t  \}$ whenever $t\geq t_{g_3}$. Hence, the lower bound $x^2$ acts as a new barrier: the equivalent of $x^1$ for the above argument with $\Tcal_{g_2}$.

For clarity, we precise the next iteration of the induction for $\Tcal_{g_3}$. The minimal depth $d(g_3)$ of an available time $\Vcal_{g_3}$ satisfies $d(g_3)\leq \log_3 (t_{k_l} - t_{g_3}+1) + 1$ using the same argument as above. Now consider the corresponding path in $\Tcal_{g_3}$ from this minimal depth node to the root $p_{d(g_3)}\to\ldots \to p_0 = t_{g_3}$. By definition of the 4C1NN learning rule, the parent $x_{p_{i-1}}$ was closer to $x_{p_i}$ than any point of $\{x_u,\; u\in \Tcal_{g_2}\cap \Dcal_t  \}$. By the previous step of the induction, we know that the maximum value of this set is at least $x^2$. Therefore, we obtain $|x_{p_{i-1}}-x_{p_{i}}|\leq |x^2-x_{p_i}| = x_{p_i}-x^2$. We recall that we also have $x_{p_{i-1}}\geq x^2$ and $x_{p_i}\geq x^2$. The same argument as above gives $x_{p_i}\geq \frac{x^2 + x_{p_{i-1}}}{2}$. Further, we obtain similarly $x_{t_{g_3}}\geq \frac{x^2+ x_{\phi(t_{g_3})}}{2} \geq \frac{x^2+ \frac{1}{2}}{2}$. Hence, for all $0\leq i\leq d(g_3)$,
\begin{equation*}
    x_{p_i} \geq  x^2 +  \frac{x_{t_{g_3}}-x^2}{2^{d(g_3)}} \geq x^2 +  \left(\frac{1}{2}-x^2\right)t_{k_l}^{-\frac{\log 2}{\log 3}}.
\end{equation*}
We denote $x^3 := \min_{0\leq i\leq d(g_3)} x_{p_i}$, which now acts as a lower barrier for the tree $\Tcal_{g_4}$ and we can apply the induction.

We complete this induction for $\Tcal_{g_3},\ldots,\Tcal_{g_{|G|}}$. This creates a sequence of distinct visited input points $(x^i)_{1\leq i\leq |G|}$ with $x^i\leq \frac{1}{2}$ such that for any $1\leq i<|G|$, $x^{i+1}\geq x^i + \left(\frac{1}{2}-x^i\right)t_{k_l}^{-\frac{\log 2}{\log 3}}$ i.e.
\begin{equation*}
    \frac{1}{2}-x^{i+1}\leq \left(\frac{1}{2}-x^i\right)\left(1-t_{k_l}^{-\frac{\log 2}{\log 3}}\right).
\end{equation*}
In particular, we can observe that $0\leq x^1<x^2<\ldots <x^{|G|}\leq \frac{1}{2}$. Further, recalling that we have $t_{k_l}<\frac{2k_l}{\epsilon}$, we get
\begin{equation*}
    \log \left(\frac{1}{2}-x^{i+1}\right) - \log \left(\frac{1}{2}-x^i\right) \leq  \log \left(1-t_{k_l}^{-\frac{\log 2}{\log 3}}\right) \leq  -t_{k_l}^{-\frac{\log 2}{\log 3}} \leq -  \left(\frac{\epsilon}{2k_l}\right)^{\frac{\log 2}{\log 3}},
\end{equation*}
for any $1\leq i\leq |G|-1$. We will now argue that most of these points $x^i$ fall in distinct sets of the type $[a_k,a_{k+1})$ where $a_k:=\frac{1}{2} - \frac{1}{2k}$ for $k\geq 1$. We observe that for any $k\geq 1$, we have by concavity $\log \left(\frac{1}{2} - a_{k+1}\right) - \log \left(\frac{1}{2} - a_{k}\right) = \log \left(1-\frac{1}{k+1}\right)\geq -\frac{\log 2}{k+1}$. Therefore, with $k^0 = \left\lceil\log 2 \cdot \left(\frac{2k_l}{\epsilon}\right)^{\frac{\log 2}{\log 3}}\right\rceil$, for any $k\geq k^0$ we have
\begin{equation*}
    \log \left(\frac{1}{2} - a_{k+1}\right) - \log \left(\frac{1}{2} - a_{k}\right)>- \left(\frac{\epsilon}{2k_l}\right)^{\frac{\log 2}{\log 3}}.
\end{equation*}
Therefore, for any $1\leq i\leq |G|-1$ such that $x^i>a_{k^0}$, $x^{i}$ and $x^{i+1}$ would lie in different sets of the type $[a_k,a_{k+1})$, $k\geq 1$. In fact because the sequence $(x^i)_{1\leq i\leq |G|}$ is increasing, if $x^{i^*}>a_{k^0}$ then all points $(x^i)_{i^*\leq i\leq |G|}$ lie in distinct sets of the type $[a_k,a_{k+1})$, $k\geq 1$. Recall that $|G|\geq \frac{k_l}{4}$. Denote $i^* = \lfloor\frac{k_l}{8} \rfloor$. Because $(k_l)_{l\geq 1}$ is an increasing sequence, we have
\begin{equation*}
    \log \left(\frac{1}{2}-x^{i^*}\right) \leq \log \left(\frac{1}{2}\right) - (i^*-1)\left(\frac{\epsilon}{2k_l}\right)^{\frac{\log 2}{\log 3}}   \underset{l\to\infty}{\sim} -c_\epsilon k_l^{1-\frac{\log 2}{\log 3}},
\end{equation*}
where $c_\epsilon:=\frac{1}{8}\left(\frac{\epsilon}{2}\right)^{\frac{\log 2}{\log 3}}$ is a constant. Therefore,
\begin{equation*}
    \log \left(\frac{1}{2}-a_{k^0}\right)=-\log (2k^0) \underset{l\to\infty}{\sim} -\frac{\log 2}{\log 3}\log k_l =o\left( \log \left(\frac{1}{2}-x^{i^*}\right)\right)
\end{equation*}
which shows that for some constant $l^0$ and any $l\geq l^0$ we have $a_{k^0}<x^{i^*}<\frac{1}{2}$. Hence, for any $l\geq l^0$, all the points $(x^i)_{i^*\leq i\leq |G|}$ lie in distinct sets of the partition and there are at least $|G|-\frac{k_l}{8}\geq \frac{k_l}{8}$ such points. Therefore, for any $l\geq l^0$,
\begin{equation*}
    |\{P\in \Pcal, \quad P\cap \mb x_{\leq t_{k_l}}\neq \emptyset\}| \geq \frac{k_l}{8} \geq \frac{\epsilon}{16} t_{k_l}.
\end{equation*}
Because $t_{k_l}\to \infty$ as $l\to\infty$, this shows that $   |\{P\in \Pcal,\quad P\cap \mb  x_{< T}\neq \emptyset\}| \neq o(T).$ Because this holds for any realization of the event $\Acal$, we obtained
\begin{equation*}
    \mathbb P ( |\{P\in \Pcal,\quad P\cap \Xbb_{< T}\neq \emptyset\}| = o(T) ) \leq \mathbb P (\Acal^c )=1-\delta<1.
\end{equation*}
This shows that $\Xbb\notin \smv_{([0,1],|\cdot|)}$ and ends the proof of the proposition.
\end{proof}
Note that using the same proof, we observe that the result from Proposition \ref{prop:consistent_interval} holds for all learning rules $k$C1NN with $k\geq 4$.

We are now ready to prove that 4C1NN is universally consistent under processes of $\smv_{([0,1],|\cdot|)}$ for the binary classification setting. Intuitively, we analyze the set of functions on which 4C1NN is consistent under a fixed process $\Xbb\in \smv_{([0,1],|\cdot|)}$ and show that this is a $\sigma$-algebra. Proposition \ref{prop:consistent_interval} will be useful to show that this $\sigma$-algebra contains all intervals and as a result is the complete Borel $\sigma$-algebra $\Bcal$ i.e. 4C1NN is universally consistent under $\Xbb$.

\begin{theorem}
\label{thm:4C1NN}
Let $\Xcal=[0,1]$ with the usual topology $\Bcal$. For the binary classification setting, the learning rule 4C1NN is universally consistent for all processes $\Xbb\in \smv_{([0,1],|\cdot|)}$.
\end{theorem}

\begin{proof}
let $\Xbb\in \smv_{([0,1],|\cdot|)}$. We will show that 4C1NN is universally consistent on $\Xbb$ by considering the set $\Scal_\Xbb$ of functions for which it is consistent. More precisely, since $\Ycal = \{0,1\}$ in the binary setting, all target functions can be described as $f^= \1_{ A_{f^*}}$ where $A_{f^*} = f^{<-1>}(\{1\})$ is a measurable set. In the following, we will refer interchangeably to the function $f^*$ or the set $A_{f^*}$, and define $\Scal_\Xbb$ using the corresponding sets:
\begin{equation*}
    \Scal_\Xbb:= \{A\in \Bcal,\quad \Lcal_\Xbb(4C1NN,\1_A)=0\quad (a.s.) \}
\end{equation*}
By construction we have $\Scal_\Xbb\subset\Bcal$. The goal is to show that in fact $\Scal_\Xbb = \Bcal$. To do so, we will show that $\Scal$ satisfies the following properties
\begin{itemize}
    \item $\emptyset\in \Scal_\Xbb$ and $\Scal_\Xbb$ contains all intervals $[0,s)$ with $0< s\leq 1$,
    \item if $A\in \Scal_\Xbb$ then $A^c\in \Scal_\Xbb$ (stable to complementary),
    \item if $(A_i)_{i\geq 1}$ is a sequence of disjoint sets of $\Scal_\Xbb$, then $\bigcup_{i\geq 1} A_i\in \Scal_\Xbb$ (stable to $\sigma-$additivity for disjoint sets),
    \item if $A,B\in \Scal_\Xbb$, then $A\cup B\in \Scal_\Xbb$ (stable to union).
\end{itemize}
Together, these properties show that $\Scal_\Xbb$ is a $\sigma-$algebra that contains all open intervals of $\Xcal=[0,1]$. Recall that by definition, $\Bcal$ is the smallest $\sigma-$algebra containing open intervals. Therefore we get $\Bcal\subset \Scal_\Xbb$ which proves the theorem. We now show the four properties.\\

We start by showing the invariance to complementary. Note that 4C1NN is invariant to labels and that the loss $\ell_{01}$ is symmetric. Therefore, if it achieves consistency for $\1_{A}$ it also achieves consistency for $\1_{ A^c}$. Indeed, at each step, 4C1NN will use the same representant for the prediction hence for any $t\geq 0$,
\begin{equation*}
   \ell_{01}(4C1NN(\mb x_{<t},\1_{\mb x_{<t}\in A},x_t),\1_{x_t\in A}) = \ell_{01}(4C1NN(\mb x_{<t},\1_{\mb x_{<t}\in A^c},x_t),\1_{x_t\in A^c}).
\end{equation*}

4C1NN is clearly consistent for $f^*= 0$. Therefore $\emptyset \in \Scal_\Xbb$. Now let $0<s \leq 1$. We will show that $[0,s)\in \Scal_\Xbb$. Proposition \ref{prop:consistent_interval} shows that $[0,\frac{1}{2}]\in \Scal_\Xbb$. In fact, one can note that the same proof shows that $[0,\frac{1}{2})\in \Scal_\Xbb$. Further, for any $0<s\leq 1$ using the same proof with the following partition centered in $s$,
\begin{equation*}
    \left\{s\right\} \cup \bigcup_{k\geq 1}\left[s\left(1 - \frac{1}{k}\right);s\left(1 - \frac{1}{k+1}\right) \right)  \cup \bigcup_{k\geq 1}\left(s+ \frac{1-s}{k+1};s+  \frac{1-s}{k}  \right]
\end{equation*}
shows that $[0,s],[0,s)\in \Scal_\Xbb$.\\

We now turn to the $\sigma-$additivity for disjoint sets. Let $(A_i)_{i\geq 1}$ is a sequence of disjoint sets of $\Scal_\Xbb$. We denote $A:= \bigcup_{i\geq 1} A_i$. We consider the target function $f^*= \1_{A}$. There are two types of statistical errors: errors of type 1 correspond to $X_t\in A$ and a predicted value $0$ while type 2 errors correspond to $X_t\notin A$ and a predicted value $1$. We then write the average loss in the following way,
\begin{equation*}
     \frac{1}{T}\sum_{t=1}^{T}\ell_{01}(4C1NN(\Xbb_{< t},\Ybb_{< t}, X_{t}), f^*(X_t)) = \frac{1}{T}\sum_{t=1}^{T} \1_{X_t\in A}  \1_{X_{\phi(t)}\notin A} + \frac{1}{T}\sum_{t=1}^{T}\1_{X_t\notin A}  \1_{X_{\phi(t)}\in A},
\end{equation*}
where the first term corresponds to type 1 errors and the second term corresponds to type 2 errors.

We suppose by contradiction that $\mathbb P(\Lcal_\Xbb(4C1NN,f^*)> 0):=\delta>0$ Therefore, there exists $\epsilon>0$ such that $\mathbb P(\Lcal_\Xbb(4C1NN,f^*)> \epsilon)\geq \frac{\delta}{2}$. We denote this event by $\Acal:=\{\Lcal_\Xbb(4C1NN,f^*)>\epsilon\}$. We first analyze the errors induced by one set $A_i$ only. We have 
\begin{align*}
    \frac{1}{T}\sum_{t=1}^{T} ( \1_{X_t\in A_i }  \1_{X_{\phi(t)}\notin A} + \1_{X_t\notin A}  \1_{X_{\phi(t)}\in A_i}) & \leq \frac{1}{T}\sum_{t=1}^{T} (\1_{X_t\in A_i }  \1_{X_{\phi(t)}\notin A_i} + \1_{X_t\notin A_i}  \1_{X_{\phi(t)}\in A_i})\\
    &= \frac{1}{T}\sum_{t=1}^{T} \ell_{01}(4C1NN(\Xbb_{< t},\1_{\Xbb_{< t}\in A_i}, X_{t}), \1_{X_t\in A_i}).
\end{align*}
Then, because 4C1NN is consistent for $\1_{ A_i}$, we have
\begin{equation*}
    \frac{1}{T}\sum_{t=1}^{T} ( \1_{X_t\in A_i }  \1_{X_{\phi(t)}\notin A} + \1_{X_t\notin A}  \1_{X_{\phi(t)}\in A_i}) \to 0 \quad (a.s.).
\end{equation*}
We take $\epsilon_i = \frac{\epsilon}{4\cdot 2^i}$ and $\delta_i = \frac{\delta}{8\cdot 2^i}$. The above equation gives
\begin{equation*}
    \mathbb P\left[\bigcup_{t_0\geq 1} \bigcap_{T\geq t_0} \left\{\frac{1}{T}\sum_{t=1}^{T} ( \1_{X_t\in A_i }  \1_{X_{\phi(t)}\notin A} + \1_{X_t\notin A}  \1_{X_{\phi(t)}\in A_i}) <\epsilon_i \right\} \right] = 1.
\end{equation*}
Therefore, let $T^i$ such that
\begin{equation*}
     \mathbb P\left[ \bigcap_{T\geq T^i} \left\{\frac{1}{T}\sum_{t=1}^{T} ( \1_{X_t\in A_i }  \1_{X_{\phi(t)}\notin A} + \1_{X_t\notin A}  \1_{X_{\phi(t)}\in A_i}) <\epsilon_i \right\} \right] \geq 1-\delta_i.
\end{equation*}
We will denote by $\Ecal_i$ this event. We now consider the scale of the process $\Xbb_{\leq T^i}$ when falling in $A_i$, by introducing $\eta_i>0$ such that
\begin{equation*}
    \mathbb P\left[ \min_{ \substack{
        t_1,t_2 \leq T^i;\; X_{t_1},X_{t_2}\in A_i; \\
        X_{t_1}\neq X_{t_2}
     }} |X_{t_1}-X_{t_2}| > \eta_i \right]  \geq 1-\delta_i.
\end{equation*}
We denote by $\Fcal_i$ this event. By the union bound, we have $\mathbb P(\bigcup_{i\geq 1} \Ecal_i^c\cup\bigcup_{i\geq 1} \Fcal_i^c)\leq \frac{\delta}{4}$. Therefore, we obtain $\mathbb P(\Acal\cap \bigcap_{i\geq 1} \Ecal_i\cap\bigcap_{i\geq 1} \Fcal_i)\geq \mathbb P(\Acal) - \mathbb P(\bigcup_{i\geq 1} \Ecal_i^c\cup\bigcup_{i\geq 1} \Fcal_i^c) \geq \frac{\delta}{4}$. We now construct a partition $\Pcal$ obtained by subdividing each set $A_i$ according to scale $\eta_i$. For simplicity, we use the notation $N_i = \lfloor \frac{1}{\eta_i}\rfloor$ and construct the partition given of $\Xcal=[0,1]$ given by
\begin{equation*}
    \Pcal\; : \quad A^c\cup \bigcup_{i\geq 1} \left\{([N_i\eta_i,1]\cap A_i ) \cup  \bigcup_{j=0}^{N_i-1} \left([j\eta_i,(j+1)\eta_i)\cap A_i\right) \right\}.
\end{equation*}
Let us now consider a realization of $\mb x$ of $\Xbb$ in the event $\Acal\cap \bigcap_{i\geq 1} \Ecal_i\cap\bigcap_{i\geq 1} \Fcal_i$. The sequence $\mb x$ is now not random anymore. Our goal is to show that $\mb x$ does not visit a sublinear number of sets in the partition $\Pcal$.

By construction, the event $\Acal$ is satisfied, therefore there exists an increasing sequence of times $(t_k)_{k\geq 1}$ such that for any $k\geq 1$, $\frac{1}{t_k}\sum_{t=1}^{t_k}\ell_{01}(4C1NN(\mb x_{< t},\1_{\mb x_{< t} \in A}, x_{t}), \1_{x_t\in A}) > \frac{\epsilon}{2}.$ Therefore, we obtain for any $k\geq 1$,
\begin{equation*}
   \sum_{i\geq 1} \frac{1}{t_k}\sum_{t=1}^{t_k} ( \1_{x_t\in A_i }  \1_{x_{\phi(t)}\notin A} + \1_{x_t\notin A}  \1_{x_{\phi(t)}\in A_i}) > \frac{\epsilon}{2}.
\end{equation*}
Also, because the events $\Ecal_i$ are met, we have
\begin{equation*}
    \sum_{i\geq 1;\; t_k\geq T^i} \frac{1}{t_k}\sum_{t=1}^{t_k} ( \1_{x_t\in A_i }  \1_{x_{\phi(t)}\notin A} + \1_{x_t\notin A}  \1_{x_{\phi(t)}\in A_i}) < \sum_{i\geq 1, t_k\geq T^i} \epsilon_i \leq \frac{\epsilon}{4}.
\end{equation*}
Combining the two above equations gives
\begin{equation}
\label{eq:main}
    \frac{1}{t_k}  \sum_{t=1}^{t_k} \sum_{i\geq 1;\; t_k<T^i} ( \1_{x_t\in A_i }  \1_{x_{\phi(t)}\notin A} + \1_{x_t\notin A}  \1_{x_{\phi(t)}\in A_i}) >\frac{\epsilon}{4}.
\end{equation}
We now consider the set of times such that an input point fell into the set $A_i$ with $T^i>t_k$, either creating a mistake in the prediction of 4C1NN or inducing a later mistake within time horizon $t_k$: $ \Tcal:= \bigcup_{i\geq 1;\; T^i>t_k} \Tcal_i$ where 
\begin{equation*}
    \Tcal_i:= \left\{t\leq t_k,\; x_t\in A_i,\; \left(x_{\phi(t)}\notin A \text{ or }\exists t<u\leq t_k \text{ s.t. }\phi(u)=t,\; x_u\notin A\right)\right\}.
\end{equation*}
We now show that all points $x_t$ for $t\in \Tcal$ fall in distinct sets of the partition $\Pcal$. Indeed, because the sets $A_i$ are disjoint, it suffices to check that for any $i\geq 1$ such that $T^i>t_k$, the points $x_t$ for $t\in \Tcal_i$ fall in distinct of the following sets
\begin{equation*}
    [N_i\eta_i,1]\cap A_i ,\quad   [j\eta_i,(j+1)\eta_i)\cap A_i,\quad 0\leq j\leq N_i-1.
\end{equation*}
Note that for any $t_1< t_2\in \Tcal_i$ we have $x_{t_1},x_{t_2}\in A_i$ and $x_{t_1}\neq x_{t_2}$. Indeed, we cannot have $x_{t_2}=x_{t_1}$ otherwise 4C1NN would make no mistake at time $t_2$ and $x_{t_2}$ would induce no future mistake either (recall that if an input point was already visited, we use simple memorization for the prediction and do not add it to the dataset). Therefore, because the event $\Fcal_i$ is satisfied, for any $t_1< t_2\in \Tcal_i$ we have $|x_{t_1}-x_{t_2}|>\eta_i$. Hence $x_{t_1}$ and $x_{t_2}$ lie in different sets among $[N_i\eta_i,1]\cap A_i$ or $[j\eta_i,(j+1)\eta_i)\cap A_i$ for $0\leq j\leq N_i-1$. This shows that all points $\{x_t,\; t\in \Tcal\}$ lie in different sets of the partition $\Pcal$. Therefore,
\begin{equation*}
    |\{P\in \Pcal, P\cap \mb x_{\leq t_k}\neq \emptyset\}| \geq |\Tcal|.
\end{equation*}
We now lower bound $|\Tcal|$, which will uncover the main interest of the learning rule 4C1NN. Intuitively, this learning rule prohibits a single input point $x_t$ to induce a large number of mistakes in the learning process. Indeed, any input point incurs at most $1+4=5$ mistakes while this number of mistakes incurred by a single point can potentially by unbounded for the traditional 1NN learning rule. We now formalize this intuition.

\begin{align*}
    \sum_{t=1}^{t_k} \sum_{i\geq 1;\; t_k<T^i} ( \1_{x_t\in A_i }  \1_{x_{\phi(t)}\notin A} &+ \1_{x_t\notin A}  \1_{x_{\phi(t)}\in A_i}) \\
    &= \sum_{t=1}^{t_k} \sum_{i\geq 1;\; t_k<T^i} \left( \1_{x_t\in A_i }  \1_{x_{\phi(t)}\notin A} + \sum_{t<u\leq t_k}\1_{x_u\notin A}  \1_{x_t\in A_i}\1_{\phi(u)=t}\right)\\ 
    & = \sum_{i\geq 1;\; T^i>t_k} \sum_{t\leq t_k,\;x_t\in A_i}  \left(\1_{x_{\phi(t)}\notin A} + \sum_{t<u\leq t_k} \1_{x_u\notin A}\1_{\phi(u)=t}\right)\\
    &\leq \sum_{i\geq 1;\; T^i>t_k} \sum_{t\leq t_k,\;x_t\in A_i} 5\max\left(\1_{x_{\phi(t)}\notin A},\1_{x_u\notin A}\1_{\phi(u)=t},\; t<u\leq t_k\right)\\
    &= 5|\Tcal|
\end{align*}
where in the last inequality we used the fact that a given time $t$ can have at most $4$ children i.e. $|\{u>t, \phi(u)=t\}|\leq 4$ with the 4C1NN learning rule. We now use Equation (\ref{eq:main}) to obtain
\begin{equation*}
     |\{P\in \Pcal, P\cap \mb x_{\leq t_k}\neq \emptyset\}| \geq |\Tcal| \geq \frac{\epsilon}{20}t_k.
\end{equation*}
This holds for any $k\geq 1$. Therefore, because $t_k\to\infty$ as $k\to\infty$ we get $|\{P\in \Pcal, P\cap \mb x_{\leq T}\neq \emptyset\}|\neq o(T).$ Finally, this holds for any realization of $\Xbb$ in the event $\Acal\cap \bigcap_{i\geq 1} \Ecal_i\cap\bigcap_{i\geq 1} \Fcal_i$. Therefore,
\begin{equation*}
    \mathbb P(|\{P\in \Pcal, P\cap \mb x_{\leq T}\neq \emptyset\}|=o(T) )\leq \mathbb P\left[\left(\Acal\cap \bigcap_{i\geq 1} \Ecal_i\cap\bigcap_{i\geq 1} \Fcal_i\right)^c\right] \leq 1-\frac{\delta}{4}<1.
\end{equation*}
Therefore, $\Xbb\notin\smv_{([0,1],|\cdot|)}$ which contradicts the hypothesis. This concludes the proof that 
\begin{equation*}
    \Lcal_\Xbb(4C1NN,\1_{\cdot\in A})=0 \quad (a.s.),
\end{equation*} 
and hence, $\Scal_\Xbb$ satisfies the $\sigma-$additivity property for disjoint sets.\\

Note that the choice of disjoint sets for the proof of $\sigma-$additivity was made for convenience so that the partition defined is not too complex. However to complete the proof of the $\sigma-$additivity of $\Scal_\Xbb$, we have to prove that we can take unions of sets. Let $A_1,A_2\in \Scal_\Xbb$. We consider $A=A_1\cup A_2$ and $f^*(\cdot)=\1_{\cdot\in A}$. Using the same arguments as above, we still have for $T\geq 1$,
\begin{equation*}
    \frac{1}{T}\sum_{t=1}^{T} ( \1_{X_t\in A_i }  \1_{X_{\phi(t)}\notin A} + \1_{X_t\notin A}  \1_{X_{\phi(t)}\in A_i}) \to 0 \quad (a.s.).
\end{equation*}
for $i\in\{1,2\}$. But note that for any $T\geq 1$,
\begin{align*}
     \frac{1}{T}\sum_{t=1}^{T} &\ell_{01}(4C1NN(\Xbb_{< t},\Ybb_{< t}, X_{t}), f^*(X_t))\\
     &= \frac{1}{T}\sum_{t=1}^{T} \1_{X_t\in A}  \1_{X_{\phi(t)}\notin A} + \frac{1}{T}\sum_{t=1}^{T}\1_{X_t\notin A}  \1_{X_{\phi(t)}\in A}\\
     &\leq \frac{1}{T}\sum_{t=1}^{T} (\1_{X_t\in A_1}+\1_{X_t\in A_2})  \1_{X_{\phi(t)}\notin A} + \frac{1}{T}\sum_{t=1}^{T}\1_{X_t\notin A}  (\1_{X_{\phi(t)}\in A_1}+\1_{X_{\phi(t)}\in A_2})\\
     &= \sum_{i=1}^2 \frac{1}{T}\sum_{t=1}^{T} ( \1_{X_t\in A_i }  \1_{X_{\phi(t)}\notin A} + \1_{X_t\notin A}  \1_{X_{\phi(t)}\in A_i}) .
\end{align*}
Therefore we obtain directly $\Lcal_\Xbb(4C1NN,\1_{\cdot\in A})=0 \quad (a.s.)$. This shows that $A_1\cup A_2\in \Scal_\Xbb$ and ends the proof of the theorem.
\end{proof}
As an immediate consequence of Theorem \ref{thm:4C1NN} and Proposition \ref{prop:smv_necessary}, we obtain the following results.

\begin{theorem}
\label{thm:suol=smv_special_case}
$\suol_{([0,1],|\cdot|),(\{0,1\},\ell_{01})}=\smv_{([0,1],|\cdot|)}$.
\end{theorem}
\begin{theorem}
\label{thm:4C1NN_optimistically_universal_special_case}
For $\Xcal=[0,1]$ with usual measure, and for binary classification, 4C1NN is an optimistically universal learning rule.
\end{theorem}

\subsection{Generalization to standard Borel input spaces and separable bounded output spaces.}
\label{subsec:standard_borel}
The specific choice of input space $\Xcal=[0,1]$ and binary classification for output setting is in fact not very restrictive. Indeed, any standard Borel input space $\Xcal$ can be reduced to either $[0,1]$ or a countable set through the Kuratowski theorem. We recall that two standard Borel spaces i.e. complete separable Borel spaces, are Borel isomorphic if there exists a measurable bijection between them.

\begin{theorem}[Kuratowski's theorem]
\label{thm:kuratowski}
Any standard Borel space $\Xcal$ is Borel isomorphic to one of (1) $\Rbb$, (2) $\Nbb$ or (3) a finite space.
\end{theorem}
This classical result can be found for example in \cite{kechris2012classical} (Section 15.B). Further, any bounded output setting $(\Ycal,\ell)$ can be reduced to binary classification \citep{blanchard2021universal}.

\begin{theorem}[\citet{blanchard2021universal}]
\label{thm:invariance}
Let $\Xcal$ be a Borel space and $k\geq 2$. For any separable near-metric space $(\Ycal, \ell)$ with $0<\bar \ell<\infty$, we have $\suol_{(\Xcal,\rho),(\Ycal,\ell)}=\suol_{(\Xcal,\rho),(\{0,1\},\ell_{01})}$. Further, if there exists an optimistically universal for the binary classification setting, then there exists an optimistically universal for the setting $(\Ycal,\ell)$. Finally, if $k$C1NN is optimistically universal for binary classification, it is also optimistically universal for the setting $(\Ycal,\ell)$.
\end{theorem}

Using these two reductions, we can generalize Theorem \ref{thm:suol=smv_special_case} and Theorem \ref{thm:4C1NN_optimistically_universal_special_case} to any standard Borel space $\Xcal$ and any separable bounded setting $(\Ycal,\ell)$.

\begin{corollary}\label{cor:suol}
For any standard Borel space $\Xcal$ and any separable near-metric output space $(\Ycal,\ell)$ with $0<\bar \ell<\infty$, we have $\suol_{(\Xcal,\rho),(\Ycal,\ell)}=\smv_{(\Xcal,\rho)}$.
\end{corollary}

\begin{corollary}\label{cor:opt}
For any standard Borel space $\Xcal$, and any separable near-metric output space $(\Ycal,\ell)$ with bounded loss, there exists an optimistically universal learning rule.
\end{corollary}

\begin{proof}[of Corollary \ref{cor:suol} and \ref{cor:opt}]
Using Theorem \ref{thm:invariance} directly gives the result for $\Xcal=[0,1]$ and any bounded separable near-metric output space. The results are already known when $\Xcal$ is countable and in these cases, memorization is an optimistically universal learning rule \cite{hanneke2021learning}. We now fix a bounded separable ouptput setting $(\Ycal, \ell)$ and a standard Borel space $\Xcal$, Borel isomorphic to $\Rbb$ and as a result Borel isomorphic to $[0,1]$. Let $g:\Xcal\to[0,1]$ be a measurable bijection and a process $\Xbb\in \smv_{(\Xcal,\rho)}$. Note that the process $g(\Xbb):=(g(X_t))_{t\geq 1}$ belongs to $\smv_{([0,1],|\cdot|)}$ by bi-measurability of $g$. We can construct the learning rule $f_\cdot$ for value setting $\Xcal$ and output setting $(\Ycal,\ell)$ such that for any $\mb x_{\leq t}\in \Xcal^t$ and $\mb y_{<t}\in \Ycal^{t-1}$ we define $f_t(x_{<t},y_{<t},x_t) = {4C1NN}_t(g(x_{<t}),y_{<t},g(x_t)).$ By construction, for target function $f^*:\Xcal\to\Ycal$ this learning rule under $\Xbb$ has same losses as 4C1NN under $g(\Xbb)$ for the target function $f^*\circ g^{-1}$. Therefore, $f_\cdot$ is universally consistent under $\Xbb$ which yields $\smv_{(\Xcal,\rho)}\subset \suol_{(\Xcal,\rho),(\Ycal,\ell)}$. Using Proposition \ref{prop:smv_necessary} we have $\suol_{(\Xcal,\rho),(\Ycal,\ell)}=\smv_{(\Xcal,\rho)}$. We can also end the proof of Corollary \ref{cor:opt} by noting that $f_\cdot$ is an optimistically universal learning rule.
\end{proof}

Although quite intuitive and direct, this generalization has two limitations. First, it only applies to standard Borel spaces instead of general separable Borel spaces. Second, it does not provide a practical optimistically universal rule in general. Indeed, the constructed optimistically universal learning rule in Corollary \ref{cor:opt} uses a bimeasurable bijection between $\Xcal$ and $[0,1]$---in the non-trivial case where $\Xcal$ is Borel isomorphic to $\Rbb$---which can be very complex and non-intuitive. For instance, the constructed learning rule for $[0,1]^2$ is not 4C1NN but instead a complex learning rule using a measurable bijection $[0,1]\to [0,1]^2$. In the next section we solve these two issues by showing that 2C1NN is optimistically universal in the general case.

\section{Generalization to all Borel spaces}
\label{sec:borel_spaces}

In this section we extend Corollary \ref{cor:suol} and \ref{cor:opt} to the general case where $\Xcal$ is a separable Borel space and $(\Ycal,\ell)$ is a separable near-metric space with bounded loss using a similar proof structure. We show that 2C1NN is in fact always optimistically universal. We begin by showing the following lemma.

\begin{lemma}
\label{lemma:path_recursion}
Consider two distinct paths $p_{d}\to p_{d-1}\to\ldots\to p_1\to p_0$ and $q_{f}\to q_{f-1}\to\ldots\to q_1\to q_0$ i.e. $\phi(p_i) = p_{i-1}$ for $1\leq i\leq d$ and $\phi(q_i)=q_{i-1}$ for $1\leq i\leq f$. Suppose $p_0<q_0$ and that there exists $t\geq\max(p_d,q_f)$ such that $p_d,q_f\in \Dcal_{t}$ in other words the two end times are in some final dataset. Then, with $v(0) := \max\{0\leq i\leq d,\; p_i< q_0\}$ we have
\begin{equation*}
    \rho(x_{p_{v(0)}},x_{q_0})\leq 2^{f+d+1}\rho(x_{p_d},x_{q_f}) \text{ and} \quad  \rho(x_{p_{v(0)}},x_{p_d})\leq 2^{f+d+1}\rho(x_{p_d},x_{q_f}).
\end{equation*}
\end{lemma}

\begin{proof}
Define
\begin{equation*}
    \begin{array}{ll}
    v(j):=\max \{0\leq i\leq d,\; p_i<q_j\}, & j=0,\ldots,f.\\
    u(i):=\max \{0\leq j\leq f,\; q_j<p_i\}, & i=v(0)+1,\ldots,d,\\
    \end{array}
\end{equation*}
Now observe that for any $v(0)+1\leq i\leq d$, we have $q_{u(i)}\in\Dcal_{p_i}$ i.e. the datapoint $q_{u(i)}$ is available in the current dataset. Indeed, it is possibly removed after all of its children have been revealed, in particular $q_{u(i)+1}$ if it exists. By definition of $u(i)$, even if $q_{u(i)+1}$ exists, it has not yet been revealed since $p_i<q_{u(i)+1}$. Therefore, we have $\rho(x_{p_i},x_{p_{i-1}}) \leq \rho(x_{p_i},x_{q_{u(i)}}).$ Similarly, we have for all $1\leq j\leq f$, $\rho(x_{q_j},x_{q_{j-1}}) \leq \rho(x_{q_j},x_{p_{v(j)}}).$ We now take $v_0+1\leq i< d$. We have $q_{u(i)}<p_{v(u(i))+1}<\ldots < p_i < q_{u(i)+1} <\ldots < q_{u(i+1)}<p_{i+1}$ (where some terms might not exist). Therefore,
\begin{align*}
    \rho(x_{p_i},x_{q_{u(i)}}) &\leq \rho(x_{p_i},x_{p_{i+1}}) + \rho ( x_{p_{i+1}},x_{q_{u(i+1)}}) + \rho(x_{q_{u(i)}},x_{q_{u(i+1)}}) \\
    &\leq 2\rho ( x_{p_{i+1}},x_{q_{u(i+1)}}) + \sum_{w=u(i)}^{u(i+1)-1} \rho(x_{q_w},x_{q_{w+1}})\\
    &\leq 2\rho ( x_{p_{i+1}},x_{q_{u(i+1)}}) + \sum_{w=u(i)}^{u(i+1)-1} \rho(x_{p_i},x_{q_{w+1}})
\end{align*}
where in the last inequality, we used the fact that for all $u(i)\leq w\leq u(i+1)-1$, we have $v(w+1)=i$. Now observe that for any $u(i)+1\leq w\leq u(i+1)-1$,
\begin{equation*}
    \rho(x_{p_i},x_{q_w}) \leq \rho(x_{p_i},x_{q_{w+1}}) + \rho(x_{q_w},x_{q_{w+1}}) \leq 2\rho(x_{p_i},x_{q_{w+1}}).
\end{equation*}
Therefore we have by induction $\rho(x_{p_i},x_{q_w}) \leq 2^{u(i+1)-w} \rho(x_{p_i},x_{q_{u(i+1)}})$. which yields
\begin{equation*}
    \rho(x_{p_i},x_{q_{u(i)}}) \leq 2\rho ( x_{p_{i+1}},x_{q_{u(i+1)}}) + (2^{u(i+1)-u(i)}-1)\rho(x_{p_i},x_{q_{u(i+1)}}).
\end{equation*}
Finally, we observe that $\rho(x_{p_i},x_{q_{u(i+1)}})\leq \rho(x_{p_i},x_{p_{i+1}}) + \rho(x_{p_{i+1}},x_{q_{u(i+1)}})\leq 2\rho(x_{p_{i+1}},x_{q_{u(i+1)}})$. Hence,
\begin{equation*}
    \rho(x_{p_i},x_{q_{u(i)}}) \leq 2^{u(i+1)-u(i)+1}\rho ( x_{p_{i+1}},x_{q_{u(i+1)}}).
\end{equation*}
By recursion, this yields
\begin{equation*}
    \rho(x_{p_{v(0)+1}},x_{q_{u(v(0)+1)}}) \leq 2^{u(d)-u(v(0)+1)+(d-v(0)-1)}\rho(x_{p_d},x_{q_{u(d)}}).
\end{equation*}
We now relate the quantity $\rho(x_{p_{v(0)+1}},x_{q_{u(v(0)+1)}})$ (resp. $\rho(x_{p_d},x_{q_{u(d)}})$) to $\rho(x_{p_{v(0)}},x_{q_0})$ (resp. $\rho(x_{p_d},x_{q_d})$). We have by construction $p_{v(0)}<q_0<q_1<\ldots < q_{u(v(0)+1)}<p_{v(0)+1}$. Therefore, similarly to before,
\begin{align*}
    \rho(x_{p_{v(0)}},x_{q_0})&\leq \rho(x_{p_{v(0)}},x_{p_{v(0)+1}}) + \rho(x_{p_{v(0)+1}},x_{q_{u(v(0)+1)}}) + \sum_{w=0}^{u(v(0)+1)-1} \rho(x_{q_w},x_{q_{w+1}}) \\
    &\leq 2\rho(x_{p_{v(0)+1}},x_{q_{u(v(0)+1)}}) + \sum_{w=1}^{u(v(0)+1)} \rho(x_{p_{v(0)}},x_{q_{w}}).
\end{align*}
But $\rho(x_{p_{v(0)}},x_{q_{w}}) \leq \rho(x_{p_{v(0)}},x_{q_{w+1}}) + \rho(x_{q_{w}},x_{q_{w+1}}) \leq 2\rho(x_{p_{v(0)}},x_{q_{w+1}})$. Hence $\rho(x_{p_{v(0)}},x_{q_{w}}) \leq 2^{u(v(0)+1)-w}\rho(x_{p_{v(0)}},x_{q_{u(v(0)+1)}})\leq 2^{u(v(0)+1)-w+1}\rho(x_{p_{v(0)+1}},x_{q_{u(v(0)+1)}}) $. Then,
\begin{equation*}
    \rho(x_{p_{v(0)}},x_{q_0}) \leq 2^{u(v(0)+1)+1} \rho(x_{p_{v(0)+1}},x_{q_{u(v(0)+1)}}) \leq 2^{u(d)+(d-v(0))}\rho(x_{p_d},x_{q_{u(d)}}).
\end{equation*}
Finally, we have $q_{u(d)}<p_d<q_{u(d)+1}<\ldots<q_f$. Then,
\begin{align*}
    \rho(x_{p_d},x_{q_{u(d)}}) &\leq \sum_{w=u(d)}^{f-1} \rho(x_{q_w},x_{q_{w+1}}) + \rho(x_{p_d},x_{q_f})\\
    &\leq \sum_{w=u(d)}^{f-1} \rho(x_{p_d},x_{q_{w+1}}) + \rho(x_{p_d},x_{q_f}).
\end{align*}
Again, note that for $u(d)+1\leq w\leq f-2$, we have $\rho(x_{p_d},x_{q_{w}})\leq \rho(x_{p_d},x_{q_{w+1}}) + \rho(x_{q_{w}},x_{q_{w+1}})\leq 2\rho(x_{p_d},x_{q_{w+1}})$. Hence, $\rho(x_{p_d},x_{q_{w}}) \leq 2^{f-w}\rho(x_{p_d},x_{q_f})$ and we obtain
\begin{align*}
    \rho(x_{p_d},x_{q_{u(d)}}) &\leq (2^{f-u(d)}-1) \rho(x_{p_d},x_{q_f}) + \rho(x_{p_d},x_{q_f})  = 2^{f-u(d)} \rho(x_{p_d},x_{q_f}).
\end{align*}
Putting everything together yields
\begin{equation*}
    \rho(x_{p_{v(0)}},x_{q_0}) \leq 2^{f+d} \rho(x_{p_d},x_{q_f}).
\end{equation*}
Finally, we compute
\begin{align*}
    \rho(x_{p_{v(0)}},x_{p_d}) &\leq \sum_{i=v(0)+1}^{d} \rho(x_{p_{i-1}},x_{p_i})\\
    &\leq \sum_{i=v(0)+1}^{d} \rho(x_{p_{i}},x_{q_{u(i)}})\\
    &\leq \sum_{i=v(0)+1}^d 2^{u(d)-u(i)+d-i} \rho(x_{p_d},x_{q_{u(d)}})\\
    &\leq \sum_{i=v(0)+1}^d 2^{u(d)-u(v(0)+1)+d-i} \rho(x_{p_d},x_{q_{u(d)}})\\
    &\leq 2^{u(d)-u(v(0)+1)+d-v(0)}\rho(x_{p_d},x_{q_{u(d)}})\\
    &\leq 2^{f-u(v(0)+1)+d-v(0)}\rho(x_{p_d},x_{q_f})\\
    &\leq 2^{f+d} \rho(x_{p_d},x_{q_f}).
\end{align*}
This ends the proof of the lemma.
\end{proof}

We are now ready to show that 2C1NN is consistent on functions representing balls of the metric $\rho$, under any process in $\smv_{(\Xcal,\rho)}$.

\begin{proposition}
\label{prop:consistent_ball_borel}
Let $(\Xcal,\Bcal)$ be a separable Borel space constructed from the metric $\rho$. We consider the binary classification setting $\Ycal =\{0,1\}$ and the $\ell_{01}$ binary loss. For any input process $\Xbb\in \smv_{(\Xcal,\rho)}$, for any $x\in \Xcal$, and $r>0$, the learning rule 2C1NN is consistent for the target function $f^*= \1_{B_\rho(x,r)}$.
\end{proposition}

\begin{proof}
We fix $\bar x\in \Xcal$, $r>0$ and $f^* = \1_{ B(\bar x,r)}$. We reason by the contrapositive and suppose that 2C1NN is not consistent on $f^*$. We will show that the process $\Xbb$ disproves the $\smv_{(\Xcal,\rho)}$ condition by considering a partition for which, the process $\Xbb$ does not visit a sublinear number of sets with nonzero probability.

Because 2C1NN is not consistent, $\delta:=\mathbb P(\Lcal_\Xbb (2C1NN,f^*)>0)>0$. Therefore, there exists $0<\epsilon \leq 1$ such that $\mathbb P(\Lcal_\Xbb (2C1NN,f^*)> \epsilon)>\frac{\delta}{2}$.
Denote $\Acal:=\{\Lcal_\Xbb (2C1NN,f^*)>\epsilon\}.
$ We therefore have $\Pbb (\Acal)>\frac{\delta}{2}$. We now define a partition $\Pcal$. Because $\Xcal$ is separable, there exists a sequence $(x^i)_{i\geq 1}$ of elements of $\Xcal$ which is dense i.e.
\begin{equation*}
    \forall x \in \Xcal, \quad \inf_{i\geq 1} \rho(x,x^i) = 0.
\end{equation*}

We focus for now on the sphere $S(\bar x,r)$ and for any $\tau>0$ we take $ (P_i(\tau))_{i\geq 1}$ the sequence of sets included in $S(\bar x,r)$ defined by
\begin{equation*}
    P_i(\tau) := \left(S(\bar x,r)\cap B(x^i,\tau)\right) \setminus \left(\bigcup_{1\leq j<i} B(x^j,\tau)\right) .
\end{equation*}
These sets are disjoint. Further, they partition $S(\bar x,r)$. Indeed, if $x\in S(\bar x,r)$, let $i\geq 1$ such that $\rho(x,x^i)\leq \tau$. Then, $x\in S(\bar x,r)\cap B(x^i,\tau) \subset \bigcup_{j\leq i}P_j^\tau$. We now pose
\begin{equation*}
    \tau_l := c_\epsilon\cdot \frac{r}{2^{l+1}},
\end{equation*}
for $l\geq 1$, where $c_\epsilon := \frac{1}{2\cdot 2^{2^5/\epsilon}}$ is a constant dependant on $\epsilon$ only. We also pose $\tau_0=r$. Then, because $\Xbb\in \smv_{(\Xcal,\rho)}$, the process visits a sublinear number of sets of $\Pcal_i(\tau_l)$ almost surely. Therefore, there exists an increasing sequence $(n_l)_{l\geq 1}$ such that for any $l\geq 1$,
\begin{equation*}
    \Pbb\left[\forall n\geq n_l,\;|\{i,\; P_i(\tau_l)\cap \Xbb_{< n}\neq \emptyset \} | \leq \frac{\epsilon}{2^7} n\right]\geq 1- \frac{\delta}{2\cdot 2^{l+2}}\quad \text{ and } \quad n_{l+1} \geq \frac{2^6}{\epsilon}n_l
\end{equation*}
We denote by $\Ecal_l$ this event. Thus, $\Pbb[\Ecal_l]\leq \frac{\delta}{2\cdot 2^{l+2}}$. Now, for any $l\geq 1$, we now construct $\mu_l>0$ such that
\begin{equation*}
     \mathbb P\left[\min_{i<j\leq n_{l},\; X_i\neq X_j} \rho(X_i,X_j) > \mu_l \right] \geq 1- \frac{\delta}{2\cdot 2^{l+2}}.
\end{equation*}
We denote this event by $\Fcal_l$. Thus $\Pbb[\Fcal_l]\leq \frac{\delta}{2\cdot 2^{l+2}}$. Note that the sequence $(\mu_l)_{l\geq 1}$ is non-increasing. We now define radiuses $(z^i)_{i\geq 1}$ as follows:
\begin{equation*}
    z^i=\begin{cases}
      \mu_{l_i+1}& \text{if }\rho(x^i,\bar x)<r,\text{ where } \frac{r}{2^{l_i+1}} < r-\rho(x^i,\bar x) \leq \frac{r}{2^{l_i}}\\
      0 & \text{if }\rho(x^i,\bar x)\geq r,
    \end{cases}
\end{equation*}
and consider the sets $R_i:=B(x^i,z^i) \cap \left\{x\in \Xcal:\;  \rho(x,\bar x) <r-\frac{r}{2^{l_i+2}}\right\}$. We construct 
\begin{equation*}
    P_i := R_i \setminus \left(\bigcup_{k<i} R_k \right),
\end{equation*}
for $i\geq 1$. We now show that $(P_i)_{i\geq 1}$ forms a partition of $B(\bar x,r)$ in the next lemma. 

\begin{lemma}
\label{lemma:P_i_partition}
$(P_i)_{i\geq 1}$ forms a partition of $B(\bar x,r)$.
\end{lemma}

\begin{proof}
These sets are clearly disjoint. Now let $x\in B(\bar x,r)$ and consider $j\geq 0$ such that $\frac{r}{2^{j+1}} < r-\rho(x,\bar x) \leq \frac{r}{2^j}$. Then, let $i\geq 1$ such that 
\begin{equation*}
    \rho(x^i,x) < \min\left(\mu_{j+1},r-\frac{r}{2^{j+1}}-\rho(x,\bar x), \rho(x,\bar x) -r + \frac{r}{2^{j-1}}\right).
\end{equation*}
We have $\rho(x^i,\bar x)\leq \rho(x^i,x)  + \rho(x,\bar x) < r-\frac{r}{2^{j+1}}$, hence $r-\frac{r}{2^{l_i}}<r-\frac{r}{2^{j+1}}$ i.e. $l_i\leq j$. Then, we obtain $\rho(x^i,x)<\mu_{j+1}\leq \mu_{l_i+1}$ which gives $x\in B(x^i,z^i)$. Last, we observe that $ \rho(x^i,\bar x)\geq  \rho(x,\bar x) - \rho(x^i,\bar x) > r-\frac{r}{2^{j-1}}$. Therefore, $r-\frac{r}{2^{l_i+1}}>r-\frac{r}{2^{j-1}}$ i.e. $l_i+1\geq j $. Therefore, we have
\begin{equation*}
    \rho(x,\bar x)< r-\frac{r}{2^{j+1}}\leq r-\frac{r}{2^{l_i+2}},
\end{equation*}
which shows $x\in R_i = \bigcup_{k\leq i}P_k$. This ends the proof that $(P_i)_{i\geq 1}$ forms a partition of $B(\bar x,r)$.
\end{proof}

We now define a second partition. We start by defining a sequence of radiuses $(r^i)_{i\geq 1}$ as follows
\begin{equation*}
    r^i=\begin{cases}
      \displaystyle c_\epsilon\inf_{x:\;\rho(x,\bar x)\leq r} \rho(x^i, x) & \text{if }\rho(x^i,\bar x)>r,\\
      \displaystyle c_\epsilon \inf_{x:\;\rho(x,\bar x)\geq r} \rho(x^i, x) & \text{if }\rho(x^i,\bar x)<r,\\
      0 & \text{if }\rho(x^i,\bar x)=r.
    \end{cases}
\end{equation*}
We consider the sets $(A_i)_{i\geq 0}$ given by $A_0 = S(\bar x,r)$ and for $i\geq 1$,
\begin{equation*}
    A_i = B(x^i,r^i)\setminus \left(\bigcup_{1\leq j<i} B(x^j,r^j)\right) .
\end{equation*}
We now show that these sets form a partition in the following lemma.

\begin{lemma}
\label{lemma:A_i_partition}
$(A_i)_{i\geq 0}$ forms a partition of $\Xcal$.
\end{lemma}

\begin{proof}
We start by proving that the sets are disjoint. By construction, if $1\leq j<i$, we have $A_i\subset B(x^j,r^j)$, therefore $A_i\cap A_j=\emptyset$ by construction. Further, for $i\geq 1$, if $\rho(x^i,\bar x)>r$, we first note that $r^i>0$. Indeed, if $r^i=0$, then there exists a sequence of points $x_j$ for $j\geq 1$ such that $\rho(x_j,\bar x)\leq r$ and $\rho(x^i,x_j)\to 0$ as $j\to\infty$. By triangle inequality,
\begin{equation*}
    \rho(x^i,\bar x)\leq \rho(x^i,x_j) + \rho(x_j,\bar x) \leq  \rho(x^i,x_j) + r.
\end{equation*}
This holds for any $j\geq 1$, therefore we obtain $\rho(x^i,\bar x)\leq r$ which contradicts our hypothesis. Therefore $r^i>0$. Further, we have $r^i< \inf_{x:\;\rho(x,\bar x)\leq r} \rho(x^i, x)$. Therefore, for any $x\in A_0=S(\bar x, r)$, we have $\rho(x^i, x) > r^i$ which implies $x\notin B(x^i,r^i)$. Hence, $A_0\cap A_i=\emptyset$. Now if $\rho(x^i,\bar x)<r$ we show again that $r^i>0$. Similarly, if this is not the case, we have a sequence $x_j$ for $j\geq 1$ such that $\rho(x_j,\bar x)\geq r$ and $\rho(x^i,x_j)\to 0$ as $j\to\infty$. Then, observing that
\begin{equation*}
    \rho(x^i,\bar x) \geq \rho(x^i,x_j)-\rho(x^i,x_j) \geq r -\rho(x^i,x_j).
\end{equation*}
This holds for any $j\geq 1$, therefore we obtain $\rho(x^i,\bar x)\geq r$ which contradicts our hypothesis. This shows $r^i>0$. Now for $x\in A_0$, we have by construction $r^i < \rho(x^i,x)$ which gives $x\notin A_i$. Hence $A_0\cap A_i=\emptyset$. Finally, if $\rho(x^i,\bar x)=r$, we have $r^i=0$ so $A_i=\emptyset$ and we obtain direly $A_0\cap A_i = \emptyset$. This ends the proof that for any $0\leq i<j$, we have $A_i\cap A_j=\emptyset$. \\

We now prove that $\cup_{i\geq 0}A_i =\Xcal.$ Let $x\in \Xcal$. If $\rho(x,\bar x) = r$ then $x\in A_0$. If $\rho(x,\bar x)>r$ (resp. $\rho(x,\bar x)<r$), using the same arguments as above, we can show that $\inf_{\tilde x:\;\rho(\tilde x,\bar x)\leq r} \rho(x, \tilde x) > 0$ (resp. $\inf_{\tilde x:\;\rho(\tilde x,\bar x)\geq r} \rho(x, \tilde x) > 0$). Therefore, we let $i\geq 1$ such that $\rho(x^i,x)<\frac{1}{1+\frac{2}{c_\epsilon}}\inf_{\tilde x:\;\rho(\tilde x,\bar x)\leq r} \rho(x, \tilde x)$ (resp. $\rho(x^i,x)<\frac{1}{1+\frac{2}{c_\epsilon}}\inf_{\tilde x:\;\rho(\tilde x,\bar x)\geq r} \rho(x, \tilde x)$). This is possible because the sequence $(x^i)_{i\geq 1}$ is dense in $\Xcal$. Then, we have for any $\tilde x$ such that $\rho(\tilde x,\bar x)\leq r$ (resp. $\rho(\tilde x,\bar x)\geq r$),
\begin{equation*}
    \rho(x^i,\tilde x)\geq \rho(x,\tilde x) - \rho(x^i,x) > \left(1+\frac{2}{c_\epsilon}-1\right)\rho(x^i,x) = \frac{2}{c_\epsilon}\rho(x^i,x).
\end{equation*}
Therefore, $r^i \geq 2 \rho(x^i,x)>\rho(x^i,x)$ which gives $x\in B(x^i,r^i)$. Now note that $\bigcup_{1\leq j\leq i}A_i = \bigcup_{1\leq j\leq i} B(x^i,r^i)$, therefore we obtain $x\in \bigcup_{1\leq j\leq i}A_i$. This ends the proof that $(A_i)_{i\geq 0}$ forms a partition of $\Xcal$.
\end{proof}

We now formally consider the product partition of $(P_{i})_{i\geq 1}$ and $(A_i)_{i\geq 0}$ i.e.
\begin{equation*}
   \Qcal:\quad  \bigcup_{i\geq 0,\; A_i\subset B(\bar x,r)} \bigcup_{j\geq 1} (A_i\cap P_{j}) \cup  \bigcup_{i\geq 0,\; A_i\subset \Xcal\setminus B(\bar x,r)} A_i.
\end{equation*}
where we used the fact that sets $A_i$ satisfy either $A_i\subset B(\bar x,r)$ or $A_i\subset \Xcal\setminus B(\bar x,r)$. We will show that this partition disproves the $\smv_{(\Xcal,\rho)}$ hypothesis on $\Xbb$. In practice, we will either prove that the process visits many sets from partition $(A_i)_{i\geq 0}$ or $(P_{i})_{i\geq 1}$ and use the fact that the same analysis would work for $\Qcal$, the product partition as well.\\

We now consider a specific realization $\mb x = (x_t)_{t\geq 0}$ of the process $\Xbb$ falling in the event $\Acal\bigcap_{l\geq 1}(\Ecal_l\cap \Fcal_l)$. This event has probability
\begin{equation*}
    \Pbb\left[\Acal\bigcap_{l\geq 1}(\Ecal_l\cap \Fcal_l)\right] \geq \Pbb[\Acal] - \sum_{l\geq 1} (\Pbb[\Ecal_l^c] + \Pbb[\Fcal_l^c]) \geq \frac{\delta}{2} - \frac{\delta}{4} = \frac{\delta}{4}.
\end{equation*}Note that $\mb x$ is not random anymore. We now show that $\mb x$ does not visit a sublinear number of sets in the partition $\Qcal$.

We now denote by $(t_k)_{k\geq 1}$ the increasing sequence of all times when 2C1NN makes an error in the prediction of $f^*(x_t)$. Because the event $\Acal$ is satisfied, $\Lcal_{\mb x} (2C1NN,f^*)>\epsilon$, therefore, we can define an increasing sequence of times $(T_l)_{l\geq 1}$ such that
\begin{equation*}
    \frac{1}{T_l}\sum_{t=1}^{T_l} \ell_{01}(2C1NN(\mb x_{<t},\mb y_{<t},x_t),f^*(x_t))> \frac{\epsilon}{2}.
\end{equation*}
For any $l\geq 1$ consider the last index $k = \max\{u,t_u\leq T_l\}$ when 2C1NN makes a mistake. Then we obtain $k > \frac{\epsilon}{2} T_l \geq \frac{\epsilon}{2} t_k$. Considering the fact that $(T_l)_{l\geq 1}$ is an increasing unbounded sequence we therefore obtain an increasing sequence of indices $(k_l)_{l\geq 1}$ such that $t_{k_l}<\frac{2k_l}{\epsilon}$.

At an iteration where the new input $x_t$ has not been previously visited we will denote by $\phi(t)$ the index of the nearest neighbor of the current dataset in the 2C1NN learning rule. Now let $l\geq 1$. We focus on the time $t_{k_l}$. Consider the tree $\Gcal$ where nodes are times $\Tcal:=\{t,\; t\leq t_{k_l},\; x_t\notin\{x_u, u<t\} \}$ for which a new input was visited, where the parent relations are given by $(t,\phi(t))$ for $t\in \Tcal\setminus\{1\}$. In other words, we construct the tree in which a new input is linked to its representant which was used to derive the target prediction. Note that by definition of the 2C1NN learning rule, each node has at most $2$ children and a node is not in the dataset at time $t_{k_l}$ when it has exactly $2$ children.

\paragraph{Step 1.}
We now suppose that the majority of input points on which 2C1NN made a mistake belong to $B(\bar x,r)$ i.e.
\begin{equation*}
    \left|\left\{t\leq t_{k_l},\; \ell_{01}(2C1NN(\mb x_{<t},\mb y_{<t},x_t),f^*(x_t))=1,\; x_t\in B(\bar x,r) \right\}\right| \geq \frac{k_l}{2},
\end{equation*}
or equivalently $|\{k\leq k_l,\; x_{t_k}\in B(\bar x,r)\}|\geq \frac{k_l}{2}$.

Let us now consider the subgraph $\tilde \Gcal$ given by restricting $\Gcal$ only to nodes in the the ball $B(\bar x,r)$ which are mapped to the true value $1$ i.e. on times $\{t\in \Tcal,\; x_t\in B(\bar x,r)\}$. In this subgraph, the only times with no parent are times $t_k$ with $k\leq k_l$ and $x_{t_k}\in B(\bar x,r)$ and possibly time $t=1$. Indeed, if a time in $\tilde \Gcal$ has a parent $\phi(t)$ in $\tilde \Gcal$, the prediction of 2C1NN for $x_t$ returned the correct answer $1$. The converse is also true except for the root time $t=1$ which has no parent in $\Gcal$. Therefore, $\tilde \Gcal$ is a collection of disjoint trees with roots times $\{t_k, \; k\leq k_l, \; x_{t_k}\in B(\bar x,r)\}$---and possibly $t=1$ if $x_1\in B(\bar x,r)$. For a given time $t_k$ with $k\leq k_l$ and $x_{t_k}\in B(\bar x,r)$, we will denote by $\Tcal_k$ the corresponding tree in $\tilde \Gcal$ with root $t_k$. We will say that the $\Tcal_k$ is a \emph{good} tree if all times $t\in \Tcal_k$ of this tree are parent in $\Gcal$ to at most $1$ time from $\Xcal\setminus B(\bar x,r)$ i.e. if 
\begin{equation*}
    \forall t\in \Tcal_k,\quad \left|\left\{u
    \leq t_{k_l},\; \phi(u) = t,\; \rho(x_u,\bar x)\geq r\right\}\right|\leq 1.
\end{equation*}
We denote by $G = \{k\leq k_l,\; x_{t_k}\in B(\bar x,r),\; \Tcal_k \text{ good}\} $ the set of indices of good trees. By opposition, we will say that a tree is \emph{bad} otherwise. We now give a simple upper bound on $N_{\text{bad}}$ the number of bad trees. Note that for any $t\in \Tcal_k$, times in $\left\{u \leq t_{k_l},\; \phi(u) = t,\; \rho(x_u,\bar x)\geq r\right\}$ are times when 2C1NN makes a mistake on $\Xcal\setminus B(\bar x,r)$. Therefore,
\begin{equation*}
    \sum_{k\leq k_l,\;  x_{t_k}\in B(\bar x,r)} \sum_{t\in \Tcal_k} \left|\left\{u<t_{k_l},\; \phi(u) = t,\; \rho(x_u,\bar x)\geq r\right\}\right| \leq \left|\left\{k\leq t_{k_l}, \rho(x_{t_k},\bar x)\geq r\right\}\right| \leq \frac{k_l}{2}
\end{equation*}
because by hypothesis $|\{k\leq k_l,\; x_{t_k}\in B(\bar x,r)\}|\geq \frac{k_l}{2}$. Therefore, since each bad tree contains a node which is parent to at least $2$ times of mistake in $\Xcal\setminus B(\bar x,r)$, we obtain
\begin{equation*}
    N_{\text{bad}} \leq \sum_{k\leq k_l,\;  x_{t_k}\in B(\bar x,r)} \sum_{t\in \Tcal_k} \frac{1}{2}\left|\left\{u<t_{k_l},\; \phi(u) = t,\; \rho(x_u,\bar x)\geq r\right\}\right| \leq  \frac{k_l}{4}.
\end{equation*}
Thus, the number of good trees is $|G|\geq\left|\left\{k\leq k_l,\; x_{t_k}\in B(\bar x,r)\right\}\right|-N_{\text{bad}}\geq \frac{k_l}{4}$. Now note that trees are disjoint, therefore, $\sum_{k\in G} |\Tcal_k|\leq t_{k_l}<\frac{2k_l}{\epsilon}.$
Therefore,
\begin{equation*}
    \sum_{k\in G}\1_{|\Tcal_k|\leq \frac{16}{\epsilon}} = |G| - \sum_{k\in G}\1_{|\Tcal_k|> \frac{16}{\epsilon}}> |G|-\frac{\epsilon}{16} \sum_{k\in G}|\Tcal_k|\geq \frac{k_l}{8}.
\end{equation*}
We will say that a tree $|\Tcal_k|$ is \emph{sparse} if it is good and has at most $\frac{\epsilon}{16}$ nodes. With $S := \{k\in G,\;|\Tcal_k|\leq \frac{16}{\epsilon} \}$ the set of sparse trees, the above equation we have $|S|\geq \frac{k_l}{8}$. We now focus only on sparse trees $\Tcal_k$ for $k\in S$ and analyze their relation with the final dataset $\Dcal_{t_{k_l}}$. Precisely, for a sparse tree $\Tcal_k$, denote $\Vcal_k = \Tcal_k\cap \Dcal_{t_{k_l}}$ the set of times which are present in the final dataset and belong to the tree induced by error time $t_k$. Because each node of $\Tcal_k$ and not present in $\Dcal_{t_{k_l}}$ has at least $1$ children in $\Tcal$, we note that $\Vcal_k\neq\emptyset$. We now consider the path from a node of $\Vcal_k$ to the root $t_k$. We denote by $d(k)$ the depth of this node in $\Vcal_k$ and denote the path by $p_{d(k)}^k\to p_{d(k)-1}^k\to p_0^k=t_k$ where $p_{d(k)}^k\in \Vcal_k$. Then we have,
\begin{equation*}
    d(k)\leq |\Tcal_k|-1 \leq \frac{16}{\epsilon} - 1.
\end{equation*}
Each arc of this path represents the fact that at the corresponding iteration $p_i^k$ of 2C1NN, the parent $x_{p_{i-1}^k}$ was closer from $x_{p_{i}^k}$ than any other point of the current dataset $\Dcal_{p_i^k}$. We will now show that all the points $\{p_{d(k)}^k,\; k\in S\}$ fall in distinct sets of the partition $(A_i)_{i\geq 0}$. Suppose by contradiction that we have $k_1\neq k_2\in S$ falling into the same set $A_i$. Note that because $x_{p_{d(k_1)}^{k_1}},x_{p_{d(k_2)}^{k_2}}\in B(\bar x,r)$, we obtain $A_i\cap B(\bar x,r)\neq \emptyset$. However, the partition $(A_i)_{i\geq 0}$ was constructed so that sets are included totally in either $B(\bar x, r)$, $S(\bar x, r)$ or $\{x\in \Xcal,\; \rho(x,\bar x)>r\}$. Therefore, we obtain $A_i\subset B(\bar x, r)$ and $x^i\in B(\bar x, r)$. We can now apply Lemma \ref{lemma:path_recursion} to $p_{d(k_1)}^{k_1}\to p_{d(k_1)-1}^{k_1}\to \ldots \to p_0^{k_1}$ and $p_{d(k_2)}^{k_2}\to p_{d(k_2)-1}^{k_2}\to \ldots \to p_0^{k_2}$---which we write by convenience $p_{d}\to p_{d-1}\to\ldots\to p_1\to p_0$ and $q_{f}\to q_{f-1}\to\ldots\to q_1\to q_0$---assuming without loss of generality that $p_0<q_0$. Therefore, $\rho(x_{p_{v(0)}},x_{q_0})\leq 2^{f+d}\rho(x_{p_d},x_{q_f}) \leq 2^{f+d+1}r^i$ and $\rho(x_{p_{v(0)}},x_{p_d})\leq 2^{f+d}\rho(x_{p_d},x_{q_f})\leq 2^{f+d+1}r^i$. But recall that these two paths come from sparse trees, so $d,f\leq \frac{16}{\epsilon} - 1$. Hence, $2^{f+d+1}\leq \frac{1}{2}2^{2^5/\epsilon}=\frac{1}{4c_\epsilon}$. Let us now consider $x_{\phi(q_0)}$ the point which induced a mistake in the prediction of $x_{q_0}$, i.e. $\rho(x_{\phi(q_0)},\bar x)\geq r$. Then,
\begin{align*}
    \rho(x_{q_0},x_{\phi(q_0)}) &\geq \rho(x_{\phi(q_0)},x^i) - \rho(x^i,x_{p_d}) - \rho(x_{p_d},x_{p_{v(0)}}) - \rho(x_{p_{v(0)}},x_{q_0}) \\
    &\geq \frac{r^i}{c_\epsilon} - r^i -  \frac{r^i}{4c_\epsilon}- \frac{r^i}{4c_\epsilon} \\
    &\geq \frac{r^i}{4c_\epsilon}
\end{align*}
where in the last inequality we used the fact that $c_\epsilon<\frac{1}{4}$. Recall that we also proved $\rho(x_{p_{v(0)}},x_{q_0})\leq \frac{r^i}{4c_\epsilon}<\rho(x_{q_0},x_{\phi(q_0)})$. However, datapoint $x_{p_{v(0)}}$ is available in dataset $\Dcal_{q_0}$. This contradicts the fact that $x_{\phi(t)}$ was chosen as representant for $x_{q_0}$. This ends the proof that all the points $\{p^k_{d(k)},\; k\in S\}$ fall in distinct sets of the partition $(A_i)_{i\geq 0}$. Therefore,
\begin{equation*}
    |\{i,\; A_i\cap \mb{x}_{\leq t_{k_l}}\neq \emptyset \}|\geq |S|\geq \frac{k_l}{8} \geq \frac{\epsilon}{16}t_{k_l}.
\end{equation*}

\paragraph{Step 2.}
We now turn to the case when the majority of input points on which 2C1NN made a mistake are not in the ball $B(\bar x,r)$ i.e.
\begin{equation*}
    \left|\left\{t\leq t_{k_l},\; \ell_{01}(2C1NN(\mb x_{<t},\mb y_{<t},x_t),f^*(x_t))=1,\; \rho(x_t,\bar x)\geq r \right\}\right| \geq \frac{k_l}{2},
\end{equation*}
or equivalently $|\{k\leq k_l,\; \rho(x_{t_k},\bar x)\geq r\}|\geq \frac{k_l}{2}$. Similarly as the previous case, we consider the graph $\tilde G$ given by restricting $\Gcal$ only to nodes outside the ball $B(\bar x,r)$ i.e. on times $\{t\in \Tcal, \rho(x_t,\bar x)\geq r\}$. Again, $\tilde \Gcal$ is a collection of disjoint trees with root times $\{t_k,\; k\leq k_l,\; \rho(x_{t_k},\bar x)\geq r\}$ (and possibly $t=1$). We denote $\Tcal_k$ the corresponding tree of $\tilde \Gcal$ rooted in $t_k$. Similarly to above, a tree is \emph{sparse} if
\begin{equation*}
    \forall t\in \Tcal_k,\quad \left|\left\{u
    \leq t_{k_l},\; \phi(u) = t,\; \rho(x_u,\bar x)< r\right\}\right|\leq 1 \quad \text{and} \quad |\Tcal_k|\leq \frac{16}{\epsilon}.
\end{equation*}
If $S=\{k\leq k_l,
; \rho(x_{t_k},\bar x)\geq r,\; \Tcal_k\text{ sparse}\}$ denotes the set of sparse trees, the same proof as above shows that $|S|\geq \frac{k_l}{8}$. Again, for any $k\in S$, if $d(k)$ denotes the depth of some node from $\Vcal_k:=\Tcal_k\cap \Dcal_{t_{k_l}}$ in $\Tcal_k$ we have $d(k)\leq  \frac{16}{\epsilon}-1$. For each $k\in S$ we consider the path from this node of $\Vcal_k $ to the root $t_k$: $p_{d(k)}^k\to p_{d(k)-1}^k\to \ldots \to p_0^k=t_k$ where $p_{d(k)}^k\in \Vcal_k$. The same proof as above shows that all the points $\{p^k_{d(k)},\; k\in S,\; \rho(x_{p^k_{d(k)}},\bar x)>r\}$ lie in distinct sets of the partition $(A_i)_{i\geq 0}$.

Indeed, let $p_{d}\to p_{d-1}\to\ldots\to p_1\to p_0$ and $q_{f}\to q_{f-1}\to\ldots\to q_1\to q_0$ two such paths with $\rho(x_{p_d},\bar x)>r$ and $\rho(x_{q_f},\bar x)>r$ and suppose by contradiction that $x_{p_d},x_{q_f}\in A_i$ for some $i\geq 0$. Necessarily, $i\geq 1$ and $\rho(x^i,\bar x)>r$. Lemma \ref{lemma:path_recursion} gives again $\rho(x_{p_{v(0)}},x_{q_0}),\rho(x_{p_{v(0)}},x_{p_d})\leq 2^{f+d}\rho(x_{p_d},x_{q_f})\leq 2^{f+d+1}r^i\leq  \frac{r^i}{4c_\epsilon}$. Then, if $x_{\phi(q_0)}$ is the point that induced a mistake in the prediction of $x_{q_0}$, we have $\rho(x_{\phi(q_0)},\bar x)<r$. Using the definition of $r^i$ we obtain the same computations
\begin{equation*}
    \rho(x_{q_0},x_{\phi(q_0)}) \geq \rho(x_{\phi(q_0)},x^i) - \rho(x^i,x_{p_d}) - \rho(x_{p_d},x_{p_{v(0)}}) - \rho(x_{p_{v(0)}},x_{q_0})\geq \frac{r^i}{4c_\epsilon} > \rho(x_{p_{v(0)}},x_{q_0})
\end{equation*} which contradicts the fact that $x_{\phi(q_0)}$ was used as representant for $x_{q_0}$. This ends the proof that all the points $\{p^k_{d(k)},\; k\in S,\; \rho(x_{p^k_{d(k)}},\bar x)>r\}$ lie in distinct sets of the partition $(A_i)_{i\geq 0}$. 
Suppose $|\{k\in S,\; \rho(x_{p^k_{d(k)}},\bar x)>r\}|\geq \frac{|S|}{2}$, then we have
\begin{equation*}
    |\{i,\; A_i\cap \mb{x}_{\leq t_{k_l}}\neq \emptyset \}|\geq |\{k\in S,\; \rho(x_{p^k_{d(k)}},\bar x)>r\}| \geq \frac{|S|}{2} \geq \frac{k_l}{16}\geq \frac{\epsilon}{32}t_{k_l}.
\end{equation*}

\paragraph{Step 3.}
In this last step, we suppose again that the majority of input points on which 2C1NN made a mistake are not in the ball $B(\bar x,r)$ and that $|\{k\in S,\; \rho(x_{p^k_{d(k)}},\bar x)>r\}|< \frac{|S|}{2}$. Therefore, we obtain
\begin{equation*}
    |\{k\in S,\; \rho(x_{p^k_{d(k)}},\bar x)=r\}| = |S|-|\{k\in S,\; \rho(x_{p^k_{d(k)}},\bar x)>r\}| \geq \frac{|S|}{2} \geq \frac{k_l}{16} \geq \frac{\epsilon}{32}t_{k_l}.
\end{equation*}
We will now make use of the partition $(P_i)_{i\geq 1}$. Because $(n_u)_{u\geq 1}$ is an increasing sequence, let $u\geq 1$ such that $n_{ u+1}\leq t_{k_l}\leq n_{ u+2}$ (we can suppose without loss of generality that $t_{k_0}>n_2$). Note that we have $n_u\leq \frac{\epsilon}{2^6}n_{u+1}\leq \frac{\epsilon}{2^6}t_{k_l}$. Let us now analyze the process between times $n_u$ and $t_{k_l}$. In particular, we are interested in the indices $T=\{k\in S,\; \rho(x_{p^k_{d(k)}},\bar x)=r\}$ and times $\Ucal_u = \{p^k_{d(k)}:\; n_u< p^k_{d(k)}\leq k_l,\; k\in T\}$. In particular, we have
\begin{equation*}
    |\Ucal_u| \geq  |\{k\in S,\; \rho(x_{p^k_{d(k)}},\bar x)=r\}| - n_u  \geq \frac{\epsilon}{32}t_{k_l} -\frac{\epsilon}{2^6}t_{k_l} = \frac{\epsilon}{2^6}t_{k_l}.
\end{equation*}
Because the event $\Ecal_u$ is met, we have
\begin{equation*}
    |\{i,\; P_i(\tau_u)\cap \mb x_{\Ucal_u} \neq \emptyset\}| \leq |\{i,\; P_i(\tau_u)\cap \mb x_{\leq t_{k_l}}\neq \emptyset\}| \leq \frac{\epsilon}{2^7}t_{k_l}.
\end{equation*}
Note that $\mb x_{\Ucal_u}\subset S(\bar x,r)$. Therefore, each of the points in $\mb x_{\Ucal_u}$ falls into one of the sets $(P_i(\tau_u))_{i\geq 1}$. Let $i\geq 1$ such that the set $P_i(\tau_u)$ was visited by $\mb x_{\Ucal_u}$ and consider $T_i = \{k\in T,\; x_{ p^k_{d(k)}}\in A_i\}$. We will show that at least $|T_i|-1$ of the points $\{x_{\phi(t_k)},\; k\in T_i\}$ fall in $B(\bar x,r)\setminus B(\bar x,r-\frac{r}{2^{u+2}})$. 

To do so, let $k_1,k_2\in T_i$. Similarly as above, for simplicity, we will refer to the path $p_{d(k_1)}^{k_1}\to p_{d(k_1)-1}^{k_1}\to\ldots \to p_0^{k_1}$ (resp. $p_{d(k_2)}^{k_2}\to p_{d(k_2)-1}^{k_2}\to \ldots \to p_0^{k_2}$) as $p_{d}\to p_{d-1}\to\ldots\to p_1\to p_0$ (resp. $q_{f}\to q_{f-1}\to\ldots\to q_1\to q_0$), and assume without loss of generality that $p_0<q_0$. Note that by hypothesis, $k_1,k_2\in T_i$, therefore, $\rho(x_{p_d},x^i),\rho(x_{q_f},x^i)\leq \tau_u$ Then, using the above computations yields
\begin{equation*}
    \rho(x_{p_{v(0)}},x_{q_0})\leq 2^{f+d} \rho(x_{p_d},x_{q_f}) \leq 2^{f+d} (\rho(x_{p_d},x^i) + \rho(x_{q_f},x^i))\leq 2^{f+d+1} \tau_u \leq \frac{\tau_u}{4c_\epsilon},
\end{equation*}
where in the last inequality we used the fact that $f,d\leq  \frac{16}{\epsilon}-1$ hence $2^{f+d+1}\leq \frac{1}{4 c_\epsilon}$. Now by definition of a representant, we obtain
\begin{equation*}
    \rho(x_{\phi(q_0)},x_{q_0})\leq \rho(x_{p_{v(0)}},x_{q_0}) \leq \frac{r}{8\cdot 2^u}.
\end{equation*}
Therefore, $\rho(x_{\phi(q_0)},\bar x)\geq \rho(x_{q_0},\bar x)- \rho(x_{\phi(q_0)},x_{q_0}) \geq r- \frac{r}{8\cdot 2^u}$. Because $x_{\phi(q_0)}$ induced a mistake in the prediction for $x_{q_0}$ we have $x_{\phi(q_0)}\in B(\bar x ,r)$. Now order $T_i=\{k_1<\ldots <k_{|T_i|}\}$. We then have $t_{k_1}<\ldots<t_{k_{|T_i|}}$. The argument above then shows that for any $2\leq j\leq |T_i|$, we have $x_{\phi(t_{k_j})}\in B(\bar x,r)\setminus B(\bar x,r-\frac{r}{2^{u+3}})$. Therefore, defining $T' := \{k\in T,\; r-\frac{r}{2^{u+3}}\leq \rho(x_{\phi(t_k)},\bar x)<r\}$ we obtain 
\begin{equation*}
    \left| T'\right| \geq |\Ucal_u|-|\{i,\; P_i(\tau_u)\cap \mb x_{\Ucal_u} \neq \emptyset\}| \geq \frac{\epsilon}{2^7}t_{k_l}.
\end{equation*}
We will now show that all the points in $\{x_{t_k},\; k\in T'\}$ lie in distinct sets of $(P_i)_{i\geq 1}$. Note that because we have $t_{k_l}\leq n_{u+2}$ and because the event $\Fcal_{u+2}$ is met, we have that for any $p,q\in  T'$ that $\rho(x_{\phi(t_p)},x_{\phi(t_q)})> \mu_{u+2}.$ Now suppose by contradiction that $x_{\phi(t_p)},x_{\phi(t_q)}\in P_i$ for some $i\geq 1$. Then, with $l_i$ such that $r-\frac{r}{2^{l_i}}\leq \rho(x^i,\bar x) < r-\frac{r}{2^{l_i+1}}$ we have that 
\begin{equation*}
    x_{\phi(t_p)},x_{\phi(t_q)} \in \left\{x\in \Xcal:\; \rho(x,\bar x) <r-\frac{r}{2^{l_i+2}}\right\}
\end{equation*}
But we know that $ \rho(x_{\phi(t_p)},\bar x)\geq r-\frac{r}{2^{u+3}}$. Therefore we obtain $r-\frac{r}{2^{l_i+2}}> r-\frac{r}{2^{u+3}}$ and hence $l_i\geq u+1$. Recall that $P_i\subset B(x^i,\mu_{l_i+1})$. Therefore, we obtain
\begin{equation*}
    \rho(x_{\phi(t_p)},x_{\phi(t_q)}) \leq \mu_{l_i+1} \leq \mu_{u+2},
\end{equation*}
which contradicts the fact that $\rho(x_{t_p},x_{t_q})> \mu_{u+2}.$ This ends the proof that all points of $\{x_{t_k},\; k\in T'\}$ lie in distinct subsets of $(P_i)_{i\geq 1}$. Now we obtain
\begin{equation*}
    |\{i,\; P_i\cap \mb x_{\leq t_{k_l}} \neq \emptyset\}| \geq |T'| \geq \frac{\epsilon}{2^7}t_{k_l}.
\end{equation*}

\paragraph{Step 4.}
In conclusion, in all cases, we obtain 
\begin{equation*}
    |\{Q\in \Qcal,\; Q\cap \mb x_{\leq t_{k_l}} \neq \emptyset\}| \geq \max(|\{i,\; A_i\cap \mb{x}_{\leq t_{k_l}}\neq \emptyset \}|,|\{i,\; P_i\cap \mb x_{\leq t_{k_l}} \neq \emptyset\}|) \geq \frac{\epsilon}{2^7}t_{k_l}.
\end{equation*}
Because this is true for all $l\geq 1$ and $t_{k_l}$ is an increasing sequence, we conclude that $\mb x$ disproves the $\smv_{(\Xcal,\rho)}$ condition for $\Qcal$. Recall that this holds whenever the event $\Acal\bigcap_{l\geq 1}(\Ecal_l\cap \Fcal_l)$ is met. Thus,
\begin{equation*}
    \Pbb[|\{Q\in \Qcal,\; Q\cap \Xbb_{<T}\}|=o(T)]\leq 1-\Pbb[\Acal\bigcap_{l\geq 1}(\Ecal_l\cap \Fcal_l)] \leq 1-\frac{\delta}{4}<1.
\end{equation*}
This shows that $\Xbb\notin \smv_{(\Xcal,\rho)}$ which is absurd. Therefore 2C1NN is consistent on $f^*$. This ends the proof of the proposition.
\end{proof}

We can now show that 2C1NN is optimistically universal for the binary classification setting, with a similar proof structure to Theorem \ref{thm:4C1NN}. Precisely, we show that under any process $\Xbb\in \smv_{(\Xcal,\rho)}$, the functions on which it is consistent form a $\sigma-$algebra which contains all balls, and as a consequence all Borel sets.

\begin{theorem}
\label{thm:2C1NN_smv}
Let $(\Xcal,\Bcal)$ be a separable Borel space. For the binary classification setting, the learning rule 2C1NN is universally consistent for all processes $\Xbb\in \smv_{(\Xcal,\rho)}$.
\end{theorem}

\begin{proof}
let $\Xbb\in \smv_{(\Xcal,\rho)}$. We will show that 2C1NN is universally consistent on $\Xbb$ by considering the set $\Scal_\Xbb$ of functions for which it is consistent. More precisely, since $\Ycal = \{0,1\}$ in the binary setting, all target functions can be described as $f^* = \1_{ A_{f^*}}$ where $A_{f^*} = f^{<-1>}(\{1\})$. We define $\Scal_\Xbb$ using the corresponding sets:
\begin{equation*}
    \Scal_\Xbb:= \{A\in \Bcal,\quad \Lcal_\Xbb(2C1NN,\1_{\cdot\in A})=0\quad (a.s.) \}
\end{equation*}
By construction we have $\Scal_\Xbb\subset\Bcal$. The goal is to show that in fact $\Scal_\Xbb = \Bcal$. To do so, we will show that $\Scal$ satisfies the following properties
\begin{itemize}
    \item $\emptyset\in \Scal_\Xbb$ and $\Scal_\Xbb$ contains all balls $B(x,r)$ with $x\in \Xcal$ and $r\geq 0$,
    \item if $A\in \Scal_\Xbb$ then $A^c\in \Scal_\Xbb$ (stable to complementary),
    \item if $(A_i)_{i\geq 1}$ is a sequence of disjoint sets of $\Scal_\Xbb$, then $\bigcup_{i\geq 1} A_i\in \Scal_\Xbb$ (stable to $\sigma-$additivity for disjoint sets),
    \item if $A,B\in \Scal_\Xbb$, then $A\cup B\in \Scal_\Xbb$ (stable to union).
\end{itemize}
Together, these properties show that $\Scal_\Xbb$ is a $\sigma-$algebra that contains all open intervals of $\Xcal$. Recall that by definition, $\Bcal$ is the smallest $\sigma-$algebra containing open intervals. Therefore we get $\Bcal\subset \Scal_\Xbb$ which proves the theorem. We now show the four properties.\\

The invariance to complementary and to finite union can be shown with the same proof as Theorem \ref{thm:4C1NN}. Further, we clearly have $\emptyset \in \Scal_\Xbb$. Now let $x\in \Xcal$ and $r\geq 0$, Proposition \ref{prop:consistent_ball_borel} shows that $B(x,r)\in \Scal_\Xbb$.

We now turn to the $\sigma-$additivity for disjoint sets. Let $(A_i)_{i\geq 1}$ is a sequence of disjoint sets of $\Scal_\Xbb$. We denote $A:= \bigcup_{i\geq 1} A_i$. We consider the target function $f^* = \1_{A}$. We write the average loss in the following way,
\begin{equation*}
     \frac{1}{T}\sum_{t=1}^{T}\ell_{01}(2C1NN(\Xbb_{< t},\Ybb_{< t}, X_{t}), f^*(X_t)) = \frac{1}{T}\sum_{t=1}^{T} \1_{X_t\in A}  \1_{X_{\phi(t)}\notin A} + \frac{1}{T}\sum_{t=1}^{T}\1_{X_t\notin A}  \1_{X_{\phi(t)}\in A}.
\end{equation*}
where the first term corresponds to type 1 errors and the second term corresponds to type 2 errors.

We suppose by contradiction that $\mathbb P(\Lcal_\Xbb(2C1NN,f^*)> 0):=\delta>0$ Therefore, there exists $\epsilon>0$ such that $\mathbb P(\Lcal_\Xbb(2C1NN,f^*)> \epsilon)\geq \frac{\delta}{2}$. We denote this event by $\Acal:=\{\Lcal_\Xbb(2C1NN,f^*)>\epsilon\}$. We first analyze the errors induced by one set $A_i$ only. We have 
\begin{align*}
    \frac{1}{T}\sum_{t=1}^{T} ( \1_{X_t\in A_i }  \1_{X_{\phi(t)}\notin A} + \1_{X_t\notin A}  \1_{X_{\phi(t)}\in A_i}) & \leq \frac{1}{T}\sum_{t=1}^{T} (\1_{X_t\in A_i }  \1_{X_{\phi(t)}\notin A_i} + \1_{X_t\notin A_i}  \1_{X_{\phi(t)}\in A_i})\\
    &= \frac{1}{T}\sum_{t=1}^{T} \ell_{01}(2C1NN(\Xbb_{< t},\1_{\Xbb_{< t}\in A_i}, X_{t}), \1_{X_t\in A_i}).
\end{align*}
Then, because 2C1NN is consistent for $\1_{\cdot\in A_i}$, we get
\begin{equation*}
    \frac{1}{T}\sum_{t=1}^{T} ( \1_{X_t\in A_i }  \1_{X_{\phi(t)}\notin A} + \1_{X_t\notin A}  \1_{X_{\phi(t)}\in A_i}) \to 0 \quad (a.s.).
\end{equation*}
We take $\epsilon_i = \frac{\epsilon}{4\cdot 2^i}$. The above equation gives $T^i$ such that
\begin{equation*}
     \mathbb P\left[ \bigcap_{T\geq T^i} \left\{\frac{1}{T}\sum_{t=1}^{T} ( \1_{X_t\in A_i }  \1_{X_{\phi(t)}\notin A} + \1_{X_t\notin A}  \1_{X_{\phi(t)}\in A_i}) <\epsilon_i \right\} \right] \geq 1-\frac{\delta}{8\cdot 2^i}.
\end{equation*}
We will denote by $\Ecal_i$ this event. We now consider the scale of the process $\Xbb_{\leq T^i}$ when falling in $A_i$, by introducing $\eta_i>0$ such that
\begin{equation*}
    \mathbb P\left[ \min_{ \substack{
        t_1,t_2 \leq T^i;\; X_{t_1},X_{t_2}\in A_i; \\
        X_{t_1}\neq X_{t_2}
     }} \rho(X_{t_1},X_{t_2}) > \eta_i \right]  \geq 1-\frac{\delta}{8\cdot 2^i}.
\end{equation*}
We denote by $\Fcal_i$ this event. By the union bound, we have $\mathbb P(\bigcup_{i\geq 1} \Ecal_i^c\cup\bigcup_{i\geq 1} \Fcal_i^c)\leq \frac{\delta}{4}$. Therefore, we obtain $\mathbb P(\Acal\cap \bigcap_{i\geq 1} \Ecal_i\cap\bigcap_{i\geq 1} \Fcal_i)\geq \mathbb P(\Acal) - \mathbb P(\bigcup_{i\geq 1} \Ecal_i^c\cup\bigcup_{i\geq 1} \Fcal_i^c) \geq \frac{\delta}{4}$. We now construct a partition $\Pcal$ obtained by subdividing each set $A_i$ according to scale $\eta_i$. Because $\Xcal$ is separable, there exists a sequence of points $(x^j)_{j\geq 1}$ in $\Xcal$ such that $\forall x\in \Xcal, \inf_{j\geq 1}\rho(x,x^j)=0.$ We construct the following partition of $\Xcal$ given by
\begin{equation*}
    \Pcal\; : \quad A^c\cup \bigcup_{i\geq 1} \bigcup_{j\geq 1}\left\{ \left(B\left(x^j,\frac{\eta_i}{2}\right)\cap A_i\right)\setminus \bigcup_{k<j} B\left(x^k,\frac{\eta_i}{2}\right)\right\}.
\end{equation*}
Let us now consider a realization of $\mb x$ of $\Xbb$ in the event $\Acal\cap \bigcap_{i\geq 1} \Ecal_i\cap\bigcap_{i\geq 1} \Fcal_i$. The sequence $\mb x$ is now not random anymore. Our goal is to show that $\mb x$ does not visit a sublinear number of sets in the partition $\Pcal$.

By construction, the event $\Acal$ is satisfied, therefore there exists an increasing sequence of times $(t_k)_{k\geq 1}$ such that for any $k\geq 1$, $\frac{1}{t_k}\sum_{t=1}^{t_k}\ell_{01}(2C1NN(\mb x_{< t},\1_{\mb x_{< t} \in A}, x_{t}), \1_{x_t\in A}) > \frac{\epsilon}{2}.$ Therefore, we obtain for any $k\geq 1$,
\begin{equation*}
   \sum_{i\geq 1} \frac{1}{t_k}\sum_{t=1}^{t_k} ( \1_{x_t\in A_i }  \1_{x_{\phi(t)}\notin A} + \1_{x_t\notin A}  \1_{x_{\phi(t)}\in A_i}) > \frac{\epsilon}{2}.
\end{equation*}
Also, because the events $\Ecal_i$ are met, we have
\begin{equation*}
    \sum_{i\geq 1;\; t_k\geq T^i} \frac{1}{t_k}\sum_{t=1}^{t_k} ( \1_{x_t\in A_i }  \1_{x_{\phi(t)}\notin A} + \1_{x_t\notin A}  \1_{x_{\phi(t)}\in A_i}) < \sum_{i\geq 1, t_k\geq T^i} \epsilon_i \leq \frac{\epsilon}{4}.
\end{equation*}
Combining the two above equations gives
\begin{equation}
\label{eq:main_bis}
    \frac{1}{t_k}  \sum_{t=1}^{t_k} \sum_{i\geq 1;\; t_k<T^i} ( \1_{x_t\in A_i }  \1_{x_{\phi(t)}\notin A} + \1_{x_t\notin A}  \1_{x_{\phi(t)}\in A_i}) >\frac{\epsilon}{4}.
\end{equation}
We now consider the set of times such that an input point fell into the set $A_i$ with $T^i>t_k$, either creating a mistake in the prediction of 4C1NN or inducing a later mistake within time horizon $t_k$: $ \Tcal:= \bigcup_{i\geq 1;\; T^i>t_k} \Tcal_i$ where 
\begin{equation*}
    \Tcal_i:= \left\{t\leq t_k,\; x_t\in A_i,\; \left(x_{\phi(t)}\notin A \text{ or }\exists t<u\leq t_k \text{ s.t. }\phi(u)=t,\; x_u\notin A\right)\right\}.
\end{equation*}
We now show that all points $x_t$ for $t\in \Tcal$ fall in distinct sets of the partition $\Pcal$. Indeed, because the sets $A_i$ are disjoint, it suffices to check that for any $i\geq 1$ such that $T^i>t_k$, the points $x_t$ for $t\in \Tcal_i$ fall in distinct of the following sets
\begin{equation*}
    P_{i,j}:=\left(B\left(x^j,\frac{\eta_i}{2}\right)\cap A_i\right)\setminus \bigcup_{k<j} B\left(x^k,\frac{\eta_i}{2}\right),\quad j\geq 1.
\end{equation*}
Note that for any $t_1< t_2\in \Tcal_i$ we have $x_{t_1},x_{t_2}\in A_i$ and $x_{t_1}\neq x_{t_2}$. Indeed, we cannot have $x_{t_2}=x_{t_1}$ otherwise 2C1NN would make no mistake at time $t_2$ and $x_{t_2}$ would induce no future mistake either (recall that if an input point was already visited, we use simple memorization for the prediction and do not add it to the dataset). Therefore, because the event $\Fcal_i$ is satisfied, for any $t_1< t_2\in \Tcal_i$ we have $\rho(x_{t_1},x_{t_2})>\eta_i$. Now suppose that $x_{t_1},x_{t_2}$ fall in the same set $P_{i,j}$ for $j\geq 1$, then we have $\rho(x_{t_1},x_{t_2})\leq \rho(x^i,x_{t_1}) + \rho(x^i,x_{t_2}) <\eta_i$, which is absurd. Therefore, all points $\{x_t,\; t\in \Tcal\}$ lie in different sets of the partition $\Pcal$. Therefore,
\begin{equation*}
    |\{P\in \Pcal, P\cap \mb x_{\leq t_k}\neq \emptyset\}| \geq |\Tcal|.
\end{equation*}
We now lower bound $|\Tcal|$, which will uncover the main interest of the learning rule 2C1NN. Intuitively, any input point incurs at most $1+2=3$ mistakes, contrary to the traditional 1NN learning rule. We now formalize this intuition.

\begin{align*}
    \sum_{t=1}^{t_k} \sum_{i\geq 1;\; t_k<T^i} ( \1_{x_t\in A_i }  \1_{x_{\phi(t)}\notin A} &+ \1_{x_t\notin A}  \1_{x_{\phi(t)}\in A_i}) \\
    &= \sum_{t=1}^{t_k} \sum_{i\geq 1;\; t_k<T^i} \left( \1_{x_t\in A_i }  \1_{x_{\phi(t)}\notin A} + \sum_{t<u\leq t_k}\1_{x_u\notin A}  \1_{x_t\in A_i}\1_{\phi(u)=t}\right)\\ 
    & = \sum_{i\geq 1;\; T^i>t_k} \sum_{t\leq t_k,\;x_t\in A_i}  \left(\1_{x_{\phi(t)}\notin A} + \sum_{t<u\leq t_k} \1_{x_u\notin A}\1_{\phi(u)=t}\right)\\
    &\leq \sum_{i\geq 1;\; T^i>t_k} \sum_{t\leq t_k,\;x_t\in A_i} 3\max\left(\1_{x_{\phi(t)}\notin A},\1_{x_u\notin A}\1_{\phi(u)=t},\; t<u\leq t_k\right)\\
    &= 3|\Tcal|
\end{align*}
where in the last inequality we used the fact that a given time $t$ can have at most $2$ children i.e. $|\{u>t, \phi(u)=t\}|\leq 2$ with the 2C1NN learning rule. We now use Equation (\ref{eq:main_bis}) to obtain
\begin{equation*}
     |\{P\in \Pcal, P\cap \mb x_{\leq t_k}\neq \emptyset\}| \geq |\Tcal| \geq \frac{\epsilon}{12}t_k.
\end{equation*}
This holds for any $k\geq 1$. Therefore, because $t_k\to\infty$ as $k\to\infty$ we get $|\{P\in \Pcal, P\cap \mb x_{\leq T}\neq \emptyset\}|\neq o(T).$ Finally, this holds for any realization of $\Xbb$ in the event $\Acal\cap \bigcap_{i\geq 1} \Ecal_i\cap\bigcap_{i\geq 1} \Fcal_i$. Therefore,
\begin{equation*}
    \mathbb P(|\{P\in \Pcal, P\cap \mb x_{\leq T}\neq \emptyset\}|=o(T) )\leq \mathbb P\left[\left(\Acal\cap \bigcap_{i\geq 1} \Ecal_i\cap\bigcap_{i\geq 1} \Fcal_i\right)^c\right] \leq 1-\frac{\delta}{4}<1.
\end{equation*}
Therefore, $\Xbb\notin\smv_{(\Xcal,\rho)}$ which contradicts the hypothesis. This concludes the proof that 
\begin{equation*}
    \Lcal_\Xbb(2C1NN,\1_{ A})=0 \quad (a.s.),
\end{equation*} 
and hence, $\Scal_\Xbb$ satisfies the disjoint $\sigma-$additivity property. This ends the proof of the theorem.
\end{proof}

In particular, Theorem \ref{thm:2C1NN_smv} shows that $\smv_{(\Xcal,\rho)}\subset \suol_{(\Xcal,\rho),([0,1],\ell_{01})}$. Together with Proposition \ref{prop:smv_necessary}, this shows that the set of learnable processes for binary classification is exactly $\smv_{(\Xcal,\rho)}$. As a result, 2C1NN is optimistically universal for binary classification. Applying the reduction from a general bounded output setting to binary classification from Theorem \ref{thm:invariance} \cite{blanchard2021universal} we obtain a full characterization of the set of processes admitting strong universal learning, and obtain that 2C1NN is optimistically universal for general input and output spaces.

\begin{corollary}\label{cor:suol_general}
For any separable Borel space $\Xcal$ and any separable near-metric space $(\Ycal,\ell)$ with $0<\bar \ell<\infty$, we have $\suol_{(\Xcal,\rho),(\Ycal,\ell)}=\smv_{(\Xcal,\rho)}$.
\end{corollary}

\begin{corollary}\label{cor:opt_general}
For any separable Borel space $\Xcal$, and any bounded separable near-metric space $(\Ycal,\ell)$, 2C1NN is an optimistically universal learning rule.
\end{corollary}
This completely closes the open problems in \cite{hanneke2021open} for strong universal learning.

\section{Weak universal learning}
\label{sec:weak_learning}

We now turn to weak universal learning. In this section, we show that the results for a characterization of learnable processes and existence of optimistically universal learning rule for the strong setting can also be adapted to the weak setting. Although the set of learnable processes differ---$\suol\subset\wuol$ in general and $\wuol\subsetneq\suol$ whenever $\Xcal$ is infinite \cite{hanneke2021learning}---we show that the same learning rule 2C1NN is optimistically universal in the weak setting. We start by adapting Proposition \ref{prop:consistent_ball_borel} for the weak setting by showing that 2C1NN is weakly consistent on balls under any process $\Xbb\in\wsmv$.

\begin{proposition}
Let $(\Xcal,\Bcal)$ be a separable Borel space constructed from some metric $\rho$. We consider the binary classification setting $\Ycal =\{0,1\}$ and the $\ell_{01}$ binary loss. For any input process $\Xbb\in \wsmv_{(\Xcal,\rho)}$, for any $x\in \Xcal$, and $r>0$, the learning rule 2C1NN is weakly consistent for the target function $f^* = \1_{B_\rho(x,r)}$.
\end{proposition}

\begin{proof}
The proof uses a similar structure to the proof of Proposition \ref{prop:consistent_ball_borel}. We fix $\bar x\in \Xcal$, $r>0$ and $f^*(\cdot) = \1_{B(\bar x,r)}$. We reason by the contrapositive and suppose that 2C1NN is not weakly consistent on $f^*$. We will show that the process $\Xbb$ disproves the $\wsmv_{(\Xcal,\rho)}$ condition.

Because 2C1NN is not weakly consistent for $f^*$, there exists $\epsilon$ and an increasing sequence of times $(T_l)_{l\geq 1}$ such that for any $l\geq 1$,
\begin{equation*}
    \Ebb \Lcal_\Xbb (f_\cdot,f^*;T_l) \geq \epsilon T_l.
\end{equation*}
We now define a partition $\Pcal$. Because $\Xcal$ is separable, there exists a sequence $(x^i)_{i\geq 1}$ of elements of $\Xcal$ which is dense. We focus for now on the sphere $S(\bar x,r)$ and for any $\tau>0$ we take $ (P_i(\tau))_{i\geq 1}$ the sequence of sets included in $S(\bar x,r)$ defined by
\begin{equation*}
    P_i(\tau) := \left(S(\bar x,r)\cap B(x^i,\tau)\right) \setminus \left(\bigcup_{1\leq j<i} B(x^j,\tau)\right) .
\end{equation*}
These sets form a partition of $S(\bar x,r)$ as shown in the proof of Proposition \ref{prop:consistent_ball_borel}. We now pose $\tau_l := c_\epsilon\cdot \frac{r}{2^{l+1}}$, for $l\geq 1$, where $c_\epsilon := \frac{1}{2\cdot 2^{2^5/\epsilon}}$ is a constant dependant on $\epsilon$ only. We also pose $\tau_0=r$. Then, because $\Xbb\in \wsmv_{(\Xcal,\rho)}$, the expected number of sets visited of $\Pcal_i(\tau_l)$ tends to $0$. Therefore, there exists an increasing sequence $(n_l)_{l\geq 1}$ such that for any $l\geq 1$,
\begin{equation*}
    \forall n\geq n_l,\quad \Ebb[|\{i,\; P_i(\tau_l)\cap \Xbb_{< n}\neq \emptyset \} |] \leq \frac{\epsilon^2}{2^{10}} n\quad \text{ and } \quad n_{l+1} \geq \frac{2^6}{\epsilon}n_l
\end{equation*}
Now, for any $l\geq 1$, we now construct $\mu_l>0$ such that
\begin{equation*}
     \mathbb P\left[\min_{i<j\leq n_{l},\; X_i\neq X_j} \rho(X_i,X_j) > \mu_l \right] \geq 1- \frac{\epsilon}{2^{l+3}}.
\end{equation*}
We denote by $\Fcal_l$ this event. Therefore $\Pbb[\Fcal_l^c]\leq \frac{\epsilon}{2^{l+3}}$.
Note that the sequence $(\mu_l)_{l\geq 1}$ is non-increasing. We now define radiuses $(z^i)_{i\geq 1}$ as follows:
\begin{equation*}
    z^i=\begin{cases}
      \mu_{l_i+1}& \text{if }\rho(x^i,\bar x)<r,\text{ where } \frac{r}{2^{l_i+1}} < r-\rho(x^i,\bar x) \leq \frac{r}{2^{l_i}}\\
      0 & \text{if }\rho(x^i,\bar x)\geq r,
    \end{cases}
\end{equation*}
and consider the sets $R_i:=B(x^i,z^i) \cap \left\{x\in \Xcal:\;  \rho(x,\bar x) <r-\frac{r}{2^{l_i+2}}\right\}$. We construct $P_i := R_i \setminus \left(\bigcup_{k<i} R_k \right),$ for $i\geq 1$. By Lemma \ref{lemma:P_i_partition}, $(P_i)_{i\geq 1}$ forms a partition of $B(\bar x,r)$. We now define a second partition $(A_i)_{i\geq 1}$ similarly as in the proof of Proposition \ref{prop:consistent_ball_borel}. We start by defining a sequence of radiuses $(r^i)_{i\geq 1}$ as follows
\begin{equation*}
    r^i=\begin{cases}
      \displaystyle c_\epsilon\inf_{x:\;\rho(x,\bar x)\leq r} \rho(x^i, x) & \text{if }\rho(x^i,\bar x)>r,\\
      \displaystyle c_\epsilon \inf_{x:\;\rho(x,\bar x)\geq r} \rho(x^i, x) & \text{if }\rho(x^i,\bar x)<r,\\
      0 & \text{if }\rho(x^i,\bar x)=r,
    \end{cases}
\end{equation*}
and consider the sets $(A_i)_{i\geq 0}$ given by $A_0 = S(\bar x,r)$ and for $i\geq 1$, $A_i = B(x^i,r^i)\setminus \left(\bigcup_{1\leq j<i} B(x^j,r^j)\right)$. By Lemma \ref{lemma:A_i_partition}, this forms a partition of $\Xcal$. We now formally consider the product partition of $(P_{i})_{i\geq 1}$ and $(A_i)_{i\geq 0}$ i.e.
\begin{equation*}
   \Qcal:\quad  \bigcup_{i\geq 0,\; A_i\subset B(\bar x,r)} \bigcup_{j\geq 1} (A_i\cap P_{j}) \cup  \bigcup_{i\geq 0,\; A_i\subset \Xcal\setminus B(\bar x,r)} A_i.
\end{equation*}
where we used the fact that sets $A_i$ satisfy either $A_i\subset B(\bar x,r)$ or $A_i\subset \Xcal\setminus B(\bar x,r)$. We will show that this partition disproves the $\wsmv_{(\Xcal,\rho)}$ hypothesis on $\Xbb$.\\

We now fix $l_0\geq 1$ such that $T_{l_0}\geq n_2$ and consider $l\geq l_0$. We focus on time $T_l$. Define the event $\Acal:=\{\Lcal_\Xbb(f_\cdot,f^*;T_l)\geq \frac{\epsilon}{2}T_l\}$. Note that we have
\begin{equation*}
    \Ebb \Lcal_\Xbb(f_\cdot,f^*;T_l) \leq \frac{\epsilon}{2}T_l + \Pbb[\Acal]T_l.
\end{equation*}
Therefore, $\Pbb[\Acal]\geq \frac{\epsilon}{2}$. Also, because $(n_u)_{u\geq 1}$ is an increasing sequence, let $u\geq 1$ such that $n_{u+1}\leq T_l\leq n_{u+2}$. We define the event $\Ecal=\{|\{i\; P_i(\tau_u)\cap \Xbb_{\leq T_l}\neq \emptyset\}|\leq \frac{\epsilon}{2^{7}}T_l \}$. Then, we have by construction 
\begin{equation*}
    \frac{\epsilon^2}{2^{10}}T_l\geq \Ebb |\{i\; P_i(\tau_u)\cap \Xbb_{\leq T_l}\neq \emptyset\}| \geq \frac{\epsilon}{2^{7}}T_l \Pbb[\Ecal^c].
\end{equation*}
Therefore, we have $\Pbb[\Ecal^c]\leq \frac{\epsilon}{8}$. Consider a specific realization $\mb x = (x_t)_{t\geq 0}$ of the process $\Xbb$ falling in the event $\Acal\cap \Ecal\cap \bigcap_{l\geq 1}\Fcal_l$. This event has probability
\begin{equation*}
    \Pbb\left[\Acal\cap \Ecal\cap \bigcap_{l\geq 1}\Fcal_l\right] \geq \Pbb[\Acal] - \Pbb[\Ecal^c]-\sum_{l\geq 1}  \Pbb[\Fcal_l^c] \geq \frac{\epsilon}{2} - \frac{\epsilon}{8} - \frac{\epsilon}{8} = \frac{\epsilon}{4}.
\end{equation*}
Note that $\mb x$ is not random anymore. We now show that $\mb x$ visits a large number of sets in the partition $\Qcal$. We now denote by $(t_k)_{k\geq 1}$ the increasing sequence of all times when 2C1NN makes an error in the prediction of $f^*(x_t)$. Define $k_l$ such the last time of error before $T_l$ i.e. $k_l=\max \{k\geq 1,\; t_k\leq T_l\}$. By construction, because $\Acal$ is met we have $k_l \geq \frac{\epsilon}{2}T_l$.

At an iteration where the new input $x_t$ has not been previously visited we will denote by $\phi(t)$ the index of the nearest neighbor of the current dataset in the 2C1NN learning rule. Now let $l\geq 1$. Consider the tree $\Gcal$ where nodes are times $\Tcal:=\{t,\; t\leq T_l,\; x_t\notin\{x_u, u<t\} \}$ for which a new input was visited, where the parent relations are given by $(t,\phi(t))$ for $t\in \Tcal\setminus\{1\}$. Again, each node has at most $2$ children and a node is not in the dataset at time $T_l$ when it has exactly $2$ children.

\paragraph{Step 1.}
We now suppose that the majority of input points on which 2C1NN made a mistake belong to the $B(\bar x,r)$ i.e.
\begin{equation*}
    \left|\left\{t\leq T_l,\; \ell_{01}(2C1NN(\mb x_{<t},\mb y_{<t},x_t),f^*(x_t))=1,\; x_t\in B(\bar x,r) \right\}\right| \geq \frac{k_l}{2},
\end{equation*}
or equivalently $|\{k\leq k_l,\; x_{t_k}\in B(\bar x,r)\}|\geq \frac{k_l}{2}$.

Let us now consider the subgraph $\tilde \Gcal$ given by restricting $\Gcal$ only to nodes in the the ball $B(\bar x,r)$ which are mapped to the true value $1$ i.e. on times $\{t\in \Tcal,\; x_t\in B(\bar x,r)\}$. As in the proof of Proposition \ref{prop:consistent_ball_borel}, $\tilde \Gcal$ is a collection of disjoint trees with roots times $\{t_k, \; k\leq k_l, \; x_{t_k}\in B(\bar x,r)\}$---and possibly $t=1$ if $x_1\in B(\bar x,r)$. For a given time $t_k$ with $k\leq k_l$ and $x_{t_k}\in B(\bar x,r)$, denote $\Tcal_k$ the corresponding tree in $\tilde \Gcal$ with root $t_k$. We will say that the tree $\Tcal_k$ is \emph{sparse} if
\begin{equation*}
    \forall t\in \Tcal_k,\quad \left|\left\{u
    \leq T_l,\; \phi(u) = t,\; \rho(x_u,\bar x)< r\right\}\right|\leq 1 \quad \text{and} \quad |\Tcal_k|\leq \frac{16}{\epsilon}.
\end{equation*}
We denote by $S = \{k\leq k_l,\; \rho(x_{t_k},\bar x)<r,\; \Tcal_k\text{ sparse}\} $ the set of sparse trees. Similarly as in the proof of Proposition \ref{prop:consistent_ball_borel}, we have $|S|\geq \frac{k_l}{8}$. We now focus only on sparse trees $\Tcal_k$ for $k\in S$ and analyze their relation with the final dataset $\Dcal_{T_l+1}$. Precisely, for a sparse tree $\Tcal_k$, denote $\Vcal_k = \Tcal_k\cap \Dcal_{T_l+1}$ the set of times which are present in the final dataset and belong to the tree induced by error time $t_k$. Because each node of $\Tcal_k$ and not present in $\Dcal_{T_l+1}$ has at least $1$ children in $\Tcal$, we note that $\Vcal_k\neq\emptyset$. We now consider the path from a node of $\Vcal_k$ to the root $t_k$. We denote by $d(k)$ the depth of this node in $\Vcal_k$ and denote the path by $p_{d(k)}^k\to p_{d(k)-1}^k\to p_0^k=t_k$ where $p_{d(k)}^k\in \Vcal_k$. Then we have, $ d(k)\leq |\Tcal_k|-1 \leq \frac{16}{\epsilon} - 1.$ The same arguments as in the proof of Proposition \ref{prop:consistent_ball_borel} show that all the points $\{p_{d(k)}^k,\; k\in S\}$ fall in distinct sets of the partition $(A_i)_{i\geq 0}$. Therefore,
\begin{equation*}
    |\{i,\; A_i\cap \mb{x}_{\leq T_l}\neq \emptyset \}|\geq |S|\geq \frac{k_l}{8} \geq \frac{\epsilon}{16}T_l.
\end{equation*}

\paragraph{Step 2.}
We now turn to the case when the majority of input points on which 2C1NN made a mistake are not in the ball $B(\bar x,r)$ i.e.
\begin{equation*}
    \left|\left\{t\leq t_{k_l},\; \ell_{01}(2C1NN(\mb x_{<t},\mb y_{<t},x_t),f^*(x_t))=1,\; \rho(x_t,\bar x)\geq r \right\}\right| \geq \frac{k_l}{2},
\end{equation*}
or equivalently $|\{k\leq k_l,\; \rho(x_{t_k},\bar x)\geq r\}|\geq \frac{k_l}{2}$. Similarly as the previous case, we consider the graph $\tilde G$ given by restricting $\Gcal$ only to nodes outside the ball $B(\bar x,r)$ i.e. on times $\{t\in \Tcal, \rho(x_t,\bar x)\geq r\}$. Again, $\tilde \Gcal$ is a collection of disjoint trees with root times $\{t_k,\; k\leq k_l,\; \rho(x_{t_k},\bar x)\geq r\}$---and possibly $t=1$. We denote $\Tcal_k$ the corresponding tree of $\tilde \Gcal$ rooted in $t_k$. Similarly to above, a tree is \emph{sparse} if
\begin{equation*}
    \forall t\in \Tcal_k,\quad \left|\left\{u
    \leq T_l,\; \phi(u) = t,\; \rho(x_u,\bar x)< r\right\}\right|\leq 1 \quad \text{and} \quad |\Tcal_k|\leq \frac{16}{\epsilon}.
\end{equation*}
If $S=\{k\leq k_l,
; \rho(x_{t_k},\bar x)\geq r,\; \Tcal_k\text{ sparse}\}$ denotes the set of sparse trees, the same proof as above shows that $|S|\geq \frac{k_l}{8}$. Again, for any $k\in S$, if $d(k)$ denotes the depth of some node from $\Vcal_k:=\Tcal_k\cap \Dcal_{t_{k_l}}$ in $\Tcal_k$ we have $d(k)\leq  \frac{16}{\epsilon}-1$. For each $k\in S$ we consider the path from this node of $\Vcal_k $ to the root $t_k$: $p_{d(k)}^k\to p_{d(k)-1}^k\to \ldots \to p_0^k=t_k$ where $p_{d(k)}^k\in \Vcal_k$. The same proof as above shows that all the points $\{p^k_{d(k)},\; k\in S,\; \rho(x_{p^k_{d(k)}},\bar x)>r\}$ lie in distinct sets of the partition $(A_i)_{i\geq 0}$. Suppose $|\{k\in S,\; \rho(x_{p^k_{d(k)}},\bar x)>r\}|\geq \frac{|S|}{2}$, then we have
\begin{equation*}
    |\{i,\; A_i\cap \mb{x}_{\leq T_l}\neq \emptyset \}|\geq |\{k\in S,\; \rho(x_{p^k_{d(k)}},\bar x)>r\}| \geq \frac{|S|}{2} \geq \frac{k_l}{16}\geq \frac{\epsilon}{32}T_l.
\end{equation*}

\paragraph{Step 3.}
In this last step, we suppose again that the majority of input points on which 2C1NN made a mistake are not in the ball $B(\bar x,r)$ and that $|\{k\in S,\; \rho(x_{p^k_{d(k)}},\bar x)>r\}|< \frac{|S|}{2}$. Therefore, we obtain
\begin{equation*}
    |\{k\in S,\; \rho(x_{p^k_{d(k)}},\bar x)=r\}| = |S|-|\{k\in S,\; \rho(x_{p^k_{d(k)}},\bar x)>r\}| \geq \frac{|S|}{2} \geq \frac{k_l}{16} \geq \frac{\epsilon}{32}T_l.
\end{equation*}
We will now make use of the partition $(P_i)_{i\geq 1}$. Recall that $u\geq 1$ was defined such that $n_{ u+1}\leq T_l\leq n_{ u+2}$. Note that we have $n_u\leq \frac{\epsilon}{2^6}n_{u+1}\leq \frac{\epsilon}{2^6}T_l$. Let us now analyze the process between times $n_u$ and $T_l$. In particular, we are interested in the indices $T=\{k\in S,\; \rho(x_{p^k_{d(k)}},\bar x)=r\}$ and times $\Ucal_u = \{p^k_{d(k)}:\;n_u< p^k_{d(k)}\leq k_l,\; k\in T\}$. We have
\begin{equation*}
    |\Ucal_u| \geq  |\{k\in S,\; \rho(x_{p^k_{d(k)}},\bar x)=r\}| - n_u  \geq \frac{\epsilon}{32}T_l -\frac{\epsilon}{2^6}T_l = \frac{\epsilon}{2^6}T_l.
\end{equation*}
Because the event $\Ecal_u$ is met, we have
\begin{equation*}
    |\{i,\; P_i(\tau_u)\cap \mb x_{\Ucal_u} \neq \emptyset\}| \leq |\{i,\; P_i(\tau_u)\cap \mb x_{\leq T_l}\neq \emptyset\}| \leq \frac{\epsilon}{2^7}T_l.
\end{equation*}
The same arguments as in the proof of Proposition \ref{prop:consistent_ball_borel} show that defining $T' := \{k\in T,\; r-\frac{r}{2^{u+2}}\leq \rho(x_{t_k},\bar x)<r\}$ we obtain 
\begin{equation*}
    \left| T'\right| \geq |\Ucal_u|-|\{i,\; P_i(\tau_u)\cap \mb x_{\Ucal_u} \neq \emptyset\}| \geq \frac{\epsilon}{2^7}T_l.
\end{equation*}
We will now show that all the points in $\{x_{t_k},\; k\in T'\}$ lie in distinct sets of $(P_i)_{i\geq 1}$. Note that because we have $T_l\leq n_{u+2}$ and because the event $\Fcal_{u+2}$ is met, we have that for any $p,q\in  T'$ that $\rho(x_{\phi(t_p)},x_{\phi(t_q)})> \mu_{u+2}.$ Now suppose by contradiction that $x_{\phi(t_p)},x_{\phi(t_q)}\in P_i$ for some $i\geq 1$. Then, with $l_i$ such that $r-\frac{r}{2^{l_i}}\leq \rho(x^i,\bar x) < r-\frac{r}{2^{l_i+1}}$ we have that 
\begin{equation*}
    x_{\phi(t_p)},x_{\phi(t_q)} \in \left\{x\in \Xcal:\; \rho(x,\bar x) <r-\frac{r}{2^{l_i+2}}\right\}
\end{equation*}
But we know that $ \rho(x_{\phi(t_p)},\bar x)\geq r-\frac{r}{2^{u+2}}$. Therefore we obtain $r-\frac{r}{2^{l_i+2}}> r-\frac{r}{2^{u+2}}$ and hence $l_i\geq u+1$. Recall that $P_i\subset B(x^i,\mu_{l_i+1})$. Therefore, we obtain $\rho(x_{\phi(t_p)},x_{\phi(t_q)}) \leq \mu_{l_i+1} \leq \mu_{u+2},$ which contradicts the fact that $\rho(x_{\phi(t_p)},x_{\phi(t_q)})> \mu_{u+2}.$ This ends the proof that all points of $\{x_{\phi(t_k)},\; k\in T'\}$ lie in distinct subsets of $(P_i)_{i\geq 1}$. Now we obtain
\begin{equation*}
    |\{i,\; P_i\cap \mb x_{\leq T_l} \neq \emptyset\}| \geq |T'| \geq \frac{\epsilon}{2^7}T_l.
\end{equation*}

\paragraph{Step 4.}
In conclusion, in all cases, we obtain 
\begin{equation*}
    |\{Q\in \Qcal,\; Q\cap \mb x_{\leq T_l} \neq \emptyset\}| \geq \max(|\{i,\; A_i\cap \mb{x}_{\leq T_l}\neq \emptyset \}|,|\{i,\; P_i\cap \mb x_{\leq T_l} \neq \emptyset\}|) \geq \frac{\epsilon}{2^7}T_l.
\end{equation*}
Recall that this holds for any realization $\mb x$ in the event $\Acal\cap \Ecal\cap \bigcap_{l\geq 1}\Fcal_l$. Therefore,
\begin{equation*}
    \Ebb [|\{Q\in \Qcal,\; Q\cap \Xbb_{\leq T_l}\neq \emptyset \}|] \geq \Pbb\left[\Acal\cap \Ecal\cap \bigcap_{l\geq 1}\Fcal_l\right] \frac{\epsilon}{2^7}T_l\geq \frac{\epsilon^2}{2^{9}}T_l.
\end{equation*}
Because this is true for all $l\geq l_0$ and $T_l$ is an increasing sequence, we conclude that $\Xbb\notin \wsmv_{(\Xcal,\rho)}$ which is absurd. Therefore 2C1NN is consistent on $f^*$.
\end{proof}

We now show that 2C1NN is weakly consistent under processes of $\wsmv_{(\Xcal,\rho)}$ for binary classification adapting the proof of Theorem \ref{thm:2C1NN_smv}.

\begin{theorem}
\label{thm:2C1NN_wsmv}
Let $(\Xcal,\Bcal)$ be a separable Borel space constructed from the metric $\rho$. For the binary classification setting, the learning rule 2C1NN is weakly universally consistent for all processes $\Xbb\in \wsmv_{(\Xcal,\rho)}$.
\end{theorem}

\begin{proof}
Again, we follow a similar proof to that of Theorem \ref{thm:2C1NN_smv}. Let $\Xbb\in \wsmv_{(\Xcal,\rho)}$ and consider the set $\Scal_\Xbb$ of functions for which it is weakly consistent $\Scal_\Xbb:= \{A\in \Bcal,\quad \Ebb\Lcal_\Xbb(2C1NN,\1_{ A})\to 0 \}$. By construction we have $\Scal_\Xbb\subset\Bcal$. The goal is to show that in fact $\Scal_\Xbb = \Bcal$. To do so, we will show that $\Scal$ satisfies the following properties
\begin{itemize}
    \item $\emptyset\in \Scal_\Xbb$ and $\Scal_\Xbb$ contains all balls $B(x,r)$ with $x\in \Xcal$ and $r\geq 0$,
    \item if $A\in \Scal_\Xbb$ then $A^c\in \Scal_\Xbb$ (stable to complementary),
    \item if $(A_i)_{i\geq 1}$ is a sequence of disjoint sets of $\Scal_\Xbb$, then $\bigcup_{i\geq 1} A_i\in \Scal_\Xbb$ (stable to $\sigma-$additivity for disjoint sets),
    \item if $A,B\in \Scal_\Xbb$, then $A\cup B\in \Scal_\Xbb$ (stable to union).
\end{itemize}
Together, these properties show that $\Scal_\Xbb$ is a $\sigma-$algebra that contains all open intervals of $\Xcal$. The invariance to complementary is again due to the fact that 2C1NN is invariant to relabeling. Further, we clearly have $\emptyset \in \Scal_\Xbb$. Now let $x\in \Xcal$ and $r\geq 0$, Proposition \ref{prop:consistent_ball_borel} shows that $B(x,r)\in \Scal_\Xbb$.\\

We now turn to the $\sigma-$additivity for disjoint sets. Let $(A_i)_{i\geq 1}$ is a sequence of disjoint sets of $\Scal_\Xbb$. We denote $A:= \bigcup_{i\geq 1} A_i$. We consider the target function $f^* = \1_{A}$. We write the average loss in the following way,
\begin{equation*}
     \frac{1}{T}\sum_{t=1}^{T}\ell_{01}(2C1NN(\Xbb_{< t},\Ybb_{< t}, X_{t}), f^*(X_t)) = \frac{1}{T}\sum_{t=1}^{T} \1_{X_t\in A}  \1_{X_{\phi(t)}\notin A} + \frac{1}{T}\sum_{t=1}^{T}\1_{X_t\notin A}  \1_{X_{\phi(t)}\in A}.
\end{equation*}
We suppose by contradiction that 2C1NN is not weakly consistent on $f^*$. Then there exists $\epsilon>0$ and an increasing sequence of times $(T_l)_{l\geq 1}$ such that $\Ebb\Lcal_\Xbb(2C1NN,f^*;T_l)\geq \epsilon T_l$. We first analyze the errors induced by one set $A_i$ only. Simililarly to the proof of Theorem \ref{thm:2C1NN_smv} we have 
\begin{equation*}
    \frac{1}{T}\sum_{t=1}^{T} ( \1_{X_t\in A_i }  \1_{X_{\phi(t)}\notin A} + \1_{X_t\notin A}  \1_{X_{\phi(t)}\in A_i}) \leq \frac{1}{T}\sum_{t=1}^{T} \ell_{01}(2C1NN(\Xbb_{< t},\1_{\Xbb_{< t}\in A_i}, X_{t}), \1_{X_t\in A_i}).
\end{equation*}
Then, because 2C1NN is consistent for $\1_{\cdot\in A_i}$, we get
\begin{equation*}
    \Ebb\left[\frac{1}{T}\sum_{t=1}^{T} ( \1_{X_t\in A_i }  \1_{X_{\phi(t)}\notin A} + \1_{X_t\notin A}  \1_{X_{\phi(t)}\in A_i})  \right]\to 0.
\end{equation*}
We take $\epsilon_i = \frac{\epsilon}{4\cdot 2^i}$ and $T^i$ such that
\begin{equation*}
     \forall T\geq T^i,\quad \Ebb \left[\frac{1}{T}\sum_{t=1}^{T} ( \1_{X_t\in A_i }  \1_{X_{\phi(t)}\notin A} + \1_{X_t\notin A}  \1_{X_{\phi(t)}\in A_i})  \right]<\frac{\epsilon_i^2}{2}.
\end{equation*}
We now consider the scale of the process $\Xbb_{\leq T^i}$ when falling in $A_i$, by introducing $\eta_i>0$ such that
\begin{equation*}
    \mathbb P\left[ \min_{ \substack{
        t_1,t_2 \leq T^i;\; X_{t_1},X_{t_2}\in A_i; \\
        X_{t_1}\neq X_{t_2}
     }} \rho(X_{t_1},X_{t_2}) > \eta_i \right]  \geq 1-\frac{\epsilon_i}{2}.
\end{equation*}
We denote by $\Fcal_i$ this event. Thus, $\Pbb[\Fcal_i^c]\leq \frac{\epsilon_i}{2}$. We now construct a partition $\Pcal$ obtained by subdividing each set $A_i$ according to scale $\eta_i$. Because $\Xcal$ is separable, there exists a sequence of points $(x^j)_{j\geq 1}$ in $\Xcal$ such that $\forall x\in \Xcal, \inf_{j\geq 1}\rho(x,x^j)=0.$ We construct the following partition of $\Xcal$ given by
\begin{equation*}
    \Pcal\; : \quad A^c\cup \bigcup_{i\geq 1} \bigcup_{j\geq 1}\left\{ \left(B\left(x^j,\frac{\eta_i}{2}\right)\cap A_i\right)\setminus \bigcup_{k<j} B\left(x^k,\frac{\eta_i}{2}\right)\right\}.
\end{equation*}
We now fix $l\geq 1$ and consider the event $\Acal:=\{\Lcal_\Xbb(2C1NN,f^*;T_l)\geq \frac{\epsilon}{2}\}$. Note that
\begin{equation*}
    \epsilon T_l\leq \Ebb \Lcal_\Xbb(2C1NN,f^*;T_l)\leq \frac{\epsilon}{2}T_l  + \Pbb[\Acal]T_l,
\end{equation*}
which gives $\Pbb[\Acal]\geq \frac{\epsilon}{2}$. We also define the following event
\begin{equation*}
    \Ecal_i = \left\{\frac{1}{T_l}\sum_{t=1}^{T_l} ( \1_{X_t\in A_i }  \1_{X_{\phi(t)}\notin A} + \1_{X_t\notin A}  \1_{X_{\phi(t)}\in A_i}) <\epsilon_i  \right\},
\end{equation*}
for any $i\in I:=\{i\geq 1,\; T_l\geq T^i\}$. Then, we have
\begin{equation*}
    \frac{\epsilon_i^2}{2}\geq \Ebb \left[\frac{1}{T_l}\sum_{t=1}^{T_l} ( \1_{X_t\in A_i }  \1_{X_{\phi(t)}\notin A} + \1_{X_t\notin A}  \1_{X_{\phi(t)}\in A_i})  \right] \geq \epsilon_i \Pbb[\Ecal_i^c],
\end{equation*}
which yields $\Pbb[\Ecal_i^c]\leq \frac{\epsilon_i}{2}$. We will now focus on the event $\Acal\cap \bigcap_{i\in I} \Ecal_i\cap\bigcap_{i\geq 1} \Fcal_i$, which has probability $\Pbb[\Acal\cap \bigcap_{i\in I} \Ecal_i\cap\bigcap_{i\geq 1} \Fcal_i]\geq \Pbb(\Acal) - \sum_{i\in I}\Pbb[\Ecal_i^c]-\sum_{i\geq 1}\Pbb[\Fcal_i^c]\geq \frac{\epsilon}{2}-\frac{\epsilon}{4}=\frac{\epsilon}{4}$. Let us now consider a realization of $\mb x$ of $\Xbb$ in the event $\Acal\cap \bigcap_{i\in I} \Ecal_i\cap\bigcap_{i\geq 1} \Fcal_i$. The sequence $\mb x$ is now not random anymore. We will show that $\mb x$ does visits a linear number of sets in the partition $\Pcal$.

Because the event $\Acal$ is met, we have
\begin{equation*}
   \sum_{i\geq 1} \frac{1}{T_l}\sum_{t=1}^{T_l} ( \1_{x_t\in A_i }  \1_{x_{\phi(t)}\notin A} + \1_{x_t\notin A}  \1_{x_{\phi(t)}\in A_i}) \geq \frac{\epsilon}{2}.
\end{equation*}
Also, because the events $\Ecal_i$ are met, we have
\begin{equation*}
    \sum_{i\in I} \frac{1}{T_l}\sum_{t=1}^{T_l} ( \1_{x_t\in A_i }  \1_{x_{\phi(t)}\notin A} + \1_{x_t\notin A}  \1_{x_{\phi(t)}\in A_i}) \leq \sum_{i\in I} \epsilon_i \leq \frac{\epsilon}{4}.
\end{equation*}
Combining the two above equations gives
\begin{equation}
\label{eq:main_bis_weak}
    \frac{1}{T_l}  \sum_{t=1}^{T_l} \sum_{i\notin I} ( \1_{x_t\in A_i }  \1_{x_{\phi(t)}\notin A} + \1_{x_t\notin A}  \1_{x_{\phi(t)}\in A_i}) >\frac{\epsilon}{4}.
\end{equation}
We now consider the set of times such that an input point fell into the set $A_i$ with $i\notin I$, either creating a mistake in the prediction of 4C1NN or inducing a later mistake within time horizon $T_l$: $ \Tcal:= \bigcup_{i\notin I} \Tcal_i$ where 
\begin{equation*}
    \Tcal_i:= \left\{t\leq T_l,\; x_t\in A_i,\; \left(x_{\phi(t)}\notin A \text{ or }\exists t<u\leq T_l \text{ s.t. }\phi(u)=t,\; x_u\notin A\right)\right\}.
\end{equation*}
Because the events $\Fcal_i$ are met, the same arguments as in the proof of Theorem \ref{thm:2C1NN_smv} show that all points $x_t$ for $t\in \Tcal$ fall in distinct sets of the partition $\Pcal$, i.e. $|\{P\in \Pcal, P\cap \mb x_{\leq t_k}\neq \emptyset\}| \geq |\Tcal|.$ We also obtain with the same arguments
\begin{equation*}
    \sum_{t=1}^{t_k} \sum_{i\notin I} ( \1_{x_t\in A_i }  \1_{x_{\phi(t)}\notin A} + \1_{x_t\notin A}  \1_{x_{\phi(t)}\in A_i})\leq 3|\Tcal|.
\end{equation*}
We now use Equation (\ref{eq:main_bis_weak}) to obtain $|\{P\in \Pcal, P\cap \mb x_{\leq t_k}\neq \emptyset\}| \geq |\Tcal| \geq \frac{\epsilon}{12}T_l.$ Therefore, because this holds for any realization in $\Acal\cap \bigcap_{i\in I} \Ecal_i\cap\bigcap_{i\geq 1} \Fcal_i$ we obtain
\begin{equation*}
    \Ebb[|\{ P\in \Pcal,\; P\cap \Xbb_{\leq T_l}\neq \emptyset \}|] \geq \Pbb\left[\Acal\cap \bigcap_{i\in I} \Ecal_i\cap\bigcap_{i\geq 1} \Fcal_i\right] \frac{\epsilon}{12}T_l \geq \frac{\epsilon^2}{48}T_l.
\end{equation*}
This holds for any $l\geq 1$. Therefore, because $(T_l)_{l\geq 1}$ is an increasing sequence, this shows that $\Xbb\notin \wsmv_{(\Xcal,\rho)}$ which contradicts the hypothesis. This concludes the proof that $A\in \Scal_\Xbb$ and hence, $\Scal_\Xbb$ satisfies the disjoint $\sigma-$additivity property.\\

We now show that $\Scal_\Xbb$ is invariant to finite unions. Let $A_1,A_2\in \Scal_\Xbb$. We consider $A=A_1\cup A_2$ and $f^*(\cdot)=\1_{\cdot\in A}$. Using the same arguments as above, we still have for $T\geq 1$,
\begin{equation*}
    \Ebb\left[ \frac{1}{T}\sum_{t=1}^{T} ( \1_{X_t\in A_i }  \1_{X_{\phi(t)}\notin A} + \1_{X_t\notin A}  \1_{X_{\phi(t)}\in A_i})\right] \to 0.
\end{equation*}
for $i\in\{1,2\}$. But note that for any $T\geq 1$,
\begin{align*}
     \frac{1}{T}\Lcal_\Xbb(2C1NN,f^*;T) &= \frac{1}{T}\sum_{t=1}^{T} \1_{X_t\in A}  \1_{X_{\phi(t)}\notin A} + \frac{1}{T}\sum_{t=1}^{T}\1_{X_t\notin A}  \1_{X_{\phi(t)}\in A}\\
     &\leq \frac{1}{T}\sum_{t=1}^{T} (\1_{X_t\in A_1}+\1_{X_t\in A_2})  \1_{X_{\phi(t)}\notin A} + \frac{1}{T}\sum_{t=1}^{T}\1_{X_t\notin A}  (\1_{X_{\phi(t)}\in A_1}+\1_{X_{\phi(t)}\in A_2})\\
     &= \sum_{i=1}^2 \frac{1}{T}\sum_{t=1}^{T} ( \1_{X_t\in A_i }  \1_{X_{\phi(t)}\notin A} + \1_{X_t\notin A}  \1_{X_{\phi(t)}\in A_i}) .
\end{align*}
Therefore we obtain directly $\Ebb \left[\frac{1}{T}\Lcal_\Xbb(2C1NN,f^*;T)\right]\to 0$. This shows that $A_1\cup A_2\in \Scal_\Xbb$ and ends the proof of the theorem.
\end{proof}

We now turn to the case of a bounded separable output setting $(\Ycal,\ell)$ and show that 2C1NN is weakly optimistically universal.

\begin{theorem}
\label{thm:2C1NN_wsmv_general}
Let $(\Xcal,\Bcal)$ be a separable Borel space constructed from the metric $\rho$. The learning rule 2C1NN is weakly universally consistent for all processes $\Xbb\in \wsmv_{(\Xcal,\rho)}$ and any bounded output setting $(\Ycal,\ell)$.
\end{theorem}

\begin{proof}
We fix an output setting $(\Ycal,\ell)$ and let $\Xbb\in\wsmv_{(\Xcal,\rho)}$. We will show that 2C1NN is weakly universally consistent on $\Xbb$ for $(\Ycal,\ell)$. 

We first start by showing that it is weakly universally consistent for classification with countable number of classes $(\Nbb,\ell_{01})$. We fix a target function $f^*:\Xcal\to \Nbb$. For any $i\in \Nbb$ we define the binary function $f^*_i:=\1(f^*(\cdot)=i)$. We define
\begin{equation*}
    \Lcal_i(T):= \sum_{t=1}^T \1_{f^*(x_t)=i} \ell_{01}(f^*(x_{\phi(t)}),f^*(x_t))
\end{equation*}
for all $i\geq 0$. Then,
\begin{equation*}
    \Lcal_i(T) = \frac{1}{T}\sum_{t=1}^T \1_{f^*(x_t)=i} \ell_{01}(f^*_i(x_{\phi(t)}),f^*_i(x_t))
    \leq  \Lcal_\Xbb(2C1NN,f^*_i;T)
\end{equation*}
Therefore, because 2C1NN is weakly universally consistent, we have $\Ebb \Lcal_\Xbb(2C1NN,f^*_i;T)\to 0$, hence $\Ebb \Lcal_i(T)\to 0$ for all $i\geq 0$. Since $\Lcal_i(T)\geq 0$ and $\sum_{i\geq 0} \Lcal_i(T) = \Lcal_\Xbb(2C1NN,f^*;T)\leq 1$, we can apply the dominated convergence theorem and obtain
\begin{equation*}
    \Ebb  \Lcal_\Xbb(2C1NN,f^*;T)\to 0,
\end{equation*}
which proves that 2C1NN is weakly universally consistent for classification with countable number of classes.

We now turn to the general setting $(\Ycal,\ell)$. Let $(y^i)_{i\geq 1}$ be a a dense sequence on $\Ycal$ with respect to $\ell$, let $\epsilon>0$ and consider the function $h(y):=\inf\{i\geq 1:\; \ell(y^i,y)<\epsilon\}$. Then, we have
\begin{align*}
    \ell(y_{\phi(t)},y_t) &\leq \bar \ell \cdot \1_{h(y_{\phi(t)}\neq h(y_t)} + \ell(y_{\phi(t),y_t})\1_{h(y_{\phi(t)}= h(y_t)}\\
    &\leq\bar \ell \cdot \ell_{01}\1_{h\circ f^*(x_{\phi(t)})\neq h\circ f^*(x_t)} + c_\ell(\ell(y_{\phi(t)},y^{h(y_{\phi(t)})})+\ell(y^{h(y_t)},y_t))\\
    &\leq \bar \ell \cdot \ell_{01}\1_{h\circ f^*(x_{\phi(t)})\neq h\circ f^*(x_t)} + 2c_\ell \epsilon.
\end{align*}
This yields $\Lcal_\Xbb(2C1NN,f^*;T)\leq \bar \ell \Lcal_\Xbb(2C1NN,h\circ f^*;T) + 2c_\ell\epsilon$. Because 2C1NN is weakly universally consistent for countably-many classification, we have $\Ebb \Lcal_\Xbb(2C1NN,h\circ f^*;T) \to 0$. Therefore, we obtain
\begin{equation*}
    \limsup_T\Ebb \Lcal_\Xbb(2C1NN,f^*;T) \leq 2c_\ell \epsilon.
    \end{equation*}
This holds for any $\epsilon>0$ therefore, $\Ebb \Lcal_\Xbb(2C1NN,f^*;T)\to 0$, which ends the proof that 2C1NN is weakly universally consistent on $\Xbb$ for the setting $(\Ycal,\ell)$.
\end{proof}

As an immediate consequence, we have $\wsmv_{(\Xcal,\rho)}\subset \wuol_{(\Xcal,\rho),(\Ycal,\ell)}$. Together with Proposition \ref{prop:consistent_interval} we obtain a complete characterization for weak learnable processes.
\begin{corollary}\label{cor:wuol}
For any separable Borel space $(\Xcal,\Bcal)$, and every separable near metric space $(\Ycal,\ell)$ with $0<\bar\ell<\infty$ we have  $\wuol_{(\Xcal,\rho),(\Ycal,\ell)}=\wsmv_{(\Xcal,\rho)}$. In particular, $\suol$ is invariant from the output setup. 
\end{corollary}
In particular, this shows that if $0<\bar\ell<\infty$, 2C1NN is weakly optimistically universal. This result still holds if $\bar\ell=0$ in which case all processes $\Xbb$ are weakly learnable and any learning rule is weakly optimistically universal.

\begin{corollary}\label{cor:opt_weak}
For any separable Borel space $(\Xcal,\Bcal)$ and any bounded separable output setting $(\Ycal,\ell)$, 2C1NN is weakly optimistically universal.
\end{corollary}
This completely closes the main questions on universal online learning \cite{hanneke2021learning,hanneke2021open} as we have now proved Theorem \ref{thm:2C1NN_optimistically_universal} (concatenation of Corollary \ref{cor:opt_general} and \ref{cor:opt_weak}) and Theorem \ref{thm:suol=smv} (concatenation of Corollary \ref{cor:suol_general} and \ref{cor:wuol}).

\section{Conclusion}
\label{sec:conclusion}
In this paper, we provided a strong and weak optimistically universal learning rule 2C1NN, which is a simple variant of the nearest neighbor algorithm. We further gave a characterization of the processes admitting strong or weak universal learning, closing the study of universal online learning with bounded losses. 

The case of unbounded losses was already settled in \cite{hanneke2021learning,blanchard2022universal}, which was shown to be very restrictive because the target functions are unrestricted. It would be interesting to bridge the gap between these two cases by considering \emph{restricted} universal learning. Specifically, by adding an additional constraint on the target functions---for example moment constraints are fairly common in the litterature \cite{gyorfi:02,gyorfi:07}---one could hope to recover the large set of learnable processes $\suol$ characterized in this paper, even for the unbounded loss case. We refer to \cite{blanchard2022universal} for further motivation of this open direction.

In our setting, we assume that the values are generated from the stochastic process $\Xbb$ through a target function $f^*$ and without noise. Another interesting line of research would be to add noise to the value process $\Ybb$. This relates to the Bayes consistency literature in which an objective is to reach the minimal risk, known as the Bayes minimal risk; instead of obtaining exact consistency i.e. vanishing average error rate as considered in this paper. A possible direction would be to find mild independence conditions on the noise---generalizing the i.i.d. setting \cite{tsir2022medoid}---so that there exist learning rules which are Bayes universally consistent under a large set of processes $\Xbb$.

\paragraph{Acknowledgements.}The author is very grateful to Patrick Jaillet, Romain Cosson and Steve Hanneke for very useful discussions and for reviewing the manuscript. This work is being partly funded by ONR grant N00014-18-1-2122.

\bibliography{refs}

\end{document}